\documentclass{article}

\usepackage{microtype}
\usepackage{graphicx}
\usepackage{subcaption}
\usepackage{booktabs} 

\usepackage{hyperref}

\usepackage{bbm}
\usepackage{amsmath}
\usepackage{amssymb}
\usepackage{mathtools}
\usepackage{amsthm}
\usepackage{amsfonts}
\usepackage{graphicx}
\usepackage{multirow}

\usepackage{geometry}
\usepackage{algorithm}
\usepackage{algorithmic}

\usepackage[capitalize,noabbrev]{cleveref}

\theoremstyle{plain}
\newtheorem{theorem}{Theorem}[section]
\newtheorem{proposition}[theorem]{Proposition}
\newtheorem{lemma}[theorem]{Lemma}

\theoremstyle{definition}
\newtheorem{definition}[theorem]{Definition}
\newtheorem{assumption}[theorem]{Assumption}
\newtheorem{condition}[theorem]{Condition}
\theoremstyle{remark}

\usepackage[textsize=tiny]{todonotes}

\usepackage{amsmath,amsfonts,bm}
\usepackage{thm-restate}

\def\eqref#1{equation~\ref{#1}}

\def\1{\bm{1}}

\def\mF{{\bm{F}}}

\def\mpi{{\bm{\pi}}}

\DeclareMathAlphabet{\mathsfit}{\encodingdefault}{\sfdefault}{m}{sl}
\SetMathAlphabet{\mathsfit}{bold}{\encodingdefault}{\sfdefault}{bx}{n}

\def\gA{{\mathcal{A}}}

\def\gD{{\mathcal{D}}}

\def\gF{{\mathcal{F}}}
\def\gG{{\mathcal{G}}}
\def\gH{{\mathcal{H}}}

\def\gM{{\mathcal{M}}}
\def\gN{{\mathcal{N}}}
\def\gO{{\mathcal{O}}}

\def\gS{{\mathcal{S}}}
\def\gT{{\mathcal{T}}}
\def\gU{{\mathcal{U}}}
\def\gV{{\mathcal{V}}}
\def\gW{{\mathcal{W}}}
\def\gX{{\mathcal{X}}}
\def\gY{{\mathcal{Y}}}
\def\gZ{{\mathcal{Z}}}

\def\sI{{\mathbb{I}}}

\def\sP{{\mathbb{P}}}

\def\sR{{\mathbb{R}}}

\def\sT{{\mathbb{T}}}

\def\sX{{\mathbb{X}}}

\def\Prob{{\mathbb{P}}}

\newcommand{\wt}{\widetilde}
\newcommand{\wh}{\widehat}

\def\wtF{{\wt{F}}}

\newcommand{\R}{\mathbb{R}}

\newcommand{\EE}{\mathbb{E}}
\newcommand{\RR}{\mathbb{R}}

\newcommand{\TT}{\mathbb{T}}
\newcommand{\cA}{\mathcal{A}}

\newcommand{\cH}{\mathcal{H}}

\newcommand{\cF}{\mathcal{F}}

\newcommand{\cS}{\mathcal{S}}
\newcommand{\cV}{\mathbb{V}}
\newcommand{\cN}{\mathcal{N}}

\newcommand{\cD}{\mathcal{D}}
\newcommand{\cT}{\mathcal{T}}
\newcommand{\innerproduct}[2]{\langle #1, #2 \rangle}
\newcommand{\cM}{\mathcal{M}}

\newcommand{\MLE}{\mathrm{MLE}}
\newcommand{\LSR}{\mathrm{LSR}}
\newcommand{\TV}{\mathrm{TV}}

\newcommand{\poly}{\mathrm{poly}}

\newcommand{\abs}[1]{\left| #1 \right|}

\newcommand{\bracket}[1]{\left(#1\right)}
\newcommand{\mbracket}[1]{\left[#1\right]}
\newcommand{\norm}[1]{\left\| #1 \right\|}
\newcommand{\sets}[1]{\left\{ #1 \right\}}
\newcommand{\PP}{\mathbb{P} }

\newcommand{\dist}{\mathrm{Dist}}

\newif\ifsup\supfalse
\suptrue

\DeclareMathOperator*{\argmax}{argmax}
\DeclareMathOperator*{\argmin}{argmin}

\newcommand{\mtheta}{{\bm{\theta}}}
\newcommand{\mTheta}{{\bm{\Theta}}}
\newcommand{\mtau}{{\bm\tau}}

\begin{document}

\title{Provable Risk-Sensitive Distributional Reinforcement Learning with General Function Approximation}

\author{
    Yu Chen$^{1, *}$,\quad
    Xiangcheng Zhang$^{1, *}$,\quad Siwei Wang$^{2}$, \quad 
    {Longbo Huang}$^{1, \dagger}$
    \\
    $^{1}$ Tsinghua University, \quad 
    $^{2}$ Microsoft Research Asia \\
    \texttt{\{chenyu23, xc-zhang21\}@mails.tsinghua.edu.cn},\\
    \texttt{siweiwang@microsoft.com}, 
    \texttt{longbohuang@tsinghua.edu.cn}
}

\date{}

\maketitle

\newcommand{\yu}[1]{{\color{blue}[Yu: #1]}}
\newcommand{\longbo}[1]{{\color{brown}[Longbo: #1]}}
\newcommand{\xiangcheng}[1]{{\color{red}[Xiangcheng: #1]}}
\newcommand{\siwei}[1]{{\color{purple}[siwei: #1]}}

\newcommand{\TODO}[1]{{\color{purple}[TODO: #1]}}

\renewcommand{\thefootnote}{\fnsymbol{footnote}}
\footnotetext[1]{These authors contributed equally.}
\footnotetext[2]{Corresponding author.}
\renewcommand*{\thefootnote}{\arabic{footnote}}

\begin{abstract}
In the realm of reinforcement learning (RL), accounting for risk is crucial for making decisions under uncertainty, particularly in applications where safety and reliability are paramount.
In this paper, we introduce a general framework on Risk-Sensitive Distributional Reinforcement Learning (RS-DisRL), with static Lipschitz Risk Measures (LRM) and general function approximation. Our framework covers a broad class of risk-sensitive RL, and facilitates  analysis of the impact of estimation functions on the effectiveness of RSRL strategies and evaluation of their sample complexity. 
We design two innovative meta-algorithms: \texttt{RS-DisRL-M}, a model-based strategy for model-based function approximation, and \texttt{RS-DisRL-V}, a model-free approach for general value function approximation. With our novel estimation techniques via Least Squares Regression (LSR) and Maximum Likelihood Estimation (MLE) in distributional RL with augmented Markov Decision Process (MDP), we derive the first $\widetilde{\mathcal{O}}(\sqrt{K})$ dependency of the regret upper bound for RSRL with static LRM, marking a pioneering contribution towards statistically efficient algorithms in this domain.

\end{abstract}



\section{Introduction}

Reinforcement learning (RL) \cite{sutton2018reinforcement} has emerged as a powerful framework for sequential decision-making in dynamic and uncertain environments. While traditional RL methods, predominantly focused on maximizing the expected return, have seen significant advancements through approaches such as Q-learning \cite{mnih2015humanlevel, jin2018qlearning} and policy gradients \cite{kakade2001natural, cai2020provably}, they often fall short in real-world scenarios demanding strict risk control, such as financial investment \cite{bielecki2000risk} , medical treatment \cite{ernst2006clinical}, and automous driving \cite{candela2023risk}.


The significance of comprehending risk management in RL has led to the emergence of Risk-Sensitive RL (RSRL). Unlike risk-neutral RL, which primarily focuses on maximizing expected returns, RSRL seeks to optimize risk metrics, such as entropy risk measures (ERM)  \cite{fei2020risksensitive, fei2021risksensitive} or conditional value-at-risk (CVaR) \cite{wang2023nearminimaxoptimal}, of the possible \emph{cumulative reward} which emphasizes its distributional characteristics. However, traditional RL framework based on Q-learning which typically considers the mean of reward-to-go and corresponding Bellman equation, cannot efficiently capture the characteristics of the cumulative reward's distribution. Therefore, there has been an upsurge of interest in Distributional RL (DisRL) due to its capacity to understand the intrinsic distributional attributes of cumulative rewards, which has already achieved significant empirical success in risk-sensitive tasks \cite{bellemare2017distributional,dabney2018implicit,keramati2020being,urpi2020riskaverse,lim2022distributional}.

However, there remains a dearth of comprehensive theoretical insights into the sample complexity of Risk-Sensitive Distributional RL (RS-DisRL), particularly in scenarios encompassing general risk measures and function approximation. Previous theory works of RS-DisRL have primarily been constrained to tabular MDPs \cite{bastani2022regret, liang2022bridging} which fail in extending to infinite-state settings, or have been confined to specific risk measures such as CVaR or ERM \cite{stanko2019risk}.

In this paper, we delve into the RS-DisRL with \emph{staic Lipschitz risk measures} (LRM), a general risk measure class that includes various well-known risk measures such as coherent risk, convex risk, CVaR, and ERM. 
In order to address the challenges posed by extremely large or infinite state spaces, we consider two distinct general function approximation scenarios: model-based function approximation and value function approximation.
The model-based function approximation, as extensively used in prior works such as \cite{sun2019model,liu2022when,liu2023optimistic}, typically assumes that the agent has access to a model class that contains the true transition model. On the other hand, the general value function approximation, as explored in \cite{wang2020reinforcement, jin2021bellman, agarwal2023vo,wu2023distributional, wang2023benefits}, offers the agent a distributionally Bellman-complete value function class with the true value distribution.

Under these settings, we introduce general model-based and model-free meta-algorithms, respectively, and employ estimate techniques including  Least Squares Regression (LSR) and Maximum Likelihood Estimation (MLE), achieving a sublinear regret upper bound with respect to the number of episodes. Importantly, our work establishes the first statistically efficient framework for RS-DisRL with static Lipschitz risk measures. 

\textbf{Challenges} Significant gaps persist in our quest to achieve statistically efficient RS-DisRL with static LRM.
(i) The utilization of static LRM in RSRL presents unique complexities. Unlike standard RL, where the focus is on maximizing the expected cumulative reward, RSRL with static LRM involves optimizing the entire distribution of cumulative rewards. This distinct characteristic makes the traditional Q-learning methods, which typically consider the mean of reward-to-go and the corresponding Bellman equation for mean value functions, inadequate for characterizing policy performance and the reward's distribution. 
(ii) 
In RSRL with static LRM, the optimal policy is non-Markovian, dependent not only on the current state but also on the rewards received thus far. Thus, it is hard to extend previous works for learning a Markovian policy within polynomial sample complexity. 

\textbf{Technical Contribution} To surmount these obstacles, our approach involves several novel technical components. (i) We integrate rigorous distribution analysis techniques from DisRL into the RSRL framework and design novel distributional learning in augmented MDPs, which help us better understand the distributional characteristics of the problem objective. 
(ii) We pioneer the application of LSR in the exploration process of distributional RL with augmented MDP, incorporating our innovative regression technique tailored for cumulative distribution functions (CDFs) (see Sections~\ref{sec:mblsr} and \ref{sec:mflsr}). 
(iii) Furthermore, we extend traditional MLE methods to DisRL within the augmented MDP framework, supported by a novel connection technique: the augmented simulation lemma (Lemma~\ref{lem:simulation lemma}). In these manners, we present the first statistically efficient algorithms for RSRL with static LRM in this paper.

We summarize Table~\ref{tb:result} to present the technical results in this paper, and discuss our detailed contribution as follows:

\textbf{(i)}
We provide a general framework for RSRL with static LRM and the general function approximation, which covers a broad class of RSRL studies with various popular static risk measures, such as ERM \cite{fei2020risksensitive,fei2021risksensitive}, CVaR \cite{wang2023nearminimaxoptimal,zhao2023provably}, and spectral risk \cite{bastani2022regret}. The  framework facilitates analysis of the impact of estimation functions on the effectiveness of RSRL strategies and evaluation of their sample complexity. 

\textbf{(ii)}
For model-based function approximation, we propose a novel meta-algorithm named \texttt{RS-DisRL-M} (Algorithm~\ref{alg:mbframe}), with a general regret upper bound $\gO(L_\infty(\rho)\xi(\texttt{M-Est}))$, where $L_\infty(\rho)$ represents the Lipschitz constant of the LRM $\rho$, and $\xi(\texttt{M-Est})$ is the effectiveness determined by the model-based estimation function. 
Based on the meta-algorithm, we also obtain the first analysis on the model-based LSR and MLE approaches in distributional RL with augmented MDPs, with effectiveness $\xi = \widetilde{\gO}(\poly(H)\operatorname{D}\sqrt{K})$, where $H$ is the horizon length, $K$ is the number of episodes, and $\operatorname{D}$ is the structural complexity (specified in Theorems~\ref{thm:mblsr} and \ref{thm:mbmle}). 

\textbf{(iii)}
For general value function approximation, we present a new model-free framework \texttt{RS-DisRL-V} (Algorithm~\ref{alg:mfframe}), with general regret upper bound $\gO(L_\infty(\rho)\zeta(\texttt{V-Est}))$ where $\zeta(\texttt{V-Est})$ describes the effectiveness of the model-free estimation approach. We also provide novel analysis of LSR and MLE in distributional RL with augmented MDPs and theoretical guarantees of $\zeta = \widetilde{\gO}(\poly(H)\operatorname{D}\sqrt{K})$, where the dimension term $\operatorname{D}$ is specified in Theorem~\ref{thm:mflsr} for the LSR case and Theorem~\ref{thm:mfmle} for the MLE case. 

\begin{table}
\caption{Summary of results in this paper, where $\operatorname{D}$ represent the structural complexity determined by specific methods.
\label{tb:result}}
\begin{center}
        \begin{small}
            \begin{tabular}{ll}
                \toprule
                Algorithm & Regret \\
                \midrule
                \multicolumn{2}{l}{Model-based Framework \texttt{RS-DisRL-M} (Algorithm~\ref{alg:mbframe})} \\
                \midrule
                  LSR Approach: \texttt{M-Est-LSR}(Algorithm~\ref{alg:mbestlsr}) & $\widetilde{\gO}(L_\infty(\rho)H\operatorname{D}_1\sqrt{K})$  (Theorem~\ref{thm:mblsr})
                         \\
                 MLE Approach: \texttt{M-Est-MLE} (Algorithm~\ref{alg:mbestmle})  & $\widetilde{\gO}(L_\infty(\rho)\poly(H)\operatorname{D}_2\sqrt{K})$ (Theorem~\ref{thm:mbmle}) \\
                \midrule
                \multicolumn{2}{l}{Model-free Framework \texttt{RS-DisRL-V} (Algorithm~\ref{alg:mfframe})} \\
                \midrule
                  LSR Approach: \texttt{V-Est-LSR} (Algorithm~\ref{alg:mfestlsr}) & $\widetilde{\gO}(L_\infty(\rho)H\operatorname{D}_3\sqrt{K})$ (Theorem~\ref{thm:mflsr}) \\
                        
                 MLE Approach: \texttt{V-Est-MLE} (Algorithm~\ref{alg:mfestmle}) & $\widetilde{\gO}(L_\infty(\rho)\poly(H)\operatorname{D}_4\sqrt{K})$ (Theorem~\ref{thm:mfmle}) \\
                \bottomrule
            \end{tabular}
        \end{small}
    \end{center}
\end{table}


\section{Related Works}
\paragraph{RSRL} In RSRL studies, there are two types of risk measures. One is to consider the iterated risk measure, i.e., computing the risk value iteratedly. For example, \cite{du2022provably,chen2023provably} considers iterated CVaR risk measures, and \cite{liang2023regret} considers iterated LRM risk measures. The other is to consider the static risk measure, i.e., a risk measure of the cumulative reward. For example, \cite{fei2020risksensitive,fei2021risksensitive} focus on RSRL with ERM, \cite{wang2023nearminimaxoptimal,zhao2023provably} investigate the static CVaR risk measures, and \cite{bastani2022regret} studies the static spectral risk measures. In this paper, we consider the static LRM, a static risk measure that encompasses the ERM, static CVaR and spectral risk, and give the theoretical learning analysis.

\paragraph{DisRL}
Many previous works \cite{rowland2018analysis, rowland2023analysis, bellemare2017distributional} develop asymptotic covergence analysis for DisRL. With MLE approaches, \cite{wu2023distributional} discusses the statistical complexity bounds for off-policy DisRL and \cite{wang2023benefits} considers the small-loss bounds for DisRL. \cite{bastani2022regret} is the first to give the sample complexity bounds for RSRL, while it only studies the static spectral risk measure within the tabular MDPs. Compared to these results, our work focuses on a more general risk-sensitive target (LRM) and enables to use the general function approximation.

\section{Notations}


For a positive integer $N$, we let $[N] = \{1, 2, \cdots, N\}$. Denote $\int_s f(s) := \int_S f(s)ds$ if we integrate $f$ over the universal set of $s \in S$. For a function $f : \gX \to \R$, we define the $\ell_1$-norm $\|f\|_1 = \int_\gX |f(x)|dx$ and $\ell_\infty$-norm $\|f\|_\infty = \sup_{x\in\gX}|f(x)|$. Denote $\Delta(\gX)$ as the distribution over space $\gX$. We use the standard $\gO(\cdot)$ to hide universal constant factors, and $\widetilde\gO(\cdot)$ to further suppress the polylogarithmic factors in $\gO(\cdot)$.

\section{Problem Formulation}

\paragraph{Episodic Markov Decision Process}
In this study, we examine an episodic Markov Decision Process (MDP) denoted as $\gM = ( \gS, \gA, K, H, \{\sP_h\}_{h=1}^H, \{\R_h\}_{h=1}^H)$ characterized by state space $\gS$, action space $\gA$, finite episode number $K$, finite time horizon length $H$, transition probabilities $\Prob_h(\cdot | s, a) \in \Delta(\gS)$ and \emph{distributional} reward $\sR_h(\cdot | s, a)\in \Delta([0, 1])$\footnote{Without loss of generality, we assume the reward $r \in [0, 1]$ for each step. Additionally, we assume that the agent has knowledge of the reward distribution, a common assumption in prior research \cite{liu2022when,liu2023optimistic,wang2023benefits}.} for step $h \in [H]$. At the outset of each episode $k \in [K]$, we start with an initial state $s_{k,1}$ chosen by the MDP.
At each step $h\in[H]$, the agent receives a historical record  $\cH_{k,h}=\sets{s_{k,1},a_{k,1},r_{k,1},\cdots,s_{k,h}}$ and select an action $a_{k,h}\sim\pi_{h}^k(\cdot|\cH_{k,h})$ by a \emph{history-dependent} policy  $\mpi_h^k=\sets{\pi_{h}^k}_{h=1}^H$ 
Then, the MDP will return a reward $r_{k,h} \sim \sR_h(\cdot | s_{k,h}, a_{k,h})$ and transfer into next state $s_{k,h+1} \sim \Prob_h(\cdot | s_{k,h}, a_{k,h})$. This episode will end in step $H+1$.
Throughout this paper, we assume that the agent lacks knowledge of the transition probabilities.
For a fixed history-dependent policy $\mpi$, the cumulative reward for an episode played with policy $\mpi$ is a bounded real-valued random variable $Z^\mpi = \sum_{h=1}^H r_h$, where $r_h \sim \R_h(\cdot | s_h,a_h)$. 

\paragraph{Lipschitz Risk Measures}
\emph{Lipschitz Risk Measures} (LRM) are quantified by a function $\rho : \gZ \to \sR$ mapping random variables $Z \in \gZ$ to real numbers, distinguished by two critical properties.
\textbf{C1. Law invariance: } If $Z, W \in \gZ$ have the same distribution functions, $F_Z = F_W$, then $\rho(Z) = \rho(W)$.
\textbf{C2. Lipschitz continuity: } Consider the supremum norm $\|\cdot\|_\infty$ over the set of all distribution functions of the random variable class $\gZ$. There exists a Lipschitz constant $L_\infty(\rho)$ such that $\left| \rho(Z) - \rho(W) \right| \leq L_\infty(\rho)\| F_Z - F_W \|_\infty$ holds for any $Z, W \in \gZ$.

The law invariance condition, foundational in risk measure studies \cite{kusuoka2001law, frittelli2005law, liang2023regret}, indicates that identical distribution functions result in equal risk measures. LRM encompass a broad spectrum of general risk assessments, including coherent risk \cite{artzner1999coherent}, monetary risk \cite{jia2020monetary}, and convexity risk measures \cite{follmer2012convex}, highlighting the versatility of LRM. Popular metrics like Entropy Risk Measures (ERM) and Conditional Value-at-Risk (CVaR) also align with these conditions, with Lipschitz constants $L_\infty(\operatorname{ERM}) = \frac{\exp(|\gamma|H) - 1}{|\gamma|}$ and $L_\infty(\operatorname{CVaR}) = \frac{H}{a}$ \cite{liang2023regret}.


\paragraph{RSRL with Static LRM}
In this paper, we delve into the \textbf{R}isk-\textbf{S}ensitive \textbf{R}einforcement \textbf{L}earning (RSRL) with static \textbf{L}ipschitz \textbf{R}isk \textbf{M}easures (LRM), focusing on optimizing risk-sensitive rewards via \emph{history-dependent} policies. 
The objective is to find an optimal policy $\mpi^*$ that maximizes the LRM-defined cumulative reward $\mpi^*=\argmax_{\mpi}\rho(Z^{\mpi})$, then we define our regret as the difference between the cumulative rewards of the optimal policy and the policy deployed at each episode: $
    \operatorname{Regret}(K)=\sum_{k=1}^K \rho(Z^{\mpi^*})-\rho(Z^{\mpi^k})$.


\section{Augmented MDPs and Distributional Bellman Equation}\label{sec:augmented mdp}
\label{sec:aug}
The key of our analysis revolves around the distributional Bellman equation applied to augmented MDPs.
Recognizing that learning an optimal history-dependent policy $\pi^*$ can be prohibitively sample-intensive, previous works \cite{bauerle2011markov, bastani2022regret} leverage the concept of augmented MDP in risk-sensitive conditions,
where any history-dependent policy $\tilde\mpi$ corresponds to a \emph{Markov} policy $\mpi^\dag$ in the augmented MDP. 
This equivalence allows for facilitating effective risk-sensitive policy learning without sacrificing computational tractability.

\paragraph{Augmented MDPs}
We embrace the notion of augmented MDPs, originally introduced by \cite{bauerle2011markov} and widely used in RSRL with static risk measures \cite{bastani2022regret,wang2023nearminimaxoptimal,zhao2023provably}. In the context of augmented MDPs, the state space is expanded to $\gS^\dag=\{ s^\dag_h=(s_h,y_h)\}$, where $y_h=\sum_{h=1}^H r_h$ represents the cumulative reward accumulated up to time step $h$.We denote the augmented MDP as  $\gM^\dag = (\gS^\dag,\gA, K, H, \{\sP_h\}_{h=1}^H, \{\R_h\}_{h=1}^H)$. To capture the augmented dynamics, we introduce the augmented transition operator as follows:
\begin{equation*}
\TT_h(s_{h+1}^\dag|s_h^\dag,a_h):=\PP_h(s_{h+1}|s_h,a_h)\RR_h(y_{h+1}-y_h|s_h,a_h)
\end{equation*}
Let $\Pi^\dag$ denote the set of the Markov policies within the augmented MDP $\cM^\dag$.  Theorem 3.1 in \cite{bastani2022regret} shows that for any history-dependent policy $\Tilde{\pi}$ in original MDP, there exists a Markov policy in the augmented MDP $\mpi\in\Pi^\dag$, such that $F_{Z^\mpi}=F_{Z^{\Tilde{\mpi}}}$. This result underscores the equivalence between the distribution of cumulative rewards under a history-dependent policy in original MDP and a corresponding Markov policy in the augmented MDP.

\paragraph{Distributional Bellman Equation} Similar to prior works such as \cite{bellemare2017distributional,bastani2022regret, wang2023benefits, wu2023distributional},
we integrate the distributional Bellman equation within an augmented MDP framework. For any policy $\mpi\in\Pi^\dag$ and $h\in[H]$, we denote the random variable $Z_h^\mpi(s_h^\dag,a_h)$ as the cumulative return $\sum_{h'=h}^H r_{h'}$ after taking action $a_h$ in state $s_h^\dag$ via policy $\mpi\in\Pi^\dag$.
\begin{definition}[Distributional Bellman Equation \cite{wang2023benefits,bastani2022regret}]
Let $F_h^\mpi(\cdot|s_h^\dag,a_h)$ be the cumulative distribution function (CDF) of random variable $Z_h^\mpi(s_h^\dag, a_h)$, and let $f_h^\mpi(\cdot | s_h^\dag, a_h)$ be its probability density function (PDF). We define the augmented distributional Bellman operator $\cT_{h,\mpi}^\dag:\Delta(\RR)\rightarrow\Delta(\RR)$ as:
\begin{align*}
    &{\cT_{h,\mpi}^\dag f_{h+1}^\mpi}(x|s_h^\dag,a_h):=\int_{s_{h+1}^\dag,a_{h+1}}\TT_h(s_{h+1}^\dag|s_h^\dag,a_h) \pi_{h+1}(a_{h+1}|s_{h+1}^\dag) f_{h+1}^\mpi(x-(y_{h+1}-y_h)|s_{h+1}^\dag,a_{h+1})\,,\\
     &{\cT_{h,\mpi}^\dag F_{h+1}^\mpi}(x|s_h^\dag,a_h):=\int_{s_{h+1}^\dag,a_{h+1}}\TT_h(s_{h+1}^\dag|s_h^\dag,a_h) \pi_{h+1}(a_{h+1}|s_{h+1}^\dag)F_{h+1}^\mpi(x-(y_{h+1}-y_h)|s_{h+1}^\dag,a_{h+1})\,.
\end{align*}
By the definition of $Z_h^{\mpi}$ and $Z_{h+1}^\mpi$, we have $f_h^\mpi = {\cT_{h,\mpi}^\dag f_{h+1}^\mpi}$ and $F_h^\mpi = {\cT_{h,\mpi}^\dag F_{h+1}^\mpi}$. Generally, we can write the distributional Bellman equation in random variable version as $Z_h^{\mpi}=\cT_{h,\mpi}^\dag Z_{h+1}^{\mpi}$.
\end{definition}

\section{\!Model-Based Meta-Algorithm\! \texttt{RS-DisRL-M}}
\label{sec:mb}
This section introduces the meta-algorithm \texttt{RS-DisRL-M} for model-based function approximation in RS-DisRL, alongside its theoretical underpinnings. It further delves into two pivotal estimation techniques: LSR and MLE, formulating statistically efficient algorithms for RSRL with LRM.

We describe the model-based function approximation framework utilized for our analysis, drawing on the methodologies previously explored by \cite{fei2021risksensitive, liu2022when,liu2023optimistic,chen2023provably}. Specifically, we model each transition probability $\Prob : \gS \times \gA \to \Delta(\gS)$ using a parametric form $\Prob_\theta$, with the true transition model for each decision epoch $h$ represented by $\theta_h^*$. The set of true models across all epochs is denoted by $\bm{\theta}^* = \{\theta_h^*\}_{h=1}^H$.
We then establish a standard realizability assumption for the model-based function approximation, ensuring that our model accurately reflects the dynamics of the environment
\begin{assumption}[Model-based realizability \cite{fei2021risksensitive, liu2022when, liu2023optimistic, chen2023provably}]
\label{ass:mbreal}
The agent is given a model class $\mTheta = \{\Theta_h : h \in [H]\}$ which specifies a class of transition probabilities $\{\Prob_{\theta_h} : \theta_h \in \Theta_h, h \in [H]\}$ and satisfies $\bm{\theta}^*  \in \bm{\Theta}$.\footnote{We assume this structure of $\mTheta$ to simplify the notations in analysis. In fact, our analysis works as long as $\mtheta^* \in \mTheta$.}
\end{assumption}

\begin{algorithm}[htbp]
   \caption{\texttt{RS-DisRL-M}}
   \label{alg:mbframe}
\begin{algorithmic}[1]
   \STATE {\bfseries Input:} Model class $\mTheta$, confidence radius $\beta$.
   \STATE {\bfseries Initialize:} $\wh{\mTheta}_1 \leftarrow \mTheta$.
   \FOR{$k=1$ {\bfseries to} $K$}
   \STATE 
   $(\mpi^{k}, \widehat\mtheta_k)=\argmax_{\mpi\in\bm{\Pi}^\dag, \mtheta\in\widehat{\bm{\Theta}}_k}\rho(Z^{\mpi}_{\mtheta})$.~\textcolor{blue}{//Optimistic planning}\label{algline:mbframeoptimisticplanning}
   \STATE Execute policy $\mpi^{k}$, add the collected data $\bm{\tau}_k=\sets{(s_{k,h},a_{k,h},r_{k,h})}_{h=1}^H$ and $\mpi^{k}$, $\widehat{\mtheta}_k$ into history $\cH_k = \cH_{k-1}\cup \{(\bm{\tau}_k, \mpi^{k}, \widehat{\mtheta}_k)\}$.~\textcolor{blue}{//Data collection}\label{algline:mbframedatacollection}
   \STATE $\widehat{\bm\Theta}_{k+1}=\texttt{M-Est}\bracket{\bm{\Theta},\cH_k,\beta}$.~\textcolor{blue}{//Confidence set construction}\label{algline:mbframeconfidenceset}
\ENDFOR
\end{algorithmic}
\end{algorithm}

We introduce the meta framework \texttt{RS-DisRL-M} (Algorithm~\ref{alg:mbframe}) for \textbf{R}isk-\textbf{S}ensitive \textbf{Dis}tributional \textbf{RL} with \textbf{M}odel-\textbf{B}ased function approximation.
\texttt{RS-DisRL-M} is a model-based algorithm which takes a model class $\mTheta$ as an input, and operates in three main steps. 

\textbf{(i) 
Optimistic planning }(Line~\ref{algline:mbframeoptimisticplanning}): the algorithm computes the optimistic model $\wh{\mtheta}_k$ and corresponding augmented policy $\mpi^k$ in the model confidence set $\wh\mTheta_k$. Here the random variable $Z^{\mpi}_{\mtheta}$ denotes the cumlative reward colloected with policy $\mpi$ in augmented MDPs modeled by $\mtheta$. 
\textbf{(ii) Data collection }(Line~\ref{algline:mbframedatacollection}): the algorithm executes the optimal policy $\mpi^k$ planned from step (i) and collects the trajectory $\mtau_k$. 
\textbf{(iii) Confidence set construction }(Line~\ref{algline:mbframeconfidenceset}): the algorithm estimates the models and constructs the new confidence set for the next episode based on a \textbf{M}odel-based \textbf{Est}imation function (\texttt{M-Est}), model class $\mTheta$, and confidence radius $\beta$. The estimation function \texttt{M-Est} can be designed by various estimation methods, such as LSR or MLE, depending on the specific structure of the MDP or the model class.

This framework encapsulates the essence of leveraging model-based approaches for efficient learning and adaptation in RSRL, aligning with strategies explored in recent literature \cite{bastani2022regret,liu2023optimistic,chen2023provably}.
\paragraph{Theoretical Guarantees}
The theoretical guarantees for the \texttt{RS-DisRL-M} algorithm are anchored on two critical conditions related to the estimation function \texttt{M-Est}. 
\begin{condition}[Concentration]\label{con:mbconcentration}
    For $\delta \in (0, 1]$, with probability at least $1 - \delta$, $\mtheta^*\in\widehat\mTheta_k$ holds for every $k \in [K]$. 
\end{condition}
The concentration condition, common in theoretical RL analysis \cite{agarwal2019reinforcement,jin2018qlearning,ayoub2020model}, ensures that the true transition model is consistently included within the algorithm's confidence set throughout the learning process.
\begin{condition}[General elliptical potential]
\label{con:mbelliptical}
For $\delta \in (0, 1]$, with probability at least $1 - \delta$, the supremum distance between the CDF of the chosen estimated model $\widehat{\mtheta}^k$ and real model $\mtheta^*$ under policy $\mpi^k$ can be bounded by $
    \sum_{k=1}^K \left\|F_{Z_{\wh{\mtheta}_k}^{\mpi^k}} - F_{Z_{\mtheta^*}^{\mpi^k}}\right\|_\infty \leq \xi(K, H, \mTheta, \beta, \delta)$, where $\xi(K,H,\mTheta, \beta, \delta)$ is the complexity bound determined by the estimation function \texttt{M-Est}.
\end{condition}
Intuitively, Condition~\ref{con:mbelliptical} bounds the estimation error by controlling the discrepancy between the CDF of the estimated and real models under the chosen policy. This condition resembles the pigeonhole principle for tabular MDPs \cite{jin2018qlearning} and the elliptical potential lemma for linear and linear mixture MDPs \cite{jin2020provably, zhou2021nearly}. However, Condition~\ref{con:mbelliptical} demands to bound the supremum difference of the CDF during the learning process, which is natrually different from the previous. 

Conditions \ref{con:mbconcentration} and \ref{con:mbelliptical} together establish the reliability and effectiveness of the estimation function in the \texttt{RS-DisRL-M} framework. Adherence to these conditions signifies that the estimation function can facilitate efficient learning in \texttt{RS-DisRL-M}.
\begin{theorem}~\label{thm:mbmeta}
    Under Assumption \ref{ass:mbreal}, if the estimation function \texttt{M-Est} satisfies Conditions~\ref{con:mbconcentration} and \ref{con:mbelliptical}, then the regret of \texttt{RS-DisRL-M} (Algorithm~\ref{alg:mbframe}) can be bounded by $
        \operatorname{Regret}(K) \leq L_\infty(\rho)\xi(K, H, \mTheta, \beta, \delta)$.
\end{theorem}

This theorem integrates and extends results from previous research on RSRL with static risk measures, offering a comprehensive view that includes notable theorems from \cite{bastani2022regret, wang2023nearminimaxoptimal, fei2021risksensitive}. The primary challenge lies in satisfying the concentration and elliptical potential conditions for the estimation function \texttt{M-Est} and managing the complexity bound $\xi(K, H, \mTheta, \beta, \delta)$. 

Below, we introduce LSR (\texttt{M-Est-LSR}, Algorithm~\ref{alg:mbestlsr}) and MLE (\texttt{M-Est-MLE}, Algorithm~\ref{alg:mbestmle}) based estimation functions. These functions meet the necessary conditions and demonstrate an effective bound $\xi(K,H,\mTheta, \beta, \delta) = \widetilde\gO(\poly(H)\operatorname{D}\sqrt{\beta K/\delta})$ with dimension term $\operatorname{D}$ specified in Theorem~\ref{thm:mblsr} for LSR  and Theorem~\ref{thm:mbmle} for MLE, giving the $\sqrt{K}$ dependency for meta-algorithm \texttt{RS-DisRL-M} and achieving minimax-optimal in terms of $K$ in tabular MDPs for ERM \cite{fei2020risksensitive} and CVaR \cite{wang2023nearminimaxoptimal}.

\subsection{Estimation by Model-Based LSR Approach}\label{sec:mblsr}

Least Squares Regression (LSR), a well-established estimation technique, has been effectively utilized in linear and linear mixture MDPs \cite{jin2020provably, zhou2021nearly}. Its common application involves regression on combinations of the transition model with bounded functions.  In risk-neutral scenarios, it's often paired with mean value functions \cite{jin2020provably}, while in risk-sensitive settings, utility functions are preferred \cite{chen2023provably,xu2023regret}. This section explores a novel approach in distributional RL by combining the transition model with mixed distribution functions based on the distributional Bellman equation for CDFs. The newly proposed \texttt{M-Est-LSR} algorithm represents a statistically efficient LSR method tailored for this context.

The intuition of LSR is to approximate the Bellman operator with the empirical sample. Different from previous stuides \cite{jin2020provably, fei2021risksensitive, chen2023provably}, we have to analyze the augmented distributional Bellman equation for transition model $\mtheta$. Denote $F_h^{\mpi,\mtheta}(x|s_h^\dag) := \int_{a_h} \pi_h(a_h|s_h^\dag)F_h^{\mpi,\mtheta}(x|s_h^\dag,a_h)$, we have the following Bellman equation
\begin{align*}
    F_h^{\mpi, \mtheta}(x | s_h^\dag) &= \int_{(a_h,s_{h+1}^\dag)} \pi_h(a_h | s_{h}^\dag)\TT_{\theta_h}(s_{h+1}^\dag | s_{h}^\dag,a_h) F_{h+1}^{\mpi, {\mtheta}}\left(x - (y_{h+1} - y_h) \middle| s_{h+1}^\dag\right) \,.
\end{align*}
Notice that the transition $\TT_{\theta_h}$ for augmented MDP compresses the real transition $\Prob_{\theta_h}$ in original MDP and the reward distribution $\R_h$. However, the only empirical observation available to the agent is the transfer sample $(s_{k,h},a_{k,h}) \to s_{k,h+1}$ and the reward sample $r_{k,h} \sim \R_h(s_{k,h},a_{k,h})$. Therefore, we have to decompose the augmented Bellman operator for estimating the transition models of original MDP. To do so, we design a mixed distribution function 
\begin{align*}
    \wh{F}_{k\!,h\!+\!1}(s)\! := \!\! \int_{y_{k\!,h}}^1 \!\!\! \sR_h(r|s_{k\!,h},\! a_{k\!,h}) F^{{\mpi}^k\!, \wh{\mtheta}_k}_{h\!+\!1}(x_{k\!,h}\!\! -\! r | s,\! y_{k\!,h}\!\! +\! r)dr\,,
\end{align*}
where $\mpi^k$ and $\wh{\mtheta}_k$ are given by the optimistic planning (Line~\ref{algline:mbframeoptimisticplanning} in Algorithm~\ref{alg:mbframe}) based on the information before episode $k$.  
Here $x_{k,h}$ is defined as Eq.(\ref{eq:mblsrxih}), which maximizes the diameter of $\wh{\Theta}_{k,h}$ with mixed CDFs, intuitively contributing to the exploration direction by maximizing the uncertainty. For simplicity, we denote the combination form $[\Prob_{\theta_h}\wh{F}_{k,h+1}](s_h,a_h):= \int_{s_{h+1}} \Prob_{\theta_h}(s_{h+1} |s_{h}, a_{h})\wh{F}_{k,h+1}(s_{h+1})$, and the combination set $\mathcal{W}_{\mTheta} := \{[\Prob_{\theta_h}F] : \gS \times \gA \to [0, 1] : \mtheta \in \mTheta, F : \gS \to [0, 1] \}$. 

We are now ready to present the procedure of the estimation function \texttt{M-Est-LSR}.
Due to space limitations, we defer the formal pseudocode to the appendix (see Algorithm~\ref{alg:mbestlsr}).
First, \texttt{M-Est-LSR} estimate a model $\theta^\LSR_{k,h}$ for step $h$ at episode $k$ by LSR based on the information before episode $k - 1$, which serves as the center of the confidence set:
\begin{align*}
\theta^{\LSR}_{k,h}\!
\leftarrow \! \argmin_{\theta_h\in\Theta_h} \!\sum_{i=1}^{k-1}\!\left(\wh{F}_{i\!,h\!+\!1}(s_{i\!,h\!+\!1}) \! - \! [\Prob_{\theta_{h}}\wh{F}_{i\!,h\!+\!1}](s_{i\!,h}, a_{i\!,h})\right)^2\,.
\end{align*}
Then, the confidence set for each step $h \in [H]$ is constructed by previous history
\begin{align*}
    \wh{\Theta}_{k,h} := \sets{\theta_h \in \Theta_h : \dist_{\cH_{k-1}, h}^{\LSR}(\theta_h||{\theta}^{\LSR}_{k,h})\leq\beta^\LSR}\,.
\end{align*}

where the Euclidean-type distance function is defined as
\begin{align*}
    &\dist_{\cH_k, h}^{\LSR}(\theta^{1} || \theta^{2})  :=  \\
  & \sum_{i=1}^{k}\left([\Prob_{\theta^{1}}\wh{F}_{i,h\!+\!1}](s_{i,h},a_{i,h}) \! - \! [\Prob_{\theta^{2}}\wh{F}_{i,h\!+\!1}](s_{i,h},a_{i,h})\right)^2
\end{align*}

Finally, the function returns the confidence set $\wh{\mTheta}_k = \sets{\mtheta \in \mTheta : \theta_h \in \wh{\Theta}_{k,h}, h \in [H]}$.
\begin{theorem}\label{thm:mblsr}
    Let the confidence radius for LSR $\beta^\LSR\! := 8\log(2H^2\gN_C(\mTheta,  1/K,\norm{\cdot}_1)/\delta) + 4\sqrt{\log(4HK^2/\delta)}$. The estimation function $\texttt{M-Est-LSR}$ satisfies Conditions \ref{con:mbconcentration} and \ref{con:mbelliptical}, with  complexity bound $\xi^\LSR(K, H, \mTheta, \beta^\LSR, \delta)$ $ = \widetilde{\mathcal{O}}\left(H\sqrt{K\dim_E(\gW_\mTheta, \frac{1}{\sqrt{K}})\log\left(\gN_C(\mTheta,\frac{1}{K},\|\cdot\|_1)/\delta\right)}\right)$, where $\dim_E$ represents the eluder dimension, and $\gN_C$ denotes the covering number.
\end{theorem}

The formal proof is detailed in Appendix~\ref{app:mdlsr}. By applying Theorem~\ref{thm:mblsr}, the meta-algorithm \texttt{RS-DisRL-M} equipped with the \texttt{M-Est-LSR} estimation function achieves a sublinear regret upper bound of $\widetilde\gO(L_\infty(\rho) H\operatorname{D}\sqrt{K})$, where $\operatorname{D} =  \sqrt{\dim_E(\gW_\mTheta, 1/\sqrt{K})\log\left(\gN_C(\mTheta, 1/K,\|\cdot\|_1)/\delta\right)}$ represents the structural dimension. Here the dimension term $\dim_E(\gW_\mTheta, \sqrt{K})$ characterizes the eluder dimension of the combination set $\gW_\mTheta$, and $\gN_C(\mTheta, 1/K,\|\cdot\|_1)$ denotes the covering number of the model class, both of which are commonly employed in the analysis of model-based LSR \cite{ayoub2020model, fei2021risksensitive, chen2023provably}. 

Compared to previous works studying RSRL with ERM, the result of Theorem~\ref{thm:mblsr} improves upon the findings of \cite{fei2020risksensitive,fei2021risksensitive}. Notably, our regret bound does not include the additional $e^{\abs{\gamma} H^2}$ term in the Lipschitz constant. This improvement is attributed to the distributional analysis, which avoids the $e^{\abs{\gamma}H}$ factor while back propagating the Bellman error in Lemma 3 of \cite{fei2020risksensitive}. Furthermore, when transitioning to the risk-neutral setting, the result established in Theorem~\ref{thm:mblsr} aligns with the regret bound presented by \cite{ayoub2020model} up to $H$ factors.

\subsection{Estimation by Model-Based MLE Approach}
We develop the MLE methods to give a refined estimation for the transition models in augmented MDPs. Our method is 
inspired by the generic model-based MLE method \texttt{OMLE} in  \cite{liu2023optimistic}.

Similarly to Algorithm 1 in \cite{liu2023optimistic}, we construct the estimation algorithm \texttt{M-Est-MLE}, where we provide the detailed pseudocode in Algorithm \ref{alg:mbestmle} for space limitation. Employing a standard MLE analysis \cite{geer2000empirical}, we effectively bound the total squared total variation (TV) distance between our estimated model and the true model by:
$\sum_{i=1}^{k-1} \sum_{h=1}^H\EE_{\nu_{\mtheta^*,h}^{\mpi^i}}\mbracket{\norm{\bracket{\PP_{\hat{\theta}_{k,h}}-\PP_{\theta^*_h}}(s_{h},a_{h})}_1^2}
\leq \mathcal{O}(\beta^{\MLE})$, where $\nu_{\mtheta^*}^{\mpi^i}$ denotes the visitation measure for $\mpi^i$ under the real transition kernel. 

However, the standard simulation lemma (e.g., Lemma 10 in \cite{sun2019model}) fails in analyzing the efficiency of \texttt{M-Est-MLE} since the policy learned in our meta-algorithm is non-Markovian for standard episodic MDPs. Instead, we proposed a novel augmented simulation lemma (see Lemma~\ref{lem:simulation lemma}) connecting the Total Variation (TV) distance between model difference with the $\ell_\infty$ distance between the CDFs of the cumulative return random variable: $
    \norm{F_{Z^{\mpi^k}_{\wh{\mtheta}_k}} - F_{Z^{\mpi^k}_{\mtheta^*}}}_\infty
    \leq\sum_{h=1}^H\EE_{ \nu_{\mtheta^*,h}^{\mpi^k}}\mbracket{\norm{\bracket{\PP_{\hat{\theta}_{k,h}}-\PP_{\theta^*_h}}(s_{h},a_{h})}_1} $.
To limit the $\ell_\infty$ distance of CDFs via the estimated error above, we adopt the witness rank $\operatorname{d_{{wit}}}$ defined in Definition \ref{ass:low witness rank}, which is a common structural complexity measure used for model-based RL \cite{sun2019model,huang2022towards,chen2022general,zhong2022gec,liu2023optimistic}. By this way, we have the following theoretical guarantees:
\begin{theorem}[Estimation by Model-Based MLE Approach] \label{thm:mbmle}
Let $\beta^\MLE := H\log\bracket{eK\cN_{[\cdot]}(\mTheta,1/K,\norm{\cdot}_{1})/\delta}$, where $\cN_{[\cdot]}(\mTheta,1/K,\norm{\cdot}_{1})$ is the bracketing number (see Definition~\ref{def:bracket}). The estimation function $\texttt{M-Est-MLE}$ satisfies the Conditions~\ref{con:mbconcentration} and ~\ref{con:mbelliptical} with $
\xi^{\MLE}(K, H, \mTheta, \beta^\MLE, \delta) = \widetilde\gO\left(\poly(H)\bracket{\sqrt{K\operatorname{d_{wit}}\beta^\MLE} }\right)$, where $\operatorname{d_{wit}}$ is the witness rank of the MDP model. (Definition~\ref{ass:low witness rank}).
\end{theorem}
We present the formal proof in Appendix~\ref{app:mbmle}. In the risk neutral setting where $L_\infty(\mathbb{E})=1$, our Theorem~\ref{thm:mbmle}, in conjunction with Theorem~\ref{thm:mbmeta}, presents a regret upper bound that aligns closely with \cite{liu2023optimistic}'s result. We also reference \cite{zhao2023provably}'s exploration of RSRL in static CVaR measures and low-rank MDPs, a special subcase of the MDPs with low V-type witness rank \cite{sun2019model,agarwal2020flambe,uehara2021representation}. 
Our analysis extends to V-type witness ranks, offering a more favorable dependence on $d$, the rank of the transition matrix, compared to \cite{zhao2023provably}'s approach. This distinction is elaborated in Appendix~\ref{sec:proof low rank}, demonstrating our method's broader applicability and efficiency.

\section{Model-Free Meta-Algorithm \texttt{RS-DisRL-V}}
\label{sec:mf}

In this section, we expand the scope of RSRL to include general value function approximation. We begin by establishing the foundational assumptions for general value function approximation. Then, we present a meta-algorithm \textbf{R}isk-\textbf{S}ensitive \textbf{Dis}tributional \textbf{RL} with general \textbf{V}alue function approximation (\texttt{RS-DisRL-V}, Algorithm~\ref{alg:mfframe}).
This algorithm's theoretical guarantees and performance are then discussed.

At first, We introduce a generic function class $\bm{\gZ} = \{\gZ_h : h \in [H]\}$, with each element $Z_h( s^\dag_h, a_h) \in \gZ_h$ representing a probability distribution, which serves as an estimator candidate for the random variable of the optimal cumulative reward $Z^{\mpi^*}_h$. Then we make the foundation assumption of the general value function approximation.
\begin{assumption}[General value function approximation \cite{wu2023distributional, wang2023benefits}]
\label{ass:mfbellman}
For each $h \in [H]$, we have $Z^{\mpi^*}_h \in \gZ_h$, and the given function set $\gZ_h$ satisfy the distributional bellman completeness, such that for any $Z_{h+1}\in\gZ_{h+1}$, we have $\cT_{h,\mpi}^\dag Z_{h+1}\in\gZ_{h}$.
\end{assumption}


While we assume the agent has access to a class of random variables containing the value distribution, in practice, we often estimate the random variable through its Cumulative Distribution Function (CDF) or Probability Density Function (PDF). The practical estimation approach will be elaborated upon in the specific settings outlined in Section \ref{sec:mflsr} and Section \ref{sec:mfmle}. 

\begin{algorithm}[htbp]
   \caption{\texttt{RS-DisRL-V}}
   \label{alg:mfframe}
\begin{algorithmic}[1]
   \STATE {\bfseries Input:} Function class $\bm{\mathcal{Z}}=\mathcal{Z}_1\times\mathcal{Z}_2\cdots\mathcal{Z}_H$, confidence radius $\gamma$.
   \STATE {\bfseries Initialize:} $\wh{\gZ}_{1,\mpi} \leftarrow \gZ$.
   \FOR{$k=1$ {\bfseries to} $K$}
   \STATE 
   $({\mpi}_k, \widehat{  Z}^k)=\argmax_{{\mpi}\in\Pi^\dag, { Z}\in\widehat\gZ_{k,\mpi}}\rho(Z_1)$.~\textcolor{blue}{//Optimistic planning}\label{algline:mfframeoptimisticplanning}
   \STATE Execute policy $\mpi^{k}$, add the collected data $\bm{\tau}_k=\sets{(s_{k,h},a_{k,h},r_{k,h})}_{h=1}^H$ and $\mpi^{k}$, $\widehat{\mtheta}_k$ into history $\cH_k = \cH_{k-1}\cup \{(\bm{\tau}_k, \mpi^{k}, \wh{Z}^k)\}$.~\textcolor{blue}{//Data collection}\label{algline:mfframedatacollection}
   \STATE $\widehat\gZ_{k+1, \bm{\pi}}=\texttt{V-Est}(\mathcal{H}_k, \bm{\gZ}, \bm\pi,\gamma)$.~\textcolor{blue}{//Confidence set construction}\label{algline:mfframeconfidenceset}
\ENDFOR
\end{algorithmic}
\end{algorithm}

The meta-algorithm \texttt{RS-DisRL-V} (Algorithm~\ref{alg:mfestlsr}) similar in structure to \texttt{RS-DisRL-M}, adopts value-type optimistic planning and confidence set construction. It uniquely constructs a version space for each augmented policy using the \texttt{V-Est} estimation algorithm, incorporating the actual optimal cumulative reward distribution. For case when augmented policy set $\Pi^\dag$ is infinite, we employ the \emph{policy covering argument}, discretizing the policy set using a normalized lower bracketing set $\underline{\Pi}$ defined in Definition \ref{def:bracket}. This approach, commonly used in prior studies \cite{kallus2022doubly,wang2023benefits,zhou2023offline,huang2022towards}, ensures the practicality and scalability of the algorithm, especially in complex policy environments (see Appendix \ref{sec:policy cover} for detailed discussion).

\paragraph{Theoretical Guarantees}
Similar to the Section~\ref{sec:mb}, we introduce two sufficient conditions which describe the validity of the estimation function \texttt{V-Est} to establish the theoretical result for \texttt{RS-DisRL-V}.

\begin{condition}[Concentration]
\label{con:mfconcentration}
With probability at least $ 1- \delta$, $\delta \in (0, 1]$, for all policy $\mpi \in \Pi^\dag$, we have that the actual random variable of the cumulative reward collected by policy $\mpi$ is in the confidence set with high probability, i.e., $Z^{\mpi} \in \wh{\gZ}_{k,\mpi}$.
\end{condition}

\begin{condition}[General elliptical potential]
\label{con:mfelliptical}
For $\delta \in (0, 1]$, the $\ell_\infty$ distance between the CDF of the optimistic reward distribution $\wh{Z}^k$ and the actual reward distribution $Z^{\mpi^k}$ under policy $\mpi^k$ can be bounded by $\sum_{k=1}^K \left\| F_{\wh{Z}^k} - F_{Z^{\mpi^k}} \right\|_\infty \leq \zeta(K, H, \gZ, \Pi^\dag, \gamma, \delta)$ with probability at least $1 - \delta$.
\end{condition}
These two conditions for \texttt{RS-DisRL-V}, paralleling those in \texttt{RS-DisRL-M}, encapsulate the efficiency of the estimation function \texttt{V-Est}. The theoretical result for \texttt{RS-DisRL-V} is presented below.
\begin{theorem}
\label{thm:mfmeta}
    Under Assumption~\ref{ass:mfbellman}, if the estimation function $\texttt{V-Est}$ satisfies Conditions \ref{con:mfconcentration} and \ref{con:mfelliptical}, then the regret of \texttt{RS-DisRL-V} can be bounded by $
        \operatorname{Regret(K)} \leq L_\infty(\rho)\zeta(K, H, \gZ, \Pi^\dag,\gamma, \delta)$.
\end{theorem}

The regret bound for the \texttt{RS-DisRL-V} algorithm is characterized by a form similar to that in the model-based case (Theorem~\ref{thm:mbmeta}). The effectiveness bound $\zeta(K,H,\gZ,\Pi^\dag,\gamma,\delta)$, sharing the same form as $\widetilde\gO(\poly(H)\operatorname{D}\sqrt{\gamma K / \delta})$ for both LSR (\texttt{V-Est-LSR}, Algorithm~\ref{alg:mfestlsr}) and MLE (\texttt{V-Est-MLE}, Algorithm~\ref{alg:mfestmle}) approaches, exhibiting a dependency of $\sqrt{K}$, indicating a sublinear complexity in terms of episodes. Here $\operatorname{D}$ is the structural complexity specified in  dependency Theorem~\ref{thm:mflsr} for LSR and Theorem~\ref{thm:mfmle} for MLE. 

\subsection{Estimation by Value-Based LSR Approach}\label{sec:mflsr}
In this section, 
we design a novel LSR approach \texttt{V-Est-LSR} for random-variable estimation through CDFs in the augmented MDP.
To the best of our knowledge, we are the first to present the statistically efficient LSR estimation for DisRL with general value function approximation. 

Denote $F_h(\cdot|s^\dag, a)$ as the CDF of $Z_h(\cdot | s^\dag, a) \in \gZ_h$. The estimation function \texttt{V-Est-LSR} focuses on estimating the Bellman operator with real transition probability. For a given target CDF $\widetilde{F}_{h+1}(\cdot | s^\dag, a)$ and $\mpi$, the data collection process
in episode $k$ gives us an empirical sample of $\gT_{h, \mpi}^\dag \widetilde{F}_{h+1}(x | s_{k,h}^\dag, a_{k,h})$ since we observe the transfer from $(s_{k,h}^\dag,a_{k,h})$ to $(s_{k,h+1}^\dag)$ with policy $\mpi^k$. Therefore, we can use $\pi_{h+1}^\top \widetilde{F}_{h+1}(x-r_{k,h}|s_{k,h+1}^\dag) := \int_{a_{h+1}}\pi_{h+1}(a_{h+1}|s_{k,h+1}^\dag)\widetilde{F}_{h+1}(x-r_{k,h}|s_{k,h+1}^\dag,a_{h+1})$ to perform an unbiased estimate of $\gT_{h, \mpi}^\dag \widetilde{F}_{h+1}(x|s_{k,h}^\dag,a_{k,h})$:
\begin{align*}
    \widehat{F}^\LSR_{k,h,\mpi, \widetilde{F}} =\argmin_{F_h \in \gZ_h} \sum_{i=1}^{k-1}\left( F_h(x_{k,h}^{\mpi, \widetilde{F}} | s_{k,h}^\dag, a_{k,h})  \right.  \left.- \pi_{h+1}^\top \widetilde{F}_{h+1}(x_{k,h}^{\mpi, \widetilde{F}} - r_{k,h} | s_{k,h+1}^\dag) \right) ^2\,,
\end{align*}
where $x_{i,h}^{{\mpi}, {\widetilde{F}}}$, defined in Eq.(\ref{eq:x_i,h mflsr}), intuitively leads the exploration direction with maximal uncertainty.
With this novel estimator, we can prove that with high probability.
\begin{align*}
    \sum_{i<k}\Big( &\gT_{h,\mpi}\widetilde{F}_{h+1}(x_{i,h}^{\mpi, \widetilde{F}} | s_{i,h}^\dag, a_{i,h})   -\wh{F}^\LSR_{k,h,\mpi, \widetilde{F}} (x_{i,h}^{\mpi, \widetilde{F}} | s_{i,h}^\dag, a_{i,h}) \Big)^2  \leq \widetilde{\gO}({\gamma^\LSR})\,.
\end{align*} 
By a union bound over the  covering $\underline{\Pi}^\dag$ and $\underline{\bm\gZ}$ (detailed in Appendix~\ref{sec:policy cover}), this property derives the concentration condition (Condition~\ref{con:mfconcentration}). With the concentration bound, we can easily establish the general elliptical potential condition (Condition~\ref{con:mfelliptical}) utilizing the similar argument in model-based LSR method (Section~\ref{sec:mblsr}). We propose the complete pseudocode of \texttt{V-Est-LSR} in Algorithm~\ref{alg:mfestlsr}. Theoretical guarantees are provided below.
\begin{theorem}
\label{thm:mflsr}
    Let $\gamma^\LSR := 16\log(HK^2/\delta) + \log(\gN_C(\Pi^\dag, 1/K,\|\cdot\|_1)) + \log(\gN_C(\bm\gZ, 1/K, \|\cdot\|_1))$. The estimation function \texttt{V-Est-LSR} satisfies the Conditions~\ref{con:mfconcentration} and \ref{con:mfelliptical} with $\zeta^\LSR=\widetilde{\gO}\left( \operatorname{poly}(H)\sqrt{K\gamma^\LSR\dim_E(\bm\gZ, \sqrt{K})} \right)$, where $\dim_E$ represents the eluder dimension.
\end{theorem}
The formal proof of Theorem \ref{thm:mflsr} is presented in Appendix~\ref{app:mflsr}. 
The \texttt{RS-DisRL-V} algorithm, when implemented with the \texttt{V-Est-LSR} estimation function attains a significant regret upper bound of RSRL with static LRM $\widetilde\gO(L_\infty(\rho)\poly(H)\sqrt{K\operatorname{D_{cov}}\dim_E(\bm\gZ, \sqrt{K})})$, where the covering dimension $\operatorname{D_{cov}} = \log(\gN_C(\Pi^\dag, 1/K,\|\cdot\|_1)) + \log(\gN_C(\bm\gZ, 1/K, \|\cdot\|_1))$. 
This bound, characterized by a $\sqrt{K}$ dependency signifies the first sample-efficient RSRL with general value function approximation and static LRM.
Furthermore, when degenerating to the risk-neutral and tabular case, this result aligns with optimal dependencies on $K$ as demonstrated in \cite{wang2020reinforcement}. 

\subsection{Estimation by Value-Based MLE Approach}\label{sec:mfmle}
For DisRL with general value function approximation, a standard estimation method adopted is estimating the candidate of true cumulative reward $Z^{\mpi}$ by its density function with MLE, which is powerful in theoretical studies \cite{wang2023benefits, wu2023distributional} and practice \cite{hessel2018rainbow, bellemare2017distributional}. Inspired by previous studies \cite{wang2023benefits, wu2023distributional}, we combine the standard MLE method with our general risk-sensitive model-free framework and provide the estimation function \texttt{V-Est-MLE} which performs efficient estimation in the augmented MDP and risk-sensitive target.

Assume $f_h(x|s_{h}^\dag, a_{h})$ is the PDF of $Z_{h}(s_h^\dag, a_h) \in \gZ_h$.
Inspired by the MLE method utilized in risk-neutral DisRL \cite{wu2023distributional,wang2023benefits}, we estimate by maximizing the log likelihood function: $\log f_h(z_{k,h}^{f,\pi}|s_{k,h},a_{k,h})$ where $z_{k,h}^{f,\pi}$ is sampled from $f_{h+1}(\cdot|s_{k,h+1}^\dag,\pi(s_{k,h+1}))+r_{k,h}$.
The details of \texttt{V-Est-MLE} are presented in Algorithm \ref{alg:mfestmle} in appendix due to the space limitation.

The following theorem addresses the efficiency of the MLE approach in risk-sensitive case. We denote $\operatorname{d_{BE}}$ as the Bellman eluder dimension (Definition~\ref{ass:bellman eluder dim}) that aligns with the approaches discussed in prior studies by \cite{jin2021bellman,wang2023benefits}.
\begin{theorem}
\label{thm:mfmle}
With $\gamma^\MLE = \log(\cN_{[\cdot]}(\bm\gZ,\epsilon,\norm{\cdot}_1))+\log(\cN_{[\cdot]}(\Pi^\dag,\epsilon,\norm{\cdot}_1))+\log(KH/\delta) $, estimation function \texttt{V-Est-MLE} satisfies Conditions \ref{con:mfconcentration} and \ref{con:mfelliptical} with $\zeta^\MLE=\widetilde{\gO}\bracket{\poly(H)\sqrt{\operatorname{d_{{BE}}} \gamma^\MLE K}}$. Here $\operatorname{d_{{BE}}}$ represents the Bellman eluder dimension, which is a common structural complexity studied in \cite{jin2021bellman, wang2023benefits}.
\end{theorem}

This result enables us to establish a regret upper bound for \texttt{RS-DisRL-V} of $\widetilde\gO(L_\infty(\rho)\poly(H)\sqrt{K\operatorname{d_{{BE}}}\operatorname{D_{cov}}})$, with the covering dimension $\operatorname{D_{cov}} = \log(\gN_{[\cdot]}(\Pi^\dag, 1/K,\|\cdot\|_1))+\log(\gN_{[\cdot]}(\bm\gZ, 1/K, \|\cdot\|_1))$. 
Notably, this bound aligns closely with results from existing research \cite{jin2021bellman} in the risk-neutral domain, demonstrating its relevance and applicability in a wide range of reinforcement learning contexts.

\section{Conclusion}

We give a comprehensive discussion of RS-DisRL with static LRM and general function approximation. We propose the model-based meta-algorithm \texttt{RS-DisRL-M} (Algorithm~\ref{alg:mbframe}) for model-based function approximation and the model-free meta-algorithm \texttt{RS-DisRL-V} (Algorithm~\ref{alg:mfframe}) for the general approximation of value functions. Equipped with our novel LSR or MLE estimation approaches, both meta-algorithms achieve the $\widetilde\gO(\sqrt{K})$ dependency of the regret upper bound, giving the first statistically efficient algorithms for RSRL with static LRM. 
Additionally, we establish a computationally tractable and statistically efficient algorithm in the specific setting with static CVaR risk measure and linear function approximation. In this case, we provide numerical experiments to validate the theoretical results (see Appendix~\ref{sec:experiments}).

\section{Impact Statements}

This paper presents work whose goal is to advance the field of Machine Learning. There are many potential societal consequences of our work, none which we feel must be specifically highlighted here.




\bibliography{ref}
\bibliographystyle{plain}

\newpage
\appendix
\onecolumn
\section{Notations}
Define $\Omega$ as the measurable space containing all the augmented trajectories $\mtau=\sets{s_1^\dag,a_1,\cdots,s_H^\dag,a_H}$. We consider the probability space $(\Omega, \Sigma, \Prob)$, where $\Sigma$ is the $\sigma$-algebra and $\Prob$ is the productive probability measure combine the transition distribution and reward distribution.

Let $Z_{\mtheta}^{{\mpi}}$ be the random variable of $\sum_{h=1}^H r_h$ defined on the $\sigma$-algebra $\Sigma$ of $\Omega$. Let $F_Z(x)$ be its cumulative distribution function (CDF). For an augmented trajectory $\mtau\in\Omega$ we denote $\mu_{\mtheta}^{\mpi}(\mtau)$ as the probability measure on $\mtau$ by following policy $\mpi$ under transition model $\mtheta$, i.e., for any augmented state action pair $s^\dag,a$,  
$$\mu_{\mtheta}^{\mpi}(s^\dag,a)=\int_{\mtau}\mu_{\mtheta}^{\mpi}(\mtau)\mathbf{1}(\tau_h=(s_h^\dag,a_h))$$
Since our policy $\mpi$ is Markov on the augmented MDP, the visitation $\mu$ admits a factorized structure: 
$$\mu_{\mtheta}^{\mpi}(\mtau)=\prod_{h=1}^H \pi_h(a_h|s_h^\dag)\TT_{\theta_h}(s_{h+1}^\dag|s_h^\dag,a_h)$$
We further define $\mu_{\mtheta}^{\mpi}(s^\dag,a)=\mu_{\mtheta}^{\mpi}(s^\dag)\mpi(a|s^\dag)$. 

We also denote $\nu$ as a probability measure defined on the original state action pairs $(s,a)$: 
$$\nu_{\mtheta}^{\mpi}(s,a)=\int_{\mtau}\mu_{\mtheta}^{\mpi}(\mtau)\mathbf{1}(s,a\in\mtau)=\int_y \mu_{\mtheta}^{\mpi}((s,y),a)$$
However, we remark that $\nu$ can not be factorized since our policy depends not only on the state $s$.

Then we introduce the standard concepts of the covering and bracketing numbers for a function class, which are widely employed in the analysis of general function approximation \cite{ayoub2020model,liu2022when,liu2023optimistic,wang2023benefits}.
\begin{definition}[Covering Number]\label{def:cover}
    The $\epsilon$-covering number of a set $\gV$ with metric $\rho$, denoted as $\cN_C(\gV,\epsilon,\rho)$, is the minimum integer $n$ such that there exists a subset $\cV_o\subset\cV$ with cardinality $n$, for every $x\in\cV$, there exists a $y\in\cV_o$, with $\rho(x,y)\leq \epsilon$
\end{definition}
\begin{definition}[Bracketing Number]\label{def:bracket}
    Let $\gG$ be a set of functions mapping $\gX \to \R$. Given $l, u \in \gG$ such that $l(x) \leq u(x)$ for all $x \in \gX$. We say that the \emph{bracket} $[l, u]$ is the set of functions $g \in \gG$ such that $l(x) \leq g(x) \leq u(x)$. for all $x \in \sX$. We call $[l, u]$ and $\epsilon$-bracket if $\|u - l \| \leq \epsilon$. Then the $\epsilon$-bracketing number of $\gG$ with respect to $\|\cdot\|$ denoted by $\gN_{[\cdot]}( \gG,\epsilon, \|\cdot\|)$ is the minimum number of $\epsilon$-brackets needed to cover $\gG$. And we denote ${\gG}^{\downarrow}$ as the set of the lower bracket functions $l$ of this $\epsilon$-brackets covering, i.e., $\left| {\gG}^{\downarrow} \right| =\gN_{[\cdot]}( \gG,\epsilon, \|\cdot\|)$
\end{definition}
Another important concept is the eluder dimension, which will be used to measure the structural complexity in the following LSR analysis. To introduce the eluder dimension, we first define the concept of $\varepsilon$-independence.
\begin{definition}[$\varepsilon$-dependence \cite{russo2013eluder}]
    For $\varepsilon > 0$ and function class $\mathcal{Z}$ whose elements are with domain $\mathcal{X}$, an element $x \in \mathcal{X}$ is $\varepsilon$-dependent on the set $\mathcal{X}_n := \{x_1, x_2, \cdots, x_n\}\subset \mathcal{X}$ with respect to $\mathcal{Z}$, if any pair of functions $z, z' \in \mathcal{Z}$ with $\sqrt{\sum_{i = 1}^n \left( z(x_i) - z'(x_i) \right)^2} \leq \varepsilon$ satisfies $z(x) - z'(x) \leq \varepsilon$.
    Otherwise, $x$ is $\varepsilon$-independent on $\mathcal{X}_n$ if it does not satisfy the condition.
\end{definition}

\begin{definition}[Eluder dimension \cite{russo2013eluder}]\label{def:eluder}
    For any $\varepsilon > 0$, and a function class $\mathcal{Z}$ whose elements are in domain $\mathcal{X}$, the Eluder dimension $\dim_E(\mathcal{Z}, \varepsilon)$ is defined as the length of the longest possible sequence of elements in $\mathcal{X}$ such that for some $\varepsilon' \geq \varepsilon$, every element is $\varepsilon'$-independent of its predecessors.
\end{definition}

\section{General Model-based framework: Algorithm \texttt{RS-DisRL-M}}

In our model-based framework for Risk-Sensitive Distributional Reinforcement Learning (RS-DisRL), we focus on estimating the transition model, denoted as $\widehat{\mtheta}_k$ for each episode $k$. This involves leveraging historical data up to episode $k-1$, represented as $\gH_{k-1}$,to construct a confidence set $\widehat{\mTheta}_k$. The construction of this set is crucial for guiding the selection of actions, as it is based on a specified confidence radius $\beta$, which helps in balancing exploration and exploitation by considering the uncertainty in our model estimates

\begin{algorithm}[htbp]
   \caption{\texttt{RS-DisRL-M}}
\begin{algorithmic}[1]
   \STATE {\bfseries Input:} Model class $\mTheta$, confidence radius $\beta$.
   \STATE {\bfseries Initialize:} $\wh{\mTheta}_1 \leftarrow \mTheta$.
   \FOR{$k=1$ {\bfseries to} $K$}
   \STATE 
   $(\mpi^{k}, \widehat\mtheta_k)=\argmax_{\mpi\in\bm{\Pi}^\dag, \mtheta\in\widehat{\bm{\Theta}}_k}\rho(Z^{\mpi}_{\mtheta})$.~\textcolor{blue}{//Optimistic planning}
   \STATE Execute policy $\mpi^{k}$, add the collected data $\bm{\tau}_k=\sets{(s_{k,h},a_{k,h},r_{k,h})}_{h=1}^H$ and $\mpi^{k}$, $\widehat{\mtheta}_k$ into history $\cH_k = \cH_{k-1}\cup \{(\bm{\tau}_k, \mpi^{k}, \widehat{\mtheta}_k)\}$.~\textcolor{blue}{//Data collection}
   \STATE $\widehat{\bm\Theta}_{k+1}=\texttt{M-Est}\bracket{\bm{\Theta},\cH_k,\beta}$.~\textcolor{blue}{//Confidence set construction}
\ENDFOR
\end{algorithmic}
\end{algorithm}
To ensure the regret bound, we require the following conditions: Condition \ref{con:mbconcentration} and Condition \ref{con:mbelliptical}.
\begin{condition}[Concentration condition]\label{ass:concentration}
    For $\delta \in (0, 1]$, we have $\mtheta^*\in\widehat\mTheta_k$ holds for any $k\in[K]$, with probability at least $1 - \delta$.
\end{condition}

\begin{condition}[Elliptical potential condition]\label{ass:pigeon-hole}
    If for any $k\in[K]$, we have for any given $\{\widehat{\mtheta}_k\}_k \subset \{\widehat\mTheta_k\}_k$ and corresponding greedy policy $\mpi^k = \argmax_{\mpi \in \Pi^\dag}\rho(Z_{\mtheta_k}^\mpi)$, the $L_p$-norm f the difference of reward-to-gos' CDFs for chosen model $\mtheta_k$ and true model $\mtheta^*$ can be bound by
    \begin{align*}
        \sum_{k=1}^K \left\|F_{Z_{\mtheta_k}^{\mpi_k}}-F_{Z_{\mtheta^*}^{\mpi_k}}\right\|\leq \xi\bracket{K,H,d_{\Theta},\beta,\delta}
    \end{align*}
    with probability at least $1 - \delta$, $\delta > 0$.
\end{condition}

This general framework emphasizes that the key to efficiently learning the MDP with a static Lipschitz risk measure is centered on constructing a confidence set for the transition model. This construction leverages the elliptical potential principle for cumulative distribution functions within the augmented MDP. Combined with above conditions, we can establish the following theoretical result

\begin{theorem}~\label{thm:appmbmeta}
    Under Assumption \ref{ass:mbreal}, if the estimation function \texttt{M-Est} satisfies Conditions~\ref{con:mbconcentration} and \ref{con:mbelliptical}, then the regret of \texttt{RS-DisRL-M} (Algorithm~\ref{alg:mbframe}) can be bounded by $
        \operatorname{Regret}(K) \leq L_\infty(\rho)\xi(K, H, \mTheta, \beta, \delta)$ with probability at least $1 - 2\delta$.
\end{theorem}

\section{Model Based Estimation by LSR Approach}\label{app:mdlsr}

In this section, we design a Least Squares Regression (LSR) based estimation method to construct the confidence set of the model at each episode, and theoretically demonstrate that our algorithm \texttt{M-Est-LSR} satisfies the Conditions~\ref{con:mbconcentration} and \ref{con:mbelliptical}.

First we introduce some notations for simplicity. We define $Z_h^{\mpi, \mtheta}(s_h,y_h) = \sum_{i=h}^H r_i$ as the random variable of the reward-to-go from step $h$, where $s_h$ is the starting state, and $y_h$ is the previous cumulative reward from step $1$ to $h - 1$. Moreover, we denote $F^{\mpi, \mtheta}_h(x | s_h, y_h)$ as the CDF of $Z^{\mpi, \mtheta}_h$. Our analysis for LSR approach in model-based function approximation is inspired by \cite{chen2023provably}. However, \cite{chen2023provably} focus on Iterated CVaR risk measure and analyse the model with value function in its general function approximation algorithm. In this paper, we develop novel technique for distribution function analysis for general model-based function approximation and augmented MDP.
\subsection{Algorithm \texttt{M-Est-LSR}}

\begin{algorithm}[htbp]
   \caption{\texttt{M-Est-LSR}$(\mTheta, \gH_{k-1}, \beta^{\LSR})$}
   \label{alg:mbestlsr}
\begin{algorithmic}
   \STATE {\bfseries Input:} History information $\gH_{k-1}$, Model class $\mTheta$, and confidence radius $\beta^{\LSR}$.
   \STATE Estimate the transition model for every $h\in [H]$
\begin{align*}
\theta^{\LSR}_{k,h}=\argmin_{\theta_h\in\Theta_h} \sum_{i=1}^{k-1}\Big(&\int_{r}\sR(r|s_{i,h},a_{i,h}){F}^{\wh{\mpi}^i, \wh{\mtheta}_i}_{h+1}(x_{i,h} - r | s_{i,h+1}, y_{i,h}+r) \\
&- \int_{s_{h+1}^\dag} \TT_{\theta_h}(s_{h+1}^\dag |s_{i,h}^\dag, a_{i,h}){F}^{\widehat{\mpi}^i, \wh{\mtheta}_i}_{h+1}(x_{i,h} - (y_{h+1} - y_{i,h}) | s_{h+1}, y_{h+1} )\Big)^2\,.
\end{align*}
    \STATE Construct Confidence set:
\begin{align*}
    \wh{\Theta}_{k,h} := \sets{\theta_h \in \Theta_h : \dist_{\cH_{k-1}, h}^{\LSR}(\theta_h||{\theta}^{\LSR}_{k,h})\leq\beta^\LSR}
\end{align*}
\begin{align*}
\wh{\mTheta}_k = \sets{\mtheta \in \mTheta : \theta_h \in \wh{\Theta}_{k,h}, h \in [H]}
\end{align*}
\STATE {\bfseries Return} $\widehat{\mTheta}_k$
\end{algorithmic}
\end{algorithm}
where the distance function is defined by
\begin{equation}
\begin{aligned}
    \dist_{\cH_k}^{\LSR}(\theta_1 || \theta_2) = \sum_{i=1}^{k}\Big(&\int_{s'} \Prob_{\theta_1}(s'|s_{k,h}, a_{k,h})\int_{r}\sR(r|s_{i,h},a_{i,h}){F}^{\widehat{\mpi}^i, \wh{\mtheta}_i}_{h+1}( x_{i,h} - r|s', y_{i,h}+r) \\
    &- \int_{s'} \Prob_{\theta_2}(s'|s_{k,h}, a_{k,h})\int_{r}\sR(r|s_{i,h},a_{i,h}){F}^{\widehat{\mpi}^i, \wh{\mtheta}_i}_{h+1}(x_{i,h} - r|s', y_{i,h}+r)\Big)^2 \,,
\end{aligned}
\end{equation}
$x_{k,h}$ is defined by:
\begin{align}\label{eq:mblsrxih}
    x_{i,h}=&\argmax_{x\in\RR} \sup_{\theta^1_h \in \wh{\Theta}_{k,h}} 
        \int_{s_{h+1}^\dag} \TT_{{\theta}^1_h}(s_{h+1}^\dag | s_{k,h}^\dag, a_{k,h})F^{\mpi^k,\wh{\mtheta}_k}_{h+1}(x - (y_{h+1} - y_{k,h})|s_{h+1}^\dag) \notag\\
        &{\quad \quad \quad\quad\quad- \inf_{\theta^2_h \in \wh{\Theta}_{k,h}}\int_{s_{h+1}^\dag} \TT_{{\theta}^2_{h}}(s_{h+1}^\dag | s_{k,h}^\dag, a_{k,h})F^{\mpi^k,\wh{\mtheta}_k}_{h+1}(x - (y_{h+1} - y_{k,h})|s_{h+1}^\dag) }\,,
\end{align}
which represents the direction of maximum uncertainty in confidence set $\widehat{\Theta}_{k,h}$,
we can obtain the least-squares estimate as:
\begin{align*}
\theta^{\LSR}_{k,h}=&\argmin_{\theta_h\in\Theta_h} \sum_{i=1}^{k-1}\Bigg(\int_{r}\sR_h(r|s_{i,h},a_{i,h}){F}^{\wh{\mpi}^i, \wh{\mtheta}_i}_{h+1}(x_{i,h} - r | s_{i,h+1}, y_{i,h}+r) \\
&\quad \quad - \int_{s_{h+1}^\dag} \TT_{\theta_h}(s_{h+1}^\dag |s_{i,h}^\dag, a_{i,h}){F}^{\widehat{\mpi}^i, \wh{\mtheta}_i}_{i,h+1}(x_{i,h} - (y_{h+1} - y_{i,h}) | s_{h+1}, y_{h+1} )\Bigg)^2 \\
=&\argmin_{\theta_h\in\Theta_h} \sum_{i=1}^{k-1}\Bigg(\int_{r}\sR_h(r|s_{i,h},a_{i,h}){F}^{\wh{\mpi}^i, \wh{\mtheta}_i}_{h+1}(x_{i,h} - r | s_{i,h+1}, y_{i,h}+r) \\
&\quad \quad - \int_{s_{h+1}} \Prob_{\theta_h}(s_{h+1} |s_{i,h}, a_{i,h})\int_r \sR_h(r | s_{i,h}, a_{i,h}){F}^{\widehat{\mpi}^i, \wh{\mtheta}_i}_{h+1}(x_{i,h} - r | s_{h+1}, y_{i,h} + r )\Bigg)^2\,.
\end{align*}

In the following prood, we show that  with $\beta^\LSR = 8\log(2H^2\gN_C(\mTheta,  1/K,\norm{\cdot}_1)/\delta) + 4\sqrt{\log(4HK^2/\delta)}$, we have the concentration condition holds with probability at least $1 - \delta$. And the elliptical potential condition holsd for 
$$\xi^\LSR(K, H, \Theta, \beta^\LSR, \delta) = O(H\sqrt{K}\cdot \sqrt{1+\dim_E(\gW_\mTheta, 1/\sqrt{K}) + \dim_E(\gW_\mTheta, 1/\sqrt{K})\beta^\LSR \log K} + H\sqrt{2K\log(1/\delta)})$$
wher $\mathcal{W}_{\mTheta} := \{\Prob_{\theta_h}F : \gS \times \gA \to [0, 1] : \mtheta \in \mTheta, F : \gS \to [0, 1] \}$ and $\dim_E$ represents the eluder dimension.

\subsection{Least Squares Form for Estimation}

Notice that we first calculate the estimator kernel 
\begin{align*}
\theta^{\LSR}_{k,h}=&\argmin_{\theta_h\in\Theta_h} \sum_{i=1}^{k-1}\Bigg(\int_{r}\sR_h(r|s_{i,h},a_{i,h}){F}^{\wh{\mpi}^i, \wh{\mtheta}_i}_{h+1}(x_{i,h} - r | s_{i,h+1}, y_{i,h}+r)\\
&\quad \quad - \int_{s_{h+1}^\dag} \TT_{\theta_h}(s_{h+1}^\dag |s_{i,h}^\dag, a_{i,h}){F}^{\widehat{\mpi}^i, \wh{\mtheta}_i}_{i,h+1}(x_{i,h} - (y_{h+1} - y_{i,h}) | s_{h+1}, y_{h+1} )\Bigg)^2 \\
=&\argmin_{\theta_h\in\Theta_h} \sum_{i=1}^{k-1}\Bigg(\int_{r}\sR_h(r|s_{i,h},a_{i,h}){F}^{\wh{\mpi}^i, \wh{\mtheta}_i}_{h+1}(x_{i,h} - r | s_{i,h+1}, y_{i,h}+r) \\
&\quad \quad - \int_{s_{h+1}} \Prob_{\theta_h}(s_{h+1} |s_{i,h}, a_{i,h})\int_r \sR_h(r | s_{i,h}, a_{i,h}){F}^{\widehat{\mpi}^i, \wh{\mtheta}_i}_{h+1}(x_{i,h} - r | s_{h+1}, y_{i,h} + r )\Bigg)^2\,.
\end{align*}
which takes the least-square regression form. 
If we denote the mixed contribution function
\begin{equation*}
    \wh{F}_{i,h}(s) := \int_r\sR_h(r|s_{i,h},a_{i,h}) F^{\wh{\mpi}^i, \wh{\mtheta}_i}_{h+1}(x_{i,h} - r | s, y_{i,h} + r)\,.
\end{equation*}
Thus we can simplify the least squares regression as
\begin{align*}
\theta^{\LSR}_{k,h}
=&\argmin_{\theta_h\in\Theta_h} \sum_{i=1}^{k-1}\left(\wh{F}_{i,h}(s_{i,h+1}) - \int_{s_{h+1}} \Prob_{\theta_h}(s_{h+1} |s_{i,h}, a_{i,h})\wh{F}_{i,h}(s_{h+1})\right)^2\,.
\end{align*}

We can further define 
\begin{align*}
    [\Prob_{\theta_h}\wh{F}_{i,h}](s_h,a_h):= \int_{s_{h+1}} \Prob_{\theta_h}(s_{h+1} |s_{h}, a_{h})\wh{F}_{i,h}(s_{h+1})
\end{align*}
and $\mathcal{W}_{\mTheta} := \{\Prob_{\theta_h}F : \gS \times \gA \to [0, 1] : \mtheta \in \mTheta, F : \gS \to [0, 1] \}$. 

Then the distance function and our constructed confidence sets can be expressed by
\begin{align*}
    \dist_{\cH_k, h}^{\LSR}(\theta_1 || \theta_2) = & {\sum_{i=1}^{k}\left([\Prob_{\theta_1}\wh{F}_{i,h}](s_h,a_h) - [\Prob_{\theta_2}\wh{F}_{i,h}](s_h,a_h)\right)^2}\,,
\end{align*}
\begin{align*}
    \wh{\Theta}_{k,h} := \sets{\theta_h \in \Theta_h : \dist_{\cH_{k-1}, h}^{\LSR}(\theta_h||\widehat{\theta}^{\LSR}_{k,h})\leq\beta^\LSR}\,,
\end{align*}
\begin{equation*}
\wh{\mTheta}_k = \sets{\mtheta \in \mTheta : \theta_h \in \wh{\Theta}_{k,h}, h \in [H]}\,.
\end{equation*}

\subsection{Concentration Condition for LSR approach}

\begin{lemma}[LSR concentration]\label{lem:mblsr concentration}
    The LSR-type construction algorithm \texttt{M-Est-LSR} satisfies Condition~\ref{con:mbconcentration}. That is, for $\delta \in (0, 1]$, with probability at least $1 - \delta$, we have $\mtheta_h^* \in \wh{\mTheta}_{k}$  for all $k \in [K]$, 
\begin{proof}
    Recall that we calculate the estimation kernel $\theta^\LSR_{k,h}$ by least squares regression as follows:
    \begin{align*}
        \theta^{\LSR}_{k,h}=\argmin_{\theta_h\in\Theta_h} \sum_{i=1}^{k-1}\left(\wh{F}_{i,h}(s_{i,h+1}) - [\Prob_{\theta_h}\wh{F}_{i,h}](s_{i,h}, a_{i,h})\right)^2\,.
    \end{align*}
    Notice that the $\wh{F}_{i,h}(s_{i,h+1})$ is $\sigma_{i,h+1}$-measurable by definition and $[\Prob_{\theta_h}\wh{F}_{i,h}](s_{i,h}, a_{i,h})$ is $\sigma_{i,h}$-measurable, with $\sigma_{i,h}$ be the filtration containing history up to the $h$ step in episode $i$.
    We have
    \begin{equation*}
        \EE\left[ \wh{F}_{i,h}(s_{i,h+1}) \middle| \sigma_{i,h} \right] = [\Prob_{\theta_h^*}\wh{F}_{i,h}](s_{i,h}, a_{i,h})\,.
    \end{equation*}
    Based on concentration lemma~\ref{lem:model free lsr auxillary concentration}, we have the following holds with probability at least $1 - \delta/H$
    \begin{align*}
        \sum_{i=1}^k\left([\Prob_{\theta_h^*}\wh{F}_{i,h}](s_{i,h},a_{i,h}) - [\Prob_{\theta_{k,h}^\LSR}\wh{F}_{i,h}](s_{i,h},a_{i,h}) \right)^2 \leq 8\log(2H\gN(\mathcal{W}_\mTheta, 1/K,\|\cdot\|_\infty)/\delta) + 4\sqrt{\log(4K^2/\delta)}\,.
    \end{align*}
    Moreover, we have for any $\mtheta^1, \mtheta^2 \in \mTheta$, we can bound the supremum distance of $\Prob_{\theta^1_h}F , \Prob_{\theta^2_h}F \in \mathcal{W}_\mTheta$ for any $h \in [H]$ by
    \begin{align*}
        \|\Prob_{\theta^1_h}F - \Prob_{\theta^2_h}F\|_\infty = &\sup_{(s, a) \in \gS \times \gA} \left| \int_{s'}(\Prob_{\theta^1_h}(s'|s,a)F(s') - \Prob_{\theta^2_h}(s'|s,a)F(s')) \right| \\
        =&\sup_{(s, a) \in \gS \times \gA} \left| \int_{s'} (\Prob_{\theta^1_h}(s'|s,a) -  \Prob_{\theta^2_h}(s'|s,a))F_2(s') \right| \\
        \leq &\sup_{(s,a) \in \gS \times \gA}\|\Prob_{\theta^1_h}(s,a) - \Prob_{\theta^2_h}(s,a)\|_1 \\
        \leq &\|\mtheta^1 - \mtheta^2\|_1\,.
    \end{align*}
    Thus we have $\gN_C(\mathcal{W}_\mTheta, 1/K,\|\cdot\|_\infty) \leq \gN_C(\mTheta,  1/K,\norm{\cdot}_1)$. Recall the definition of $\dist^\LSR_{\gH_{k-1}, h}$. Taking union bound over $h \in [H]$, we have
    \begin{align*}
        \dist^\LSR_{\gH_{k-1}, h}(\theta^*_h || \theta^\LSR_{k,h}) \leq 8\log(2H^2\gN_C(\mTheta,  1/K,\norm{\cdot}_1)/\delta) + 4\sqrt{\log(4HK^2/\delta)} = \beta^\LSR\,,
    \end{align*}
    which shows that $\theta^*_h \in \wh{\Theta}_{k,h}$ for every $k, h \in [K] \times [H]$ with probability at least $ 1- \delta$ and implies that $\mtheta^* \in \wh{\mTheta}_k$ for every $k\in[K]$ with probability at least $1- \delta$.
\end{proof}
\end{lemma}

\subsection{Elliptical Potential Condition for LSR Approach}

We have the bellman equation for distributional function
\begin{align*}
    F_h^{\mpi, \mtheta}(x | s_h, y_h) =  F_h^{\mpi, \mtheta}(x | s_h^\dag) = \int_{s_{h+1}^\dag} \int_{a_h} \pi_h(a_h | s_{h}^\dag)\TT_{\theta_h}(s_{h+1}^\dag | s_{h}^\dag,a_h) F_{h+1}^{\mpi, {\mtheta}}(x - (y_{h+1} - y_h)| s_{h+1}, y_{h+1})\,.
\end{align*}
\begin{align*}
    \gT_{\theta_h,\mpi}F^{\mpi,\mtheta}_{h+1}(x|s_h^\dag)=\int_{s_{h+1}^\dag} \int_{a_h} \pi_h(a_h | s_{h}^\dag)\TT_{\theta_h}(s_{h+1}^\dag | s_{h}^\dag,a_h) F_{h+1}^{\mpi, {\mtheta}}(x - (y_{h+1} - y_h)| s_{h+1}, y_{h+1})\,.
\end{align*}

\begin{lemma}\label{lem:mblsrellitical}
    With probability at least $1 - \delta$, we have 
    \begin{align*}
        &\sum_{k=1}^K\sum_{h=1}^H \left(\sup_{\theta^1_h \in \wh{\Theta}_{k,h}} [\Prob_{\theta^1_h}\wh{F}_{k,h}](s_{k,h},a_{k,h}) - \inf_{\theta^2_h \in \wh{\Theta}_{k,h}} [\Prob_{\theta^2_h}\wh{F}_{k,h}](s_{k,h},a_{k,h}) \right)^2\\
        \leq& H + H\dim_E(\gW_\mTheta, 1/\sqrt{K}) + 4H\beta^\LSR\dim_E(\gW_\mTheta, 1/\sqrt{K})(\log(K) + 1)\,.
    \end{align*}
\begin{proof}
    This proof is almost same with the elliptical potential lemma for general function approximation given in Lemma 9 of \cite{chen2023provably}. 
\end{proof}
\end{lemma}

\begin{lemma}[LSR elliptical potential]
    The algorithm \texttt{M-Est-LSR} satisfies Condition~\ref{con:mbelliptical} with
    \begin{align*}
        \xi^\LSR(K,H,\mTheta, \beta^\LSR, \delta) = \widetilde{\mathcal{O}}\left(H\sqrt{K\beta^\LSR\dim_E(\gW_\mTheta, 1/\sqrt{K})}\right)\,.
    \end{align*}
\begin{proof}
    We have
    \begin{align*}
        &\sum_{k=1}^K \sup_{x\in[0,H]} \left|F_{Z^{{\mpi}^k}_{\wh{\mtheta}_k}}(x) - F_{Z^{{\mpi}^k}_{{\mtheta}^*}}(x)\right| \\
        = &\sum_{k=1}^K \sup_{x\in[0,H]}\left| F_1^{\mpi^k, \wh{\mtheta}_k}(x | s_{k,1}, 0) - F_1^{\mpi^k, {\mtheta}^*}(x | s_{k,1}, 0)  \right| \\
        = &\sum_{k=1}^K \sup_{x\in[0,H]}\left| 
        \gT_{ \wh{\theta}_{k,1},\mpi^k}F_2^{\mpi^k,\wh{\mtheta}_k}(x | s_{k,1}, 0) - \gT_{\theta^*_1,\mpi^k}F_2^{\mpi^k,\mtheta^*}(x|s_{k,1},0) \right| \\
        \leq &\sum_{k=1}^K \sup_{x\in[0,H]}\left| 
         \gT_{ \wh{\theta}_{k,1},\mpi^k}F_2^{\mpi^k,\wh{\mtheta}_k}(x | s_{k,1}^\dag) - \gT_{\theta^*_1,\mpi^k}F_2^{\mpi^k,\wh{\mtheta}_k}(x|s_{k,1}^\dag) \right| \\
        &+ \sum_{k=1}^K \sup_{x\in[0,H]}\left| \gT_{\theta^*_1,\mpi^k}F_2^{\mpi^k,\wh{\mtheta}_k}(x|s_{k,1}^\dag) - \gT_{\theta^*_1,\mpi^k}F_2^{\mpi^k,\mtheta^*}(x|s_{k,1}^\dag) \right| \\
        \leq &\sum_{k=1}^K  \sup_{x\in[0,H]} \Bigg| 
        \int_{s_{2}^\dag} \TT_{\wh{\theta}_{k,1}}(s_{2}^\dag | s_{k,1}^\dag, a_{k,1})F^{\mpi^k,\wh{\mtheta}_k}_{2}(x - (y_{2} - y_{k,1})|s_{2}^\dag) \\
        &{\quad \quad \quad \quad  \quad - \int_{s_{2}^\dag} \TT_{{\theta}_{1}^*}(s_{2}^\dag | s_{k,1}^\dag, a_{k,1})F^{\mpi^k,\wh{\mtheta}_k}_{2}(x - (y_{2} - y_{k,1})|s_{2}^\dag)\Bigg| }\\
        &+\sum_{k=1}^K \sup_{x\in[0,H]} \left| F_2^{\mpi^k,\wh{\mtheta}_k}(x|s_{k,2}^\dag) - F_2^{\mpi^k,\mtheta^*}(x|s_{k,2}^\dag) \right| \\
        &+\sum_{k=1}^K \Delta_{k,1}^{(1)} + \Delta_{k,1}^{(2)} + \Delta_{k,1}^{(3)} + \Delta_{k,1}^{(4)}\,,
    \end{align*}
    where the sequence $\Delta_{k,h}^{(i)}$ for $i = 1, 2, 3, 4$ is defined as follows.
    \begin{align*}
        \Delta_{k,h}^{(1)} = \sup_{x\in[0,H]} \Bigg| 
        &\int_{s_{h+1}^\dag} \TT_{\wh{\theta}_{k,h}}(s_{h+1}^\dag | s_{k,h}^\dag, a_{k,h})F^{\mpi^k,\wh{\mtheta}_k}_{h+1}(x - (y_{h+1} - y_{k,h})|s_{h+1}^\dag) -  \gT_{\wh{\theta}_{k,h},\mpi^k}F_{h+1}^{\mpi^k,\wh{\mtheta}_k}(x | s_{k,h}^\dag) \Bigg|
        \\
        = \sup_{x\in[0,H]}\Bigg| 
        &\int_{s_{h+1}^\dag} \TT_{\wh{\theta}_{k,h}}(s_{h+1}^\dag | s_{k,h}^\dag, a_{k,h})F^{\mpi^k,\wh{\mtheta}_k}_{h+1}(x - (y_{h+1} - y_{k,h})|s_{h+1}^\dag) \\
        &-  \int_{a_{h}}\pi^k_h(a_h|s_{k,h}^\dag)\int_{s_{h+1}^\dag} \TT_{\wh{\theta}_{k,h}}(s_{h+1}^\dag | s_{k,h}^\dag, a_{h})F^{\mpi^k,\wh{\mtheta}_k}_{h+1}(x - (y_{h+1} - y_{k,h})|s_{h+1}^\dag, y_{h+1})\Bigg|\,,
    \end{align*}
    \begin{align*}
        \Delta_{k,h}^{(2)} = \sup_{x\in[0,H]}\Bigg| 
        &\int_{s_{h+1}^\dag} \TT_{{\theta}_{h}^*}(s_{h+1}^\dag | s_{k,h}^\dag, a_{k,h})F^{\mpi^k,\wh{\mtheta}_k}_{h+1}(x - (y_{h+1} - y_{k,h})|s_{h+1}^\dag) -  \gT_{{\theta}_{h}^*,\mpi^k}F_{h+1}^{\mpi^k,\wh{\mtheta}_k}(x | s_{k,h}^\dag) \Bigg|
        \\
        = \sup_{x\in[0,H]}\Bigg| 
        &\int_{s_{h+1}^\dag} \TT_{{\theta}_{h}^*}(s_{h+1}^\dag | s_{k,h}^\dag, a_{k,h})F^{\mpi^k,\wh{\mtheta}_k}_{h+1}(x - (y_{h+1} - y_{k,h})|s_{h+1}^\dag) \\
        &-  \int_{a_{h}}\pi^k_h(a_h|s_{k,h}^\dag)\int_{s_{h+1}^\dag} \TT_{{\theta}_{h}^*}(s_{h+1}^\dag | s_{k,h}^\dag, a_{h})F^{\mpi^k,\wh{\mtheta}_k}_{h+1}(x - (y_{h+1} - y_{k,h})|s_{h+1}^\dag)\Bigg|\,,
    \end{align*}
    \begin{align*}
    \Delta_{k,h}^{(3)} = \sup_{x\in[0,H]}\left| \gT_{\theta^*_h,\mpi^k}F_{h+1}^{\mpi^k,\wh{\mtheta}_k}(x|s_{k,h}^\dag) - F_{h+1}^{\mpi^k,\wh{\mtheta}_k}(x- r_{k,h}|s_{k,h+1}^\dag)\right|\,,
    \end{align*}
    \begin{align*}
    \Delta_{k,h}^{(4)} = \sup_{x\in[0,H]}\left| \gT_{\theta^*_h,\mpi^k}F_{h+1}^{\mpi^k,{\mtheta}^*}(x|s_{k,h}^\dag) - F_{h+1}^{\mpi^k,{\mtheta}^*}(x - r_{k,h}|s_{k,h+1}^\dag)\right|\,.
    \end{align*}
    Thus we have $\EE[\Delta_{k,h}^{(i)} | \sigma_{k,h}] = 0$ for any $k \in [K], h \in [H]$ and $i \in \{1, 2, 3, 4\}$. Thus we have $\{\Delta_{k,h}^{(i)}\}_{i=1}^k$ is a martingale difference sequence. Repeat the above method for $h$ steps,
    \begin{align*}
        &\sum_{k=1}^K \left\| F_{Z^{{\mpi}^k}_{\wh{\mtheta}_k}} - F_{Z^{{\mpi}^k}_{{\mtheta}^*}} \right\|_\infty \\
        \leq &\sum_{k=1}^K \sum_{h=1}^H  \sup_{x\in[0,H]} \Bigg| 
        \int_{s_{h+1}^\dag} \TT_{\wh{\theta}_{k,h}}(s_{h+1}^\dag | s_{k,h}^\dag, a_{k,h})F^{\mpi^k,\wh{\mtheta}_k}_{h+1}(x - (y_{h+1} - y_{k,h})|s_{h+1}^\dag) \\
        &\underbrace{\quad \quad \quad \quad \quad  \quad - \int_{s_{h+1}^\dag} \TT_{{\theta}_{h}^*}(s_{h+1}^\dag | s_{k,h}^\dag, a_{k,h})F^{\mpi^k,\wh{\mtheta}_k}_{h+1}(x - (y_{h+1} - y_{k,h})|s_{h+1}^\dag)\Bigg| }_I\\
        &+\underbrace{\sum_{k=1}^K \sum_{h=1}^H\Delta_{k,h}^{(1)} + \Delta_{k,h}^{(2)} + \Delta_{k,h}^{(3)} + \Delta_{k,h}^{(4)}}_J\,.
    \end{align*}
Applying the standard Azuma-Hoeffding inequality to the martingale difference sequence, we have 
\begin{align*}
    J \leq \widetilde{\mathcal{O}}(H\sqrt{K})
\end{align*}
The main challenge falls in bounding the term $I$. By Lemma~\ref{lem:mblsr concentration}, we have with probability at least $1 -\delta$, $\theta_{h}^* \in \widehat{\Theta}_{k,h}$ holds for all $k \in[K]$ and $h \in [H]$. Therefore,
\begin{align*}
    I \leq 
    &\sum_{k=1}^K \sum_{h=1}^H  \sup_{x\in[0,H]} \sup_{\theta^1_h \in \wh{\Theta}_{k,h}} 
        \int_{s_{h+1}^\dag} \TT_{{\theta}^1_h}(s_{h+1}^\dag | s_{k,h}^\dag, a_{k,h})F^{\mpi^k,\wh{\mtheta}_k}_{h+1}(x - (y_{h+1} - y_{k,h})|s_{h+1}^\dag) \\
        &{\quad \quad \quad\quad\quad- \inf_{\theta^2_h \in \wh{\Theta}_{k,h}}\int_{s_{h+1}^\dag} \TT_{{\theta}^2_{h}}(s_{h+1}^\dag | s_{k,h}^\dag, a_{k,h})F^{\mpi^k,\wh{\mtheta}_k}_{h+1}(x - (y_{h+1} - y_{k,h})|s_{h+1}^\dag) } \\
        =&\sum_{k=1}^K \sum_{h=1}^H  \sup_{\theta^1_h \in \wh{\Theta}_{k,h}} 
        \int_{s_{h+1}^\dag} \TT_{{\theta}^1_h}(s_{h+1}^\dag | s_{k,h}^\dag, a_{k,h})F^{\mpi^k,\wh{\mtheta}_k}_{h+1}(x_{k,h} - (y_{h+1} - y_{k,h})|s_{h+1}^\dag) \\
        &{\quad \quad \quad- \inf_{\theta^2_h \in \wh{\Theta}_{k,h}}\int_{s_{h+1}^\dag} \TT_{{\theta}^2_{h}}(s_{h+1}^\dag | s_{k,h}^\dag, a_{k,h})F^{\mpi^k,\wh{\mtheta}_k}_{h+1}(x_{k,h} - (y_{h+1} - y_{k,h})|s_{h+1}^\dag) } \\
        =&\sum_{k=1}^K\sum_{h=1}^H \sup_{\theta^1_h \in \wh{\Theta}_{k,h}} [\Prob_{\theta^1_h}\wh{F}_{k,h}](s_{k,h},a_{k,h}) - \inf_{\theta^2_h \in \wh{\Theta}_{k,h}} [\Prob_{\theta^2_h}\wh{F}_{k,h}](s_{k,h},a_{k,h})\,,
\end{align*}
where the first inequality holds by $\wh{\mtheta}_k \in \wh{\mTheta}_k$ and $\mtheta^* \in \wh{\mTheta}_k$ with high probability, the first equality is due to the definition of $x_{k,h}$, and the last equality holds by the definition of $\wh{F}_{k,h}(s)$. By Lemma~\ref{lem:mblsrellitical} and Cauchy-Schwartz inequality, we have
\begin{align*}
    I \leq \widetilde{\gO}\left(H\sqrt{K\beta^\LSR\dim_E(\gW_\mTheta, 1/\sqrt{K})}\right)\,.
\end{align*}
Overall, we can conclude the result
\begin{align*}
    \sum_{k=1}^K \left\| F_{Z^{{\mpi}^k}_{\wh{\mtheta}_k}} - F_{Z^{{\mpi}^k}_{{\mtheta}^*}} \right\|_\infty \leq \widetilde{\mathcal{O}}\left(H\sqrt{K\beta^\LSR\dim_E(\gW_\mTheta, 1/\sqrt{K})}\right)\,.
\end{align*}
\end{proof}
\end{lemma}

\section{Model-Based Estimation by MLE Approach}\label{app:mbmle}
In this section we propose our algorithm and analysis for model-based risk-sensitive RL via the MLE estimation approach. 
\subsection{Algorithm \texttt{M-Est-MLE}}


\begin{algorithm}[htbp]
   \caption{\texttt{M-Est-MLE}$(\mTheta, \gH_{k-1}, \beta)$}
   \label{alg:mbestmle}
\begin{algorithmic}
   \STATE {\bfseries Input:} History information $\gH_{k-1}$, Model class $\mTheta$, and confidence radius $\beta^{\MLE}$.
   \STATE Estimate the transition model for every $h\in [H]$
   \begin{align*}
\theta^{\MLE}_{k,h}=\argmax_{\theta_h\in\mTheta_h} \sum_{i=1}^{k-1}\log\left[\Prob_{\mtheta_h}(s_{i,h+1}|s_{i,h},a_{i,h})\right]\,.
\end{align*}
    \STATE Construct Confidence set:
\begin{align*}
\widehat{\mTheta}_k=\sets{\mtheta\in\mTheta:\sum_{i=1}^{k-1}\sum_{h=1}^H\log \PP_{\theta_h}(s_{i,h+1}|s_{i,h},a_{i,h})\geq \sum_{i=1}^{k-1}\sum_{h=1}^H\log \PP_{\theta^{\MLE}_{k,h}} (s_{i,h+1}|s_{i,h},a_{i,h})-\beta^\MLE}\,.
\end{align*}
\STATE {\bfseries Return} $\widehat{\mTheta}_k^{\MLE}$.
\end{algorithmic}
\end{algorithm}
Here we present our construction of confidence set via MLE in Algorithm \ref{alg:mbestmle}, which is inspired by the OMLE algorithm of \cite{liu2023optimistic}.
In this algorithm, we first calculate the maximal likelihood estimator $\theta_{k,h}^\MLE$ for each step $h$ based on the history $\gH_{k-1}$ before episode $k$ by the following equation.
\begin{align*}
\theta^{\MLE}_{k,h}=\argmax_{\theta_h\in\Theta_h} \sum_{i=1}^{k-1}\log \PP_{\theta_h}(s_{i,h+1}|s_{i,h},a_{i,h})\,.
\end{align*}
Then we can construct the confidence set centered at the maximal likelihood estimator with radius $\beta$:
\begin{align*}
\widehat{\mTheta}_k=\sets{\mtheta\in\mTheta:\sum_{i=1}^{k-1}\sum_{h=1}^H\log \PP_{\theta_h}(s_{i,h+1}|s_{i,h},a_{i,h})\geq \sum_{i=1}^{k-1}\sum_{h=1}^H\log \PP_{\theta^{\MLE}_{k,h}} (s_{i,h+1}|s_{i,h},a_{i,h})-\beta^\MLE}\,.
\end{align*}
Where $\beta^\MLE=H\log\bracket{eHK\cN_{[\cdot]}(\epsilon,\Theta,\norm{\cdot}_1)/\delta}$ .

\subsection{Simulation Lemma in Augmented MDP}

Next, we build the relationship between the supremum norm of the cumulative distribution function and the $\ell_1$-norm of the trajectory probability kernel $\mu_\mtheta^\mpi$. Throughout this section, we define $R(\bm\tau)$ as the cumulative reward of the trajectory $\bm\tau$.

\begin{lemma}[Distribution difference]\label{lem:distribution diff}
For any fixed model $\mtheta \in \mTheta$, $Z_\mtheta^\mpi$ is the random variable of the cumulative reward collected by policy $\mpi$ in the MDP modeled by $\mtheta$.Thus we have the following holds
    \begin{align*}
        \norm{F_{Z_\mtheta^{\mpi}}-F_{Z_{\mtheta^*}^{\mpi}}}_{\infty}\leq \norm{\mu_{\mtheta}^{\mpi}-\mu_{\mtheta^*}^{\mpi}}_1\,.
    \end{align*}
\end{lemma}
\begin{proof}
By the definition of the CDF, we have $F_{Z^\mpi_\mtheta}(x) = \int_{\bm\tau} \mu_\mtheta^\mpi(\bm\tau)\sI(R(\bm\tau) \leq x)$. Thus we have
\begin{align*}
    \norm{F_{Z_\mtheta^{\mpi}}-F_{Z_{\mtheta^*}^{\bm\pi}}}_\infty=&\sup_{x\in [0, H]}\abs{F_{Z_\mtheta^{\mpi}}(x)-F_{Z_{\mtheta^*}^{\mpi}}(x)}\\
    =& \sup_{x\in[0,H]}\abs{\int_{\bm\tau}\bracket{\mu_{\mtheta}^{\mpi}(\bm\tau)-\mu_{\mtheta^*}^{\mpi}(\bm\tau)}\sI(R(\mtau)\leq x)}\\
    \leq & \int_{\bm\tau} \abs{\mu_{\mtheta}^\mpi(\bm\tau)-\mu_{\mtheta^*}^{\mpi}(\bm\tau)}\\
    =&\norm{\mu_{\mtheta}^{\mpi}-\mu_{\mtheta^*}^{\mpi}}_1\,,
\end{align*}
where the first inequality holds by the triangle inequality, and the last equality is due to the definition of the $\ell_1$-norm.
\end{proof}
We then establish the simulation lemma for augmented MDP, which connect the $\ell_1$-norm difference of the probability measure $\mu_\mtheta^\mpi$ on the augmented MDP with the $\ell_1$-norm difference of transition probabilities. This is one of the key lemmas that bridge the gap between the analysis in augmented MDP and origin MDP.

\begin{lemma}[Augmented simulation lemma]\label{lem:simulation lemma}
\begin{align*}
    \norm{\mu_{\mtheta}^{\mpi}-\mu_{\mtheta^*}^{\mpi}}_1\leq \sum_{h=1}^H \EE_{(s_h,a_h)\sim\nu_{\mtheta^*, h}^{\mpi}}\norm{\PP_{\theta_h}(s_h,a_h)-\PP_{\theta^*_h}(s_h,a_h)}_1\leq 2H\norm{\mu_{\mtheta}^{\mpi}-\mu_{\mtheta^*}^{\mpi}}_1\,.
\end{align*}
    
\end{lemma}
\begin{proof}
    First, we decompose $\norm{\mu_{\mtheta}^{\mpi}-\mu_{\mtheta^*}^{\mpi}}_1$. Following standard simulation lemma analysis techniques, we have for any trajectory $\bm\tau$,
    \begin{align*}
        &\abs{\mu_{\mtheta}^{\mpi}(\bm\tau)-\mu_{\mtheta^*}^{\mpi}(\bm\tau)}\\
        =&\abs{\prod_{h=1}^H \TT_{\theta^*_h}(s_{h+1}^\dag|s_h^\dag,a_h)\pi_h(a_h|s_h^\dag)-\prod_{i=1}^H \TT_{\theta_h}(s_{h+1}^\dag|s_h^\dag,a_h)\pi_h(a_h|s_h^\dag)}\\
        =& \abs{\sum_{h=1}^H \prod_{i=1}^{h-1} \TT_{\theta_i^*}(s_{i+1}^\dag|s_i^\dag,a_i)\pi_i(a_i|s_i^\dag)\bracket{\TT_{\theta_h^*}(s_{h+1}^\dag|s_h^\dag,a_h)-\TT_{\theta_h}(s_{h+1}^\dag|s_h^\dag,a_h)}\pi_h(a_h|s_h^\dag)\prod_{i=h+1}^H \TT_{\theta_i}(s_{i+1}^\dag|s_i^\dag,a_i)\pi_i(a_i|s_i^\dag)}\\
        \leq& \sum_{h=1}^H \prod_{i=1}^{h-1} \TT_{\theta^*_i}(s_{i+1}^\dag|s_i^\dag,a_i)\pi_i(a_i|s_i^\dag) \abs{\TT_{\theta^*_h}(s_{h+1}^\dag|s_h^\dag,a_h)-\TT_{\theta_h}(s_{h+1}^\dag|s_h^\dag,a_h)}\pi_h(a_h|s_h^\dag)\,.
    \end{align*}
    Integral the above inequality over the entire space $\Omega$, we have
    \begin{align*}
        &\int_{\Omega} \abs{\mu_{\mtheta}^{\mpi}(\mtau)-\mu_{\mtheta^*}^{\mpi}(\mtau)}d\mtau = \int_{\bm\tau} \abs{\mu_{\mtheta}^{\mpi}(\mtau)-\mu_{\mtheta^*}^{\mpi}(\mtau)} = \norm{\mu_{\mtheta}^{\mpi}-\mu_{\mtheta^*}^{\mpi}}_1 \\
        \leq& \sum_{h=1}^H \int_\Omega \prod_{i=1}^{h-1} \TT_{\theta^*_i}(s_{i+1}^\dag|s_i^\dag,a_i)\pi_i(a_i|s_i^\dag) \pi_h(a_h|s_h^\dag)\abs{\TT_{\theta^*_h}(s_{h+1}^\dag|s_h^\dag,a_h)-\TT_{\theta_h}(s_{h+1}^\dag|s_h^\dag,a_h)}d\mtau\\
        =&\sum_{h=1}^H \int_{(s_h^\dag, a_h, s_{h+1}^\dag) \in \gS^\dag \times \gA \times \gS^\dag} \abs{\TT_{\theta^*_h}(s_{h+1}^\dag|s_h^\dag,a_h)-\TT_{\theta_h}(s_{h+1}^\dag|s_h^\dag,a_h)}d s_{h+1}^\dag d\mu_{\mtheta^*, h}^{\mpi}(s_h^\dag,a_h)\\
        =&\sum_{h=1}^H \EE_{(s_h^\dag,a_h)\sim\mu_{\mtheta^*, h}^{\mpi}}\norm{\TT_{\theta_h}(s_h^\dag,a_h)-\TT_{\theta^*_h}(s_h^\dag,a_h)}_1\,,
    \end{align*}
    where $\mu^\mpi_{\mtheta, h}(s_h^\dag, a_h)$ represent the probability of arriving the augmented state-action pair $(s_h^\dag, a_h)$ at step $h$ with policy $\mpi$ in the MDP modeled by $\mtheta$. Moreover, we define $\mu^\mpi_{\mtheta, h}(s_h^\dag, a_h, s_{h+1}^\dag)$ as the probability of visit $(s_h^\dag, a_h, s_{h+1}^\dag)$ at step $h$ and $h+1$, i.e., $\mu^\mpi_{\mtheta,h}(s_h^\dag, a_h, s_{h+1}^\dag) = \mu^\mpi_{\mtheta, h}(s_h^\dag, a_h)\TT_h(s_{h+1}^\dag | s_h^\dag, a_h)$. With this fact, we can rewrite the summation of the $\ell_1$-norm difference of the transition probabilities of the augmented MDP.
    \begin{align*}
    &\EE_{(s_h^\dag,a_h)\sim \mu_{\mtheta^*, h}^{\mpi}}\norm{\TT_{\theta_h}(s_h^\dag,a_h)-\TT_{\theta^*_h}(s_h^\dag,a_h)}_1\\
    =&\int_{\gS^\dag \times \gA} d\mu_{\mtheta^*, h}^{\mpi}(s_h^\dag,a_h)\int_{\gS^\dag}\abs{\TT_{\theta_h}(s_{h+1}^\dag|s_h^\dag,a_h)-\TT_{\theta^*_h}(s_{h+1}^\dag|s_h^\dag,a_h)} ds_{h+1}^\dag\\
    \leq&\int_{\gS^\dag\times\gA\times\gS^\dag} \abs{\mu_{\mtheta^*, h}^{\mpi}(s_h^\dag,a_h,s_{h+1}^\dag)-\mu_{\mtheta, h}^{\mpi}(s_h^\dag,a_h,s_{h+1}^\dag)}+\abs{\mu_{\mtheta,h}^{\mpi}(s_h^\dag,a_h)-\mu_{\mtheta^*_h}^{\mpi}(s_h^\dag,a_h)}\TT_{\theta_h}(s_{h+1}^\dag|s_h^\dag,a_h) ds_h^\dag da_h ds_{h+1}^\dag\,.
    \end{align*}
    By the fact that for any state action pair $(s_h^\dag, a_h)$, we have
    \begin{align*}
        &\int_{\gS^\dag \times \gA} \abs{\mu_{\mtheta,h}^{\mpi}(s_h^\dag,a_h)-\mu_{\mtheta^*_h}^{\mpi}(s_h^\dag,a_h)} ds_h^\dag da_h\\
        = &\int_{s_h^\dag,a_h}\int_{\{\mtau : \tau_h = (s_h^\dag, a_h)\}}\abs{\mu_\mtheta^\mpi(\mtau) - \mu_{\mtheta^*}^\mpi(\mtau)}d\mtau \\
        \leq &\int_\mtau \abs{\mu_\mtheta^\mpi(\mtau) - \mu_{\mtheta^*}^\mpi(\mtau)} \\
        = &\norm{\mu_{\mtheta^*}^{\pi}-\mu_{\mtheta}^{\pi}}_1\,.
    \end{align*}
    We can apply a similar method to give
    \begin{equation*}
        \int_{\gS^\dag\times\gA\times\gS^\dag} \abs{\mu_{\mtheta^*, h}^{\mpi}(s_h^\dag,a_h,s_{h+1}^\dag)-\mu_{\mtheta, h}^{\mpi}(s_h^\dag,a_h,s_{h+1}^\dag)} ds_h^\dag da_h ds_{h+1}^\dag \leq \norm{\mu_{\mtheta^*}^{\pi}-\mu_{\mtheta}^{\pi}}_1\,.
    \end{equation*}
    Thus we can get
    \begin{equation*}
        \sum_{h=1}^H\EE_{(s_h^\dag,a_h)\sim \mu_{\mtheta^*, h}^{\mpi}}\norm{\TT_{\theta_h}(s_h^\dag,a_h)-\TT_{\theta^*_h}(s_h^\dag,a_h)}_1 \leq 2H\norm{\mu_{\mtheta^*}^{\pi}-\mu_{\mtheta}^{\pi}}_1\,.
    \end{equation*}
    At last, we need to prove the equivalence between augmented transition difference and original transition difference.
    \begin{align*}
    &\EE_{(s_h^\dag,a_h)\sim \mu_{\mtheta^*,h}^{\pi}}\mbracket{\norm{\TT_{\theta^*_h}(s_h^\dag,a_h)-\TT_{\theta_h}(s_h^\dag,a_h)}_1}\\
    =& \int_{\gS^\dag \times \gA} d\mu_{\theta^*,h}^{\pi}(s_h^\dag,a_h) \int_{\gS}\int_{[y_h,1]}   \abs{\PP_{\theta^*_h}(s_{h+1}|s_h,a_h)-\PP_{\theta_h}(s_{h+1}|s_h,a_h)}\RR_h(y_{h+1}-y_h|s_h,a_h) dy_{h+1}ds_{h+1}\\
    =&\int_{\gS\times\gA}\int_0^1\norm{\PP_{\theta^*_h}(s_h,a_h)-\PP_{\theta_h}(s_h,a_h)}_1 d\mu_{\mtheta^*,h}^{\mpi}((s_h,y_h),a_h) \int_{y_h}^1 \RR_h(y_{h+1}-y_h|s_h,a_h)d y_{h+1} \\
    =&\int_{\gS\times\gA}\norm{\PP_{\theta^*_h}(s_h,a_h)-\PP_{\theta_h}(s_h,a_h)}_1 d\nu_{\mtheta^*,h}^{\mpi}(s_h,a_h) \\
    =& \int_{\gS\times\gA} \norm{\PP_{\theta^*_h}(s_h,a_h)-\PP_{\theta'_h}(s_h,a_h)}_1 d\nu_{\theta^*,h}^{\pi}(s_h,a_h)\\
    =& \EE_{(s_h,a_h)\sim \nu_{\mtheta^*,h}^{\mpi}}\mbracket{\norm{\PP_{\theta^*_h}(s_h,a_h)-\PP_{\theta'_h}(s_h,a_h)}_1}\,.
\end{align*}
where the first equality holds by definition, the second equality is due to the Fubini Theorem, and the third equality holds by the decomposition of the probability measures of augmented MDP.
\end{proof}
\subsection{Concentration Condition of MLE Approach}
In this section, we prove the concentration condition of the MLE approach. The key idea is utilizing the property of MLE which is also studied by \cite{agarwal2020flambe,liu2022when,liu2023optimistic}.

We define the $\ell_1$-norm difference on the model set $\mTheta$. For any $\mtheta^1, \mtheta^2 \in \mTheta$, we define $\|\mtheta^1 - \mtheta^2\|_1$ as:
\begin{align*}
    &\|\mtheta^1 - \mtheta^2 \|_1 := \sup_{h\in[H], (s_h, a_h) \in \gS \times \gA} \|\Prob_{\theta^1_h}(s_h,a_h) - \Prob_{\theta^2_h}(s_h,a_h)\|_1 
    \\
    = &\sup_{h\in[H], (s_h, a_h) \in \gS \times \gA} \int_{\gS} \abs{\Prob_{\theta^1_h}(s_{h+1} | s_h, a_h) - \Prob_{\theta^2_h}(s_{h+1} | s_h, a_h)} ds_{h+1}\,.
\end{align*}

Then we denote $\mTheta^\downarrow$ as the lower bracket function set of $\mTheta$ such that for any $\mtheta \in \mTheta$, we have $\mtheta^\downarrow \in \mTheta^\downarrow$ satisfying $\| \mtheta - \mtheta^\downarrow \|_1 \leq \epsilon$. We have $\abs{\mTheta^\downarrow} \leq \gN_{[\cdot]}(\mTheta,\epsilon, \|\cdot\|_1)$.
\begin{lemma}[Likelihood difference]\label{lem:likelihood diff}
    Consider the probability constant $\delta \in (0, 1]$. For \textbf{all} $\mtheta\in\mTheta$ $h\in[H]$, and $k\in[K]$, we have:
    \begin{align*}
        \sum_{i=1}^{k} \log \mbracket{\frac{\PP_{\theta_h}(s_{i,h+1}|s_{i,h},a_{i,h})}{\PP_{\theta^*_h}(s_{i,h+1}|s_{i,h},a_{i,h})}}\leq \log\bracket{eK\cN_{[\cdot]}(\mTheta,\epsilon,\norm{\cdot}_{1})/\delta}
    \end{align*}
    holds for fixed $h$ with probability at least $1 - \delta$.
\begin{proof}
    This proof is standard \cite{geer2000empirical,liu2022when, liu2023optimistic}. Consider the lower bracket set $\mTheta^\downarrow$. For any $\mtheta \in \mTheta$, we can find a $\mtheta^\downarrow \in \mTheta^\downarrow$ satisfying $\| \mtheta - \mtheta^\downarrow\|_1 \leq \epsilon$. Moreover, for every $h \in [H]$, we denote $s_{i,h},a_{i_h}$ as the corresponding state action in trajectory $\tau_{i,h}$, and we have:
    \begin{align*}
        &\EE_{s_{i,h+1}, i\in[k]}\left[ \exp\left( \sum_{i=1}^k \log\left[ \frac{\Prob_{\theta^\downarrow_h}(s_{i,h+1} | s_{i,h}, a_{i,h})}{\Prob_{\theta^*_h}(s_{i,h+1} | s_{i,h}, a_{i,h})} \right] \right) \right] \\
        =&\EE_{s_{i,h+1},i \in [k]}
        \left[ \exp\left( \sum_{i=1}^{k-1} \log\left[ \frac{\Prob_{\theta^\downarrow_h}(s_{i,h+1} | s_{i,h}, a_{i,h})}{\Prob_{\theta^*_h}(s_{i,h+1} | s_{i,h}, a_{i,h})} \right] \right) \cdot \EE_{ s_{k,h+1}  }\left[ \frac{\Prob_{\theta^\downarrow_h}(s_{k,h+1} | s_{k,h}, a_{k,h})}{\Prob_{\theta^*_h}(s_{k,h+1} | s_{k,h}, a_{k,h})}  \right] \right] 
    \end{align*}
    By the definition of the expectation operator, we have
    \begin{align*}
        &\EE_{s_{i,h+1},i \in [k]}
        \left[ \exp\left( \sum_{i=1}^{k-1} \log\left[ \frac{\Prob_{\theta^\downarrow_h}(s_{i,h+1} | s_{i,h}, a_{i,h})}{\Prob_{\theta^*_h}(s_{i,h+1} | s_{i,h}, a_{i,h})} \right] \right) \cdot \EE_{ s_{k,h+1}  }\left[ \frac{\Prob_{\theta^\downarrow_h}(s_{k,h+1} | s_{k,h}, a_{k,h})}{\Prob_{\theta^*_h}(s_{k,h+1} | s_{k,h}, a_{k,h})}  \right] \right]  \\
        =& \EE_{s_{i,h+1},i \in [k]}
        \left[ \exp\left( \sum_{i=1}^{k-1} \log\left[ \frac{\Prob_{\theta^\downarrow_h}(s_{i,h+1} | s_{i,h}, a_{i,h})}{\Prob_{\theta^*_h}(s_{i,h+1} | s_{i,h}, a_{i,h})} \right] \right) \cdot \int_\gS\left[ \frac{\Prob_{\theta^\downarrow_h}(s_{k,h+1} | s_{k,h}, a_{k,h})}{\Prob_{\theta^*_h}(s_{k,h+1} | s_{k,h}, a_{k,h})}  \right]d\Prob_{\theta^*_h}(s_{k,h+1}|s_{k,h},a_{k,h}) \right] \\
        =& \EE_{s_{i,h+1},i \in [k]} \left[ \exp\left( \sum_{i=1}^{k-1} \log\left[ \frac{\Prob_{\theta^\downarrow_h}(s_{i,h+1} | s_{i,h}, a_{i,h})}{\Prob_{\theta^*_h}(s_{i,h+1} | s_{i,h}, a_{i,h})} \right] \right) \cdot \int_\gS {\Prob_{\theta^\downarrow_h}(s_{k,h+1} | s_{k,h}, a_{k,h})}d{s_{k,h+1}}\right] \\
        \leq &\EE_{s_{i,h+1},i \in [k]} \left[ \exp\left( \sum_{i=1}^{k-1} \log\left[ \frac{\Prob_{\theta^\downarrow_h}(s_{i,h+1} | s_{i,h}, a_{i,h})}{\Prob_{\theta^*_h}(s_{i,h+1} | s_{i,h}, a_{i,h})} \right] \right) \cdot (\|\Prob_{\theta_h}(s_{k,h},a_{k,h})\|_1 + \|\Prob_{\theta_h}(s_{k,h},a_{k,h}) - \Prob_{\theta_h^\downarrow}(s_{k,h},a_{k,h})\|_1)\right] \\
        \leq &\EE_{s_{i,h+1},i \in [k]} \left[ \exp\left( \sum_{i=1}^{k-1} \log\left[ \frac{\Prob_{\theta^\downarrow_h}(s_{i,h+1} | s_{i,h}, a_{i,h})}{\Prob_{\theta^*_h}(s_{i,h+1} | s_{i,h}, a_{i,h})} \right] \right) \cdot (1 + \epsilon)\right] \\
        \leq &\cdots \\
        \leq &(1+\epsilon)^k\,.
    \end{align*}
    where the first inequality holds by the triangle inequality, and the second inequality holds by the definition of $\theta_h^\downarrow$.
    Moreover, since $\epsilon < 1/K$, we have $(1 + \epsilon)^k < e$. Therefore, by Markov's inequality, we have that
    \begin{align*}
    \Prob\left[ \sum_{i=1}^k \log\left[ \frac{\Prob_{\theta^\downarrow_h}(s_{i,h+1} | s_{i,h}, a_{i,h})}{\Prob_{\theta^*_h}(s_{i,h+1} | s_{i,h}, a_{i,h})}\right] \geq \log(1/\delta') \right] \leq e\delta'\,.
    \end{align*}
    holds for $\delta' \in (0, 1]$. Thus taking a union bound for all $\mtheta^\downarrow \in \mTheta^\downarrow$, $h \in [H]$ and $k \in [K]$, we have with probability at least $1 - \delta$, 
    \begin{align*}
        \sum_{i=1}^k \log\left[ \frac{\Prob_{\theta^\downarrow_h}(s_{i,h+1} | s_{i,h}, a_{i,h})}{\Prob_{\theta^*_h}(s_{i,h+1} | s_{i,h}, a_{i,h})}\right] \leq \log(eHK|\mTheta^\downarrow|/\delta)\,.
    \end{align*}
    Since $|\mTheta^\downarrow| = \gN_{[\cdot]}(\mTheta, \epsilon, \|\cdot\|_1)$ and $\Prob_{\mtheta^\downarrow} \leq \Prob_{\mtheta}$, we have
    \begin{align*}
        \sum_{i=1}^k \log\left[ \frac{\Prob_{\theta_h}(s_{i,h+1} | s_{i,h}, a_{i,h})}{\Prob_{\theta^*_h}(s_{i,h+1} | s_{i,h}, a_{i,h})}\right] \leq \log(eHK\gN_{[\cdot]}(\mTheta, \epsilon, \|\cdot\|_1)/\delta)
    \end{align*}
    holds for any $\mtheta \in \mTheta, h \in [H]$ and $k \in [K]$.
\end{proof}
\end{lemma}

\begin{lemma}[MLE concentration] The estimation function \texttt{M-Est-MLE} satisfying the Condition~\ref{con:mbconcentration}, i.e., for $\delta \in (0, 1]$,  we have with probability at least $ 1- \delta$, $\mtheta^* \in \wh{\mTheta}_k$  for every $k \in [K]$.
\begin{proof}
    Apply Lemma~\ref{lem:likelihood diff}, we have for every $k \in [K], h \in [H]$ and $\mtheta \in \mTheta$,
    \begin{align*}
        \sum_{i=1}^k {\Prob_{\theta_h}(s_{i,h+1} | s_{i,h}, a_{i,h})} \leq \sum_{i=1}^k {\Prob_{\theta^*_h}(s_{i,h+1} | s_{i,h}, a_{i,h})} + \log(eHK\gN_{[\cdot]}(\mTheta, \epsilon, \|\cdot\|_1)/\delta)
    \end{align*}
    with probability at least $ 1- \delta$. Recall that $\beta^\MLE = H\log(eHK\gN_{[\cdot]}(\mTheta, \epsilon, \|\cdot\|_1)/\delta)$. Summing the both sides of the inequality over $h \in [H]$ directly gives the result.
\end{proof}
\end{lemma}

\subsection{Total Variance Distance of Transition Models}

As a direct result of Lemma \ref{lem:mle generalization bound}, we have the following bound for model-based MLE estimation, which is also presented in Proposition 14 in \cite{liu2022when} and Proposition B.2 in \cite{liu2023optimistic}.
\begin{lemma}[TV distance and likelihood distance]\label{lem:TV and likelihood}
\begin{align*}
    \sum_{i=1}^{k-1}\EE_{(s_h,a_h)\sim \mu_{\mtheta^*,h}^{\mpi^i}}\mbracket{\TV\bracket{\PP_{\theta_h}(s_{h},a_{h})||\PP_{\theta^*_h}(s_{h},a_{h})}^2}\leq \mathcal{O}\bracket{\sum_{i=1}^{k-1} \log \mbracket{\frac{\PP_{\theta^*_h}(s_{i,h+1}|s_{i,h},a_{i,h})}{\PP_{\theta_h}(s_{i,h+1}|s_{i,h},a_{i,h})}}+\beta}\,.
\end{align*}
\end{lemma}
Finally, we can bound the Total Variation (TV) distance of the transition models.
\begin{lemma}[TV distance]\label{lem:TV distance}
For $\delta \in (0, 1]$ and any $\mtheta_k\in\wh\mTheta_k$, we have the following concentration on their TV-distance with :
    \begin{align*}
        \sum_{i=1}^{k-1} \sum_{h=1}^H\EE_{(s_h,a_h)\sim \mu_{\mtheta^*,h}^{\mpi^i}}\mbracket{\TV\bracket{\PP_{\theta_{k,h}}(s_{h},a_{h})||\PP_{\theta^*_h}(s_{h},a_{h})}^2}\leq \mathcal{O}(\beta^{\MLE})\,.
    \end{align*}
\end{lemma}
\begin{proof}
    For any $\mtheta_k\in{\wh\mTheta}_k^{\MLE}$, we have that:
\begin{align*}
    \sum_{i=1}^{k-1}\sum_{h=1}^H \log \mbracket{\PP_{\theta_{k,h}}(s_{i,h+1}|s_{i,h},a_{i,h})}\geq\sum_{i=1}^{k-1}\sum_{h=1}^H\log\mbracket{\PP_{{\theta}_{k,h}^{\MLE}}(s_{i,h+1}|s_{i,h},a_{i,h})}-\beta^\MLE\,.
\end{align*}
From the definition of $\theta_{k,h}^{\MLE}$, we have:
\begin{align*}
    \sum_{i=1}^{k-1}\sum_{h=1}^H \log \mbracket{{\PP_{\theta^*_h}(s_{i,h+1}|s_{i,h},a_{i,h})}}\leq \sum_{i=1}^{k-1}\sum_{h=1}^H \log \mbracket{{\PP_{{\theta}_{k,h}^{\MLE}}(s_{i,h+1}|s_{i,h},a_{i,h})}}\,. 
\end{align*}
Thus, we have:
\begin{align*}
     \sum_{i=1}^{k-1}\sum_{h=1}^H \log \mbracket{{\PP_{\theta^*_h}(s_{i,h+1}|s_{i,h},a_{i,h})}}\leq \sum_{i=1}^{k-1}\sum_{h=1}^H \log \mbracket{\PP_{\theta_{k,h}}(s_{i,h+1}|s_{i,h},a_{i,h})}+\beta^{\MLE}\,.
\end{align*}
Using Lemma \ref{lem:TV and likelihood}, we have:
\begin{align*}
    &\sum_{i=1}^{k-1} \sum_{h=1}^H\EE_{(s_h,a_h)\sim \mu_{\mtheta^*,h}^{\mpi^i}}\mbracket{\TV\bracket{\PP_{\theta_{k,h}}(s_{h},a_{h})||\PP_{\theta^*_h}(s_{h},a_{h})}^2} \\
    \leq &\mathcal{O}\bracket{\sum_{i=1}^{k-1}\sum_{h=1}^H \log \mbracket{\frac{\PP_{\theta^*_h}(s_{i,h+1}|s_{i,h},a_{i,h})}{\PP_{\theta_{k,h}}(s_{i,h+1}|s_{i,h},a_{i,h})}}+\beta^\MLE}
    \leq \mathcal{O}(\beta^\MLE)\,.
\end{align*}
\end{proof}

\subsection{Elliptical Potential Condition of MLE Approach}
First, we define the low witness rank, which is the structural complexity measure for this setting. We remark this is the Q-type witness rank and we provide a separate proof for V-type witness rank in Section \ref{sec:proof low rank}. 
\begin{definition}[Q-type low witness rank (Definition 6.6 of \cite{liu2023optimistic})]\label{ass:low witness rank}
We say the model class satisfies $(d,\alpha,B)$ witness rank condition, if there exists mappings: $\sets{f_h}$ and $\sets{g_h}$ from $\mTheta$ to $\RR^d$, such that for any $h\in[H]$:
    \begin{align*}
        & \EE_{(s_h,a_h)\sim\nu_{\mtheta^*,h}^{\mpi_{\mtheta}}}\mbracket{\TV\bracket{\PP_{\mtheta'_h}(s_h,a_h)||\PP_{\mtheta^*_h}(s_h,a_h)}}\leq \innerproduct{f_h(\mtheta)}{g_h(\mtheta')}\\
&\EE_{(s_h,a_h)\sim\nu_{\mtheta^*,h}^{\mpi_{\mtheta}}}\mbracket{\TV\bracket{\PP_{\mtheta'_h}(s_h,a_h)||\PP_{\mtheta^*_h}(s_h,a_h)}}\geq\alpha^{-1}\innerproduct{f_h(\mtheta)}{g_h(\mtheta')}\\
&\norm{f_h(\mtheta)}_1\norm{g_h(\mtheta')}_\infty\leq B\,.
    \end{align*}
\end{definition}
As a special case, we can show that the factored MDPs \cite{sun2019model} have low witness rank.
\begin{definition}[Factored MDP]
    In factored MDPs the states admit a factored structure. Each state contains $m$ factors $(s[1],s[2],\cdots,s[m])\in\gX^m$. Each factor $i\in[m]$ has a parent set $pa_i\subset[m]$, with respect to which the transitions admit a factored form:
    \begin{align*}
        \PP_h(s_{h+1}|s_h,a_h)=\prod_{i=1}^m \PP_h^i(s_{h+1}[i]|s_h[pa_i],a_h)\,.
    \end{align*}
\end{definition}
The following proposition establishes the low witness rank property for factored MDPs, which comes directly from Proposition 6.8 of \cite{liu2023optimistic}. 
\begin{proposition}\label{prop:factor mdp}
    Let $\mTheta$ denote all the factored MDPs with the same factorization structure, then $\mTheta$ satisfies low witness rank with $d=|\cA|\sum_{i=1}^m \abs{\gX}^{|pa_i|}$, $\alpha=m$, $B=\sum_{i=1}^m \abs{\gX}^{|pa_i|}$\,.
\end{proposition}

\begin{proof}[Proof of Condition \ref{ass:pigeon-hole}]

From Lemma \ref{lem:TV distance} and using Cauchy inequality, we have:
\begin{align*}
    \sum_{h=1}^H\sum_{i=1}^{k-1}\abs{\innerproduct{f_h(\mtheta_i)}{g_h(\mtheta_k)}}^2\leq& \alpha^2 \sum_{h=1}^H\sum_{i=1}^{k-1}\bracket{\EE_{(s_h,a_h)\sim\mu_{\mtheta^*,h}^{\mpi^i}}\mbracket{\TV\bracket{\PP_{\theta_{k,h}}(s_h,a_h)||\PP_{\theta^*_h}(s_h,a_h)}}}^2\\
    \leq& \alpha^2\sum_{h=1}^H\sum_{i=1}^{k-1}\EE_{(s_h,a_h)\sim\mu_{\mtheta^*,h}^{\mpi^i}}\mbracket{\TV\bracket{\PP_{\theta_{k,h}}(s_h,a_h)||\PP_{\theta^*_h}(s_h,a_h)}^2}\\
    \leq& \mathcal{O}\bracket{\alpha^2\beta^{\MLE}}\,.
\end{align*}
Thus, we have by standard elliptical potential arguments (eg:Theorem 6.4 in \cite{liu2023optimistic}):
\begin{align*}
    \sum_{h=1}^H\sum_{i=1}^k \EE_{(s_h,a_h)\sim\mu_{\mtheta^*,h}^{\pi^i}}\mbracket{\TV\bracket{\PP_{\theta_{k,h}}(s_h,a_h)||\PP_{\theta^*_h}(s_h,a_h)}}\leq& \sum_{h=1}^H \sum_{i=1}^k \innerproduct{f_h(\mtheta_i)}{g_h(\mtheta_i)}\\
    \leq& \mathcal{O}\bracket{dB+\alpha\sqrt{d\beta^{\MLE} k}}\,.
\end{align*}
Thus, from Lemma \ref{lem:distribution diff} and Lemma \ref{lem:simulation lemma}, we have:
\begin{align*}
    \sum_{k=1}^K \norm{F_{Z_{\mtheta_k}^{\mpi^k}}-F_{Z_{\mtheta^*}^{\mpi_k}}}_\infty \leq& 2\sum_{k=1}^K\sum_{h=1}^H \EE_{(s_h,a_h)\sim\mu_{\mtheta^*,h}^{\mpi^k}}\mbracket{\TV\bracket{\PP_{\theta_{k,h}}(s_h,a_h)||\PP_{\theta^*_h}(s_h,a_h)}} \\
    \leq& \mathcal{O}\bracket{dB+\alpha\sqrt{d\beta^{\MLE} K}}\,.
\end{align*}
Thus, using the $L_\infty$ Lipschitz of $\rho$, we have:
\begin{equation*}
    \sum_{k=1}^K \rho(Z_{\mtheta_k}^{\mpi^k})-\rho(Z_{\mtheta^*}^{\mpi^k})\leq \mathcal{O}\bracket{L_{\infty}\bracket{dB+\alpha\sqrt{d\beta^{\MLE} K}}}\,.
\end{equation*}
    
\end{proof}

\subsection{Proof for Low Rank MDP}\label{sec:proof low rank}
Extend our analysis above, we further provide an algorithm and regret analysis for low-rank MDPs, a special case of V-type witness rank, which also follows the MLE procedure above. The main difference is that we need to construct the mappings $f_h(\mtheta)$ and $g_h(\mtheta)$ in a novel manner to capture the low V-type witness rank of low rank MDPs in the augmented MDP. 

First we define low rank MDPs as below \cite{agarwal2020flambe,uehara2021representation}:
\begin{definition}[Low Rank MDP]\label{ass:low rank mdp}
    The transition kernel $\PP_{\theta^*_h}(s_{h+1}|s_h,a_h)$ admits a low rank structure, i,e, there exists two sets of mappings $\bm{\phi}^*:\cS\times\cA\rightarrow\RR^d$ and $\bm{\psi}^*:\cS\rightarrow\RR^d$, such that:
    \begin{align*}
        \PP_{\theta^*_h}(s_{h+1}|s_h,a_h)=\innerproduct{\phi^*_h(s_h,a_h)}{\psi^*_h(s_{h+1})}\,.
    \end{align*}
    We have $\norm{\phi^*_h(s_h,a_h)}\leq 1$ and $\norm{\int_{s_h}\psi^*_h(s_h)g(s_h)}\leq \sqrt{d}\norm{g}_{\infty}$ for all $h\in[H]$.
    Also, assume that we have access to two embedding classes $\Phi$ and $\Psi$ such that $\bm{\phi}^*\in\Phi$ and $\bm{\psi}^*\in\Psi$. 
\end{definition}
The model class $\mTheta$ consists of all the transition kernels with the low rank structure defined by the inner-product of the embedding in $\Phi$ and $\Psi$, with $|\mTheta|=|\Psi||\Phi|$. Define the exploratory policy class for a policy $\mpi_{\mtheta}$ as $\Pi_{\exp}(\mpi_\mtheta)=\sets{\tilde{\pi}_{\theta,h}:\pi_\theta[1:h-1]\circ \gU[h:H]}_{h=1}^H$, where $\gU$ is the uniform policy. $\tilde{\pi}_{\theta,h}$ is defined as following $\pi_{\mtheta}$ for the first $h-1$ steps then taking uniform actions. We have $\abs{\Pi_{\exp}}=H$.
Define the two sets of mappings used to construct low witness rank for low rank MDPs as:
\begin{align*}
    &f_h(\mtheta)=\int_{s_{h-1},a_{h-1}}\nu_{\mtheta^*,h-1}^{\mpi_\theta}(s_{h-1},a_{h-1})\phi^*_{h-1}(s_{h-1},a_{h-1})\\
    &g_h(\mtheta)=\int_{s_h,a_h}\psi^*_h(s_h)\gU(a_h)\norm{\PP_{\theta^*_h}(s_h,a_h)-\PP_{\theta_h}(s_h,a_h)}_1\,,
\end{align*}
where $\mpi_\mtheta$ is the optimal risk-sensitive policy given model $\theta$.
We have that for any $\tilde{\pi}_{\theta,h}\in\Pi_{\exp}(\mpi_\mtheta)$:
\begin{align}\label{eq:lw rank mle rhs}
   &2H \norm{\mu_{\mtheta^*}^{\Tilde{\pi}_{\theta,h}}-\mu_{\mtheta'}^{\Tilde{\pi}_{\theta,h}}}_1\notag\\
   \geq & \int_{s_h^\dag,a_h}\mu_{\theta^*,h}^{\Tilde{\pi}_{\theta,h}}(s_h^\dag,a_h)\norm{\TT_{\theta'_h}(s_h^\dag,a_h)-\TT_{\theta^*_h}(s_h^\dag,a_h)}_1\notag\\
   =&\int_{s_h^\dag,a_h}\mu_{\mtheta^*,h}^{\Tilde{\pi}_{\theta,h}}(s_h^\dag,a_h)\norm{\PP_{\theta^*_h}(s_h,a_h)-\PP_{\theta'_h}(s_h,a_h)}_1\notag\\
   =& \int_{s_{h-1}^\dag,a_{h-1},s_h,a_h,y_h}\mu_{\mtheta^*,h-1}^{\pi_{\theta}}(s_{h-1}^\dag,a_{h-1})\RR(y_h-y_{h-1}|s_{h-1},a_{h-1})\\
   &\cdot\PP_{\theta_{h-1}^*}(s_h|s_{h-1},a_{h-1})\gU(a_h)\norm{\PP_{\theta^*_h}(s_h,a_h)-\PP_{\theta'_h}(s_h,a_h)}_1\notag\\
   =& \int_{s_{h-1}^\dag,a_{h-1}}\mu_{\mtheta^*,h-1}^{\pi_{\theta}}(s_{h-1}^\dag,a_{h-1})\phi^*_{h-1}(s_{h-1},a_{h-1}) \\
   &\cdot\int_{y_h}\RR(y_h-y_{h-1}|s_{h-1},a_{h-1}) \int_{s_h,a_h}\psi^*_{h-1}(s_h)\gU(a_h)\norm{\PP_{\theta^*_h}(s_h,a_h)-\PP_{\theta'_h}(s_h,a_h)}_1\notag\\
   =&\innerproduct{f_{h}(\mtheta)}{g_h(\mtheta')}\,,
\end{align}
where the first inequality is from the right hand side of simulation lemma (Lemma \ref{lem:simulation lemma}).

On the other hand, we have:
\begin{align}\label{eq:low rank mle lhs}
    &\norm{\mu_{\mtheta^*}^{\mpi_{\mtheta}}-\mu_{\mtheta}^{\mpi_{\mtheta}}}_1\notag\\
    \leq &\sum_{h=1}^H\int_{s_h^\dag,a_h}\mu_{\mtheta^*,h}^{\mpi_\theta}(s_h^\dag,a_h)\norm{\PP_{\theta^*_h}(s_h,a_h)-\PP_{\theta_h}(s_h,a_h)}_1\notag\\
    =& \sum_{h=1}^H\int_{s_h^\dag,a_h}\mu_{\mtheta^*,h}^{\mpi_{\mtheta}}(s_h^\dag)\pi_{\theta,h}(a_h|s_h^\dag)\norm{\PP_{\theta^*_h}(s_h,a_h)-\PP_{\theta_h}(s_h,a_h)}_1\notag\\
    \leq &A\sum_{h=1}^H \int_{s_h^\dag,a_h}\mu_{\mtheta^*,h}^{\mpi_{\mtheta}}(s_h^\dag)\gU(a_h)\norm{\PP_{\theta^*_h}(s_h,a_h)-\PP_{\theta_h}(s_h,a_h)}_1\notag\\
    \leq&A\sum_{h=1}^H\int_{s_{h-1}^\dag,a_{h-1}}\mu_{\mtheta^*,h-1}^{\mpi_{\mtheta}}(s_{h-1}^\dag,a_h)\int_{y_h}\RR(y_h-y_{h-1}|s_{h-1},a_{h-1})\\
    &\cdot\int_{s_h,a_h}\PP_{\theta^*_{h-1}}(s_h|s_{h-1},a_{h-1})\gU(a_h)\norm{\PP_{\theta^*_h}(s_h,a_h)-\PP_{\theta_h}(s_h,a_h)}_1\notag\\
    =&A\sum_{h=1}^H\int_{s_{h-1}^\dag,a_{h-1}}\mu_{\mtheta^*,h-1}^{\mpi_{\mtheta}}(s_{h-1}^\dag,a_h)\phi^*_{h-1}(s_{h-1},a_{h-1})\int_{s_h,a_h}\psi^*_{h-1}(s_h)\gU(a_h)\norm{\PP_{\theta^*_h}(s_h,a_h)-\PP_{\theta_h}(s_h,a_h)}_1\notag\\
    =&A\sum_{h=1}^H \innerproduct{f_h(\mtheta)}{g_h(\mtheta)}\,.
\end{align}
We present a modified version of the algorithm as Algorithm \ref{alg:mle low rank},
\begin{algorithm}[htbp]
   \caption{\texttt{RS-DisRL-Low-Rank-MDP}$(\mTheta,\beta)$}
   \label{alg:mle low rank}
\begin{algorithmic}
   \STATE {\bfseries Input:} Model class $\mTheta$, confidence radius $\beta^{\MLE}=\log(|\mTheta|/\delta)$.
   \STATE {\bfseries Init:} $\mTheta_1\leftarrow\mTheta$
   \FOR{$k=1$ {\bfseries to} $K$}
   \STATE Optimistic Planning:
   $(\mpi_{\widehat\theta_k}, \widehat\mtheta_k)=\argmax_{\mpi\in\Pi^\dag, \mtheta\in\widehat{\mTheta}_k}\rho(Z_{\mtheta}^\mpi)$.
   \STATE Execute and collect information:
   For every policy $\Tilde{\mpi}\in\Pi_{\exp}(\mpi_{\widehat\theta_k})$
   Execute policy $\Tilde{\mpi}$, add the collected data $\mtau=\sets{(s_{h}^\dag,a_{h},s_{h+1}^\dag)}_{h=1}^H$ into history $\cH_k = \cH_{k-1}\cup \{\mtau\}$.
   \STATE Estimate the MLE solution:
   \begin{align*}
       \theta_{k+1,h}^{\MLE}=(\phi_{k+1,h}^{\MLE},\psi_{k+1,h}^{\MLE})=\argmax_{\phi\in\Phi,\psi\in\Psi}\sum_{\mtau\in\gH_k}\log \innerproduct{\phi(s_h^\mtau,a_h^\mtau)}{\psi(s_{h+1}^\mtau)}\,,~~\forall h\in[H]\,,
   \end{align*}
   where $s_h^\mtau,a_h^\mtau$ denotes the $h$ step state action pair in trajectory $\mtau$.
   \STATE Construct confidence set:
   \begin{align*}
       \widehat{\mTheta}_{k+1}=\sets{\phi\in\Psi,\psi\in\Psi:\sum_{\mtau\in\gH_k}\sum_{h=1}^H \log \bracket{\frac{\innerproduct{\phi_{k+1,h}^{\MLE}(s_h^\mtau,a_h^\mtau)}{\psi_{k+1,h}^{\MLE}(s_{h+1}^\mtau)}}{\innerproduct{\phi_{h}(s_h^\mtau,a_h^\mtau)}{\psi_{h}(s_{h+1}^\mtau)}}}\leq \beta}\,.
   \end{align*}
\ENDFOR

\end{algorithmic}
\end{algorithm}
where we modify the data collection process such that in each episode, instead of executing policy $\mpi_{\widehat{\mtheta}_k}$, we execute all the policies $\tilde{\mpi}_{\widehat{\theta}_k,h}\in\Pi_{\exp}(\mpi_{\widehat{\mtheta}_k})$. Notice that in each step we collect $H$ trajectories by taking the combination of the optimistic policy and the uniform exploratory policy. Similar to the proof in the Q-type witness rank MLE concentration, we have $\mtheta^*\in\widehat{\mTheta}_k$ for all $k\in[K]$ by choosing $\beta^{\MLE}=\log(|\mTheta|/\delta)$. Also, we can bound the sum of the square distance similar as Lemma \ref{lem:TV distance}:
\begin{align*}
    \sum_{i=1}^k \sum_{\tilde{\mpi}\in \Pi_{\exp}(\mpi_{\widehat{\theta}_i})} \sum_{h=1}^H\EE_{\mu_{\mtheta^*,h}^{\Tilde{\mpi}}}\mbracket{\norm{\PP_{\widehat{\theta}_{k,h}}(s_h,a_h)-\PP_{\theta^*_h}(s_h,a_h)}_1^2}\leq \mathcal{O}(\beta^{\MLE})\,.
\end{align*}
Thus, using Lemma \ref{lem:simulation lemma} we have that for any $\widehat{\mtheta}_k\in\widehat{\mTheta}_k^{\MLE}$:
\begin{align*}
    &\sum_{i=1}^k \sum_{\tilde{\mpi}\in \Pi_{\exp}(\mpi_{\widehat{\mtheta}_i})}\norm{\mu_{\mtheta^*}^{\tilde{\mpi}}-\mu_{\widehat{\mtheta}_k}^{\tilde{\mpi}}}_1^2\\
    \leq &  \sum_{i=1}^k \sum_{\tilde{\mpi}\in \Pi_{\exp}(\mpi_{\widehat{\mtheta}_i})} \bracket{\sum_{h=1}^H \EE_{\mu_{\mtheta^*,h}^{\Tilde{\mpi}}}\mbracket{\norm{\PP_{\widehat{\theta}_{k,h}}(s_h,a_h)-\PP_{\theta^*_h}(s_h,a_h)}_1}}^2\\
    \leq& \mathcal{O}(\poly(H)\beta^{\MLE})\,.
\end{align*}
Thus we have from Equation \ref{eq:lw rank mle rhs}:
\begin{align*}
    \sum_{i=1}^k\sum_{h=1}^H\bracket{\innerproduct{f_h(\widehat{\mtheta}_i)}{g_h(\widehat{\mtheta}_k)}}^2\leq \mathcal{O}\bracket{\poly(H)\beta^\MLE}\,.
\end{align*}
Since we have that $\norm{f_h(\mtheta)}_1\leq \sqrt{d}$ and $\norm{g_h(\mtheta)}_\infty\leq \sqrt{d}$, we have by standard elliptical arguments and Equation \ref{eq:low rank mle lhs}:
\begin{align*}
    \sum_{i=1}^k \norm{\mu_{\mtheta^*}^{\mpi_{\widehat{\mtheta}_i}}-\mu_{\widehat{\mtheta}_i}^{\mpi_{\widehat{\mtheta}_i}}}_1\leq &A\sum_{i=1}^k\sum_{h=1}^H \innerproduct{f_h(\widehat{\mtheta}_i)}{g_h(\widehat{\mtheta}_i)}\\
    \leq& \mathcal{O}(\poly(H)A\sqrt{d\beta^\MLE k})\,,
\end{align*}
where we ignore the constant and low order terms in $k$.
Thus we have:
\begin{align*}
    \sum_{k=1}^K\rho(Z^{\mpi^*}_{\mtheta^*})-\rho(Z^{\mpi^k}_{\mtheta^*})\leq& \sum_{k=1}^K \rho(Z^{\mpi_{\widehat{\mtheta}_k}}_{\widehat{\mtheta}_k})-\rho(Z^{\mpi_{\widehat{\mtheta}_k}}_{\widehat{\mtheta}_k})\\
    \leq & L_\infty \sum_{k=1}^K \norm{F_{Z^{\mpi_{\widehat{\mtheta}_k}}_{\widehat{\mtheta}_k}}-F_{Z^{\mpi_{\widehat{\mtheta}_k}}_{\widehat{\mtheta}_k}}}_\infty\\
\leq&L_\infty\sum_{k=1}^K\norm{\mu_{\mtheta^*}^{\mpi_{\widehat{\mtheta}_k}}-\mu_{\widehat{\mtheta}_k}^{\mpi_{\widehat{\mtheta}_k}}}_1\\
    \leq& \mathcal{O}(L_\infty\poly(H)A\sqrt{d\beta K})\,.
\end{align*}

\section{General Model-free Framework: Algorithm \texttt{RS-DisRL-V}}
In this section we review the general model-free framework algorithm.
\begin{algorithm}[htbp]
   \caption{\texttt{RS-DisRL-V}}
\begin{algorithmic}[1]
   \STATE {\bfseries Input:} Function class $\bm{\mathcal{Z}}=\mathcal{Z}_1\times\mathcal{Z}_2\cdots\mathcal{Z}_H$, confidence radius $\gamma$.
   \STATE {\bfseries Initialize:} $\wh{\gZ}_{1,\mpi} \leftarrow \gZ$.
   \FOR{$k=1$ {\bfseries to} $K$}
   \STATE 
   $({\mpi}_k, \widehat{  Z}^k)=\argmax_{{\mpi}\in\Pi^\dag, { Z}\in\widehat\gZ_{k,\mpi}}\rho(Z_1)$.~\textcolor{blue}{//Optimistic planning}
   \STATE Execute policy $\mpi^{k}$, add the collected data $\bm{\tau}_k=\sets{(s_{k,h},a_{k,h},r_{k,h})}_{h=1}^H$ and $\mpi^{k}$, $\widehat{\mtheta}_k$ into history $\cH_k = \cH_{k-1}\cup \{(\bm{\tau}_k, \mpi^{k}, \wh{Z}^k)\}$.~\textcolor{blue}{//Data collection}
   \STATE $\widehat\gZ_{k+1, \bm{\pi}}=\texttt{V-Est}(\mathcal{H}_k, \bm{\gZ}, \bm\pi,\gamma)$.~\textcolor{blue}{//Confidence set construction}
\ENDFOR
\end{algorithmic}
\end{algorithm}

We restate Condition \ref{con:mfconcentration} and Condition \ref{con:mfelliptical} as below:
\begin{condition}\label{lem:model free framework concentration}
    For all policy $\mpi\in\Pi^\dag$, we have that the random variable representing the true return is in the version space:
    \begin{align*}
        Z^\mpi\in\widehat{\mathcal{Z}}_{k+1,\mpi}\,.
    \end{align*}
    established with probability at least $1 - \delta$, $\delta \in (0, 1]$.
\end{condition}
\begin{condition}\label{lem:model free pigeon hole}
    For $0 < \delta \leq 1$, we have 
    \begin{align*}
        \sum_{k=1}^K \norm{F_{\widehat{Z}^k}-F_{Z^{\mpi^k}}}_\infty\leq \zeta(K,H,\operatorname{d},\bm\gZ,\Pi^\dag,\delta,\gamma)\,,
    \end{align*}
    holds with probability at least $1 -\delta$.
    Here $\operatorname{d}$ is some structural complexity measure of the problem.
\end{condition}
Given these two conditions, our regret bound can be stated as follows:
\begin{theorem}[Full version of Theorem \ref{thm:mfmeta}]
    Under the general value function approximation (Assumption~\ref{ass:mfbellman}) If the given estimation functio \texttt{V-Est} satisfies Condition \ref{lem:model free framework concentration} and Condition \ref{lem:model free pigeon hole}, we have:
    \begin{align*}
        \sum_{k=1}^K \rho(Z^{\mpi^*})-\rho(Z^{\mpi^k})\leq L_\infty(\rho) \zeta(K,H,\operatorname{d},\bm\gZ,\Pi^\dag,\delta,\gamma)\,.
    \end{align*}
    holds with probability at least $1 -\delta$, $\delta \in (0, 1]$.
\end{theorem}
\begin{proof}
    Since the concentration condition \ref{lem:model free framework concentration} holds, we have for any $\mpi\in\Pi^\dag$ and $k\in K$, 
    \begin{align*}
        \rho(Z^{\mpi})\leq \max_{Z\in\widehat{\gZ}_{k,\mpi}}\rho(Z)\leq \rho(\widehat{Z}^k)~\,.
    \end{align*}
    Thus, we have:
    \begin{align*}
        \sum_{k=1}^K\rho(Z^{\mpi^*})-\rho(Z^{\mpi^k})\leq&\sum_{k=1}^K \rho(\widehat{Z}^k)-\rho(Z^{\mpi^k})
        \leq L_\infty \sum_{k=1}^K \norm{F_{\widehat{Z}^k}-F_{Z^{\mpi^k}}}_\infty
        \leq L_\infty \zeta(K,H,\operatorname{d},\gZ,\Pi^\dag,\delta,\gamma)\,.
    \end{align*}
    which gives this result.
\end{proof}


\subsection{Policy Cover}\label{sec:policy cover}
Notice that our regret is defined via the optimal policy in the policy set, which is adopted in many model free valued-based scenarios, such as \cite{xie2021bellman,wang2023benefits}. The main reason why our algorithm can only operate in a given policy set is that the optimal risk-sensitive policy can not be computed via dynamical programming. In contrast, in the risk neutral setting we can always select the greedy policy and ensure that it is the optimistic policy given our estimation. We remark, however, that when specified to specific risk measures such as OCE \cite{xu2023regret}, CVaR \cite{wang2023nearminimaxoptimal} and ERM \cite{fei2020risksensitive}, where the optimal policy have a similar greedy property, we can ensure global optimality without the policy set.

In this section, we discuss the policy covering given a policy class $\Pi^\dag$.
For any policy $\pi\in\Pi^\dag:\cS^\dag\rightarrow\Delta(\cA)$, we define its lower $\epsilon$-bracket $\pi^{\downarrow}$ as $\pi^{\downarrow}\leq \pi$ and $\norm{\pi^\downarrow(\cdot|s^\dag)-\pi(\cdot|s^\dag)}_1\leq \epsilon$ for all $s^\dag$. Since $\pi^{\downarrow}$ may not be a valid distribution, we define its normalized version as: $\underline{\pi}(a|s^\dag)=\pi^{\downarrow}(a|s^\dag)/\int_{\cA}\pi^{\downarrow}(a'|s^\dag)$. Since $1-\epsilon\leq \int_{\cA}\pi^{\downarrow}(a'|s^\dag)\leq 1$, we have that $\underline{\pi}(a|s^\dag)\leq(1+2\epsilon)\pi^{\downarrow}(a|s^\dag)$.
Its bracketing number is denoted as $\cN_{[\cdot]}(\Pi^\dag,\epsilon,\norm{\cdot}_{1})$.

\textbf{Instances}
Consider a softmax policy set $\Pi(U,\tau)$ with temperature $\tau$ and utility function $u\in U:\cS^\dag\times\cA\rightarrow\RR$:
$\pi(a|s^\dag)=\frac{e^{\tau u(s^\dag,a)}}{\int_{\cA}e^{\tau u(s^\dag,a')}}$. We consider a $\epsilon'=\frac{\epsilon}{8\tau}$ covering of $U$, such that for any $ u$, there exists $\underline{u}$ in the covering $\underline{U}$ with $\norm{\underline{u}-u}_\infty\leq \epsilon'$.
and we can construct the lower bracket as:$\pi^\downarrow=\frac{e^{\tau (\underline{u}(s^\dag,a)-\epsilon')}}{\int_{\cA}e^{\tau (\underline{u}(s^\dag,a')+\epsilon')}}$. Its normalized version is $\underline{\pi}=\frac{e^{\tau \underline{u}(s^\dag,a)}}{\int_{\cA}e^{\tau \underline{u}(s^\dag,a')}}$. We can verify that: $\pi^\downarrow\leq\pi(a|s^\dag)$ and $
    \norm{\pi^\downarrow(\cdot|s^\dag)-\pi(\cdot|s^\dag)}_1=1-\int_{\cA}\pi^\downarrow(a|s^\dag)=1-e^{-2\tau\epsilon'}\leq\epsilon$. So the bracketing number is the same as the $\frac{\epsilon}{8\tau}$ covering number of the utility function. $\cN_{[\cdot]}(\Pi(U,\tau),\epsilon,\norm{\cdot}_1)=\cN_C(U,\frac{\epsilon}{8\tau},\norm{\cdot}_\infty)$.

Moreover, we have for any $s^\dag$, $ \underline{\pi}(a|s^\dag)= e^{2\tau\epsilon'}\frac{e^{\tau (\underline{u}(s^\dag,a)-\epsilon')}}{\int_{\cA}e^{\tau (\underline{u}(s^\dag,a')+\epsilon')}}\leq (1+\epsilon)\pi(a|s^\dag)$ and $\underline{\pi}(a|s^\dag)= e^{-2\tau\epsilon'}\frac{e^{\tau (\underline{u}(s^\dag,a)+\epsilon')}}{\int_{\cA}e^{\tau (\underline{u}(s^\dag,a')-\epsilon')}}\geq (1+\epsilon)\pi(a|s^\dag)$. Thus, $ \norm{\pi(\cdot|s^\dag)-\underline{\pi}(\cdot|s^\dag)}\leq \int_{a}\abs{\underline{\pi}(a|s^\dag)-\pi(a|s^\dag)}\leq \int_a \epsilon\pi(a|s^\dag)=\epsilon$, we have that $\underline{\Pi}$ is also a $\epsilon$-cover of $\Pi$ under $\ell_1$ norm.

\section{Model-Free Estimation by LSR Approach}\label{app:mflsr}

In a model-free environment, we assume that the random variable of the cumulative reward $Z^{\bm \pi}$ is determined by the cumulative distribution function $\bm{F}$. Since the algorithm is given a random variable function class $\bm\gZ$, we assume for every $\bm{Z}\in\bm\gZ$, its CDF $\mF\in\bm\cF$. Investigate $\bm\gZ$ is equal to explore the CDF class $\bm\cF$. Throughout this section, we use CDF to characterize the random variable.

\subsection{Estimation and Algorithms}

In this section, we use the least squares regression to estimate the confidence set of CDF.

Here we need the covering for $\Pi^\dag$ and $\gF$, defined in Definition \ref{def:cover}. For the policy set $\Pi^\dag$ and function set $\cF$, we use the metric: $\rho(\pi^1,\pi^2)= \max_{s^\dag}\norm{\pi_1\bracket{\cdot|s^\dag}-\pi_2\bracket{\cdot|s^\dag}}_1$ and $\rho(F_1,F_2)=\max_{s^\dag,a}\norm{F_1(\cdot|s^\dag,a)-F_2(\cdot|s^\dag,a)}_{\infty}$. For any policy $\pi\in\Pi^\dag$ and $F\in\cF$, we denote its $\epsilon$-approximation in the cover $\underline{\Pi}^\dag$ and $\underline{\cF}$ as $\underline{\pi}$ and $\underline{F}$ respectively. Since $\underline{\pi}$ and $\underline{F}$ are $\epsilon$-approximations of $\pi$ and $F$, we have $\norm{\pi-\underline{\pi}}\leq\epsilon$ and $\norm{F-\underline{F}}\leq\epsilon$. We denote the covering number as $\cN_C\bracket{\Pi^\dag,\epsilon,\norm{\cdot}_1}$ and $\cN_C(\cF,\epsilon,\norm{\cdot}_\infty)$ respectively.
\begin{algorithm}[htbp]
   \caption{\texttt{V-Est-LSR}$(\gH_{k-1}, \bm\gF, \mpi,\gamma^{\LSR})$}
   \label{alg:mfestlsr}
\begin{algorithmic}
   \STATE {\bfseries Input:} History information $\gH_k$, CDF Model class $\gF$, and policy $\mpi$.
   \STATE {\bfseries Return:} $\widehat\gF_{k,\mpi} \leftarrow \left\{ \mF \in \bm\gF : F_h \in \widehat{\gF}^\LSR_{k,h,{\mpi}, {F}} , \forall h \in [H]\right\}$, where $\widehat{\gF}^\LSR_{k,h,\mpi, \widetilde{F}}$ is defined by
   \begin{equation*}
       \widehat{F}_{k,h,\underline{\mpi}, \underline{\wt{F}}} = \argmin_{F_h \in \gF_h} \sum_{i=1}^{k-1}\left( F_h(x_{i,h}^{\underline{\mpi},\underline{\wtF}} \mid s_{i,h}^\dag, a_{i,h}) - \int_{a_{h+1}} \pi_h(a_{h+1} \mid s_{i,h+1}^\dag)   \wtF_{h+1}(x_{i,h}^{\underline{\mpi},\underline{\wtF}} - r_{i,h} \mid s_{i,h+1}^\dag, a_{h+1}) \right)^2\,,
   \end{equation*}
   \begin{equation*}
       \widehat{\gF}^\LSR_{k,h,\mpi, \widetilde{F}}= \left\{ F_h \in \gF_h : {\sum_{i=1}^{k-1} \left( F_h(x_{i,h}^{\underline{\mpi},\underline{\wtF}}\mid s_{i,h}^\dag, a_{i,h}) - \widehat{F}_{k,h,\underline{\mpi},\underline{\wtF}}(x_{i,h}^{\underline{\mpi},\underline{\wtF}} \mid s_{i,h}^\dag, a_{i,h}) \right)^2} \leq \gamma^{\LSR}
       \right\}\,.
   \end{equation*}
\end{algorithmic}
\end{algorithm}

We define the LSR-type distance function used here:
\begin{align}
    \operatorname{Dist}^\LSR_{k,h,\mpi,{\wtF}}(F_1 || F_2) = {\sum_{i=1}^{k-1}\left( F_1(x_{i,h}^{{\mpi},{\wtF}} \mid s_{i,h}^\dag,a_{i,h}) - F_2(x_{i,h}^{{\mpi},{\wtF}} \mid s_{i,h}^\dag, a_{i,h}) \right)^2}\,,
\end{align}
and we define $x_{i,h}^{\mpi,\wtF}$ as below, which represents the direction with largest uncertainty.
\begin{equation}\label{eq:x_i,h mflsr}
    x_{i,h}^{\mpi,\wtF} = \argmax_{x\in[0,H]}\left| \sup_{F_1 \in \widehat{F}_{i,h,\mpi,\wtF}} F_1(x \mid s_{i,h}^\dag,a_{i,h}) - \inf_{F_2 \in \widehat{F}_{i,h,\mpi,\wtF}} F_2(x \mid s_{i,h}^\dag,a_{i,h})  \right|\,.
\end{equation}
We now describe our estimation procedure above. For a target random variable with CDF $\widetilde{{F}}_{h+1}$ and policy $\mpi\in\Pi^\dag$, we estimate $\cT_{h,\mpi}^{\dag}\widetilde{F}_{h+1}$ via least squares: $\widehat{F}_{k,h,\mpi,\widetilde{F}}$. Define the distance function
\begin{align*}
\dist_{k,h,\underline{\mpi},\underline{\wtF}}^{\LSR}\bracket{F_h||  \widehat{F}_{k,h,\underline{\mpi}, \underline{\wt{F}}}} = {\sum_{i=1}^{k-1} \left( F_h(x_{i,h}^{\underline{\mpi},\underline{\wtF}}\mid s_{i,h}^\dag, a_{i,h}) - \widehat{F}_{k,h,\underline{\mpi},\underline{\wtF}}(x_{i,h}^{\underline{\mpi},\underline{\wtF}} \mid s_{i,h}^\dag, a_{i,h}) \right)^2}
\end{align*}
Then we can rewrite our version sapce using the distance metric defined above:
\begin{equation*}
    \widehat{\cF}_{k,h,\mpi,\Tilde{F}}^{\LSR}=\sets{F_h\in\cF_h: \dist_{k,h,\underline{\mpi},\underline{\wtF}}^{\LSR}\bracket{F_h||  \widehat{F}_{k,h,\underline{\mpi}, \underline{\wt{F}}}}\leq \gamma^{\LSR}}\,,
\end{equation*}
with the confidence radius $\gamma^{\LSR}=16\log(HK^2/\delta) + \log(\gN_C(\Pi^\dag, 1/K,\|\cdot\|_1)) + \log(\gN_C(\bm\gF, 1/K, \|\cdot\|_\infty))$.

The next lemma shows the one-step-back concentration guarantee, which will be used to prove Condition \ref{con:mfconcentration}.
\begin{lemma}\label{lem:model free lsr onestep concentration}
    For any $\pi\in\Pi^\dag$, $F\in\cF$, $h\in[H]$, we have with probability at least $1-\delta$, for all $k\in[K]$:
    \begin{align*}
        \cT_{h,\mpi}^\dag \Tilde{F}_{h+1}\in\wh{\cF}_{k,h,\underline{\mpi},\underline{\wtF}}^{\LSR}=\wh{\cF}_{k,h,\mpi,\wtF}^{\LSR}\,.
    \end{align*}
\end{lemma}
\begin{proof}
First we fix $h \in [H]$, $\tilde{F}$ and $\pi$. 
    Since
    \begin{align*}
        \left\{\gT^\dag_{h,\mpi}\wtF_h(x_{i,h}^{\mpi, \wtF} \mid s_{i,h}^\dag, a_{i,h}) - \int_{a_{h+1}} \pi_{h+1}(a_{h+1} \mid s_{i,h+1}^\dag) \wtF_{h+1}(x_{i,h}^{\mpi,\wtF} - r_{i,h} \mid s_{i,h+1}^\dag, a_{h+1})  \right\}_{i=1}^k
    \end{align*}
    is a $1$-sub-Gaussian. Moreover, we have
    \begin{align*}
        &\EE\left[ \int_{a_h+1} \pi_{h+1}(a_{h+1} \mid s_{i,h+1}^\dag) \wtF_{h+1}(x_{i,h}^{\mpi, \wtF} - r_{i,h} \mid s_{i,h+1}^\dag, a_{h+1}) \middle | \tau_{i,h} \right] \\
        =& \int_{s_{h+1}} \Prob_h(s_{h+1} | s_{i,h}, a_{i,h})\int_{r_{h+1}}\sR_h(r_{h+1} | s_{i,h}, a_{i,h}) \int_{a_{h+1}} \pi_{h+1}(a_{h+1} | s_{h+1}^\dag) \wtF_{h+1}(x_{i,h}^{\mpi, \wtF} | s_{h+1}^\dag, a_{h+1}) \\
        =& \int_{s_{h+1}^\dag}\sT_h(s_{h+1}^\dag | s_{i,h}^\dag, a_{i,h})\int_{a_{h+1}} \pi_{h+1}(a_{h+1} | s_{h+1}^\dag) \wtF_{h+1}(x_{i,h}^{\mpi, \wtF} - r_{h+1} | s_{h+1}^\dag, a_{h+1}) \\
        =& \gT_{h,\mpi}\wtF_{h+1}(x_{i,h}^{\mpi, \wtF} | s_{i,h}^\dag, a_{i,h})\,,
    \end{align*}
    where $\tau_{i,h}$ denotes history up to and include step $h$ in episode $i$.
    Thus by Lemma~\ref{lem:model free lsr auxillary concentration}, we have with probability at least $1 - \delta / H$, for all $k \in [K]$,

    \begin{equation}
        \dist_{k,h,\mpi,\tilde{F}}^{\LSR}\bracket{\cT_{h,\mpi}^\dag\tilde{F}_{h+1}||\widehat{F}_{k,h,\mpi,\tilde{F}}}\leq 8\log\left( \frac{2H}{\delta}  \right)  + 4\left(1 + \sqrt{\log\left(\frac{4HK^2}{\delta}\right)}\right)\,.
    \end{equation}

Applying a union bound for all $h\in[H]$, $\underline{\wtF}\in\underline{\cF}$, and $\underline{\mpi}\in\underline{\Pi}$, we have:

 \begin{align*}
\dist_{k,h,\underline{\mpi},\underline{\wtF}}^{\LSR}\bracket{\cT_{h,\underline{\mpi}}^\dag\underline{\tilde{F}}_{h+1}||\widehat{F}_{k,h,\underline{\mpi},\underline{\wtF}}}\leq\mathcal{O}(\gamma^{\LSR}) \,.
 \end{align*}

    Moreover, we have
    \begin{align*}
        \operatorname{Dist}^\LSR_{k,h,\underline{\mpi},\underline{\wtF}}(\gT^\dag_{h,\mpi}\wtF_{h+1} || \gT^\dag_{h, \underline{\mpi}} \underline{\wtF}_{h+1})\leq\operatorname{Dist}^\LSR_{k,h,\underline{\mpi},\underline{\wtF}}(\gT^\dag_{h,\mpi}\wtF_{h+1} || \gT^\dag_{h, \underline{\mpi}} {\wtF}_{h+1})+\operatorname{Dist}^\LSR_{k,h,\underline{\mpi},\underline{\wtF}}(\gT^\dag_{h,\underline{\mpi}}\wtF_{h+1} || \gT^\dag_{h, \underline{\mpi}} \underline{\wtF}_{h+1}) \,.
    \end{align*}
    For the first term, we have:
    \begin{align*}
        &\operatorname{Dist}^\LSR_{k,h,\underline{\mpi},\underline{\wtF}}(\gT^\dag_{h,\mpi}\wtF_{h+1} || \gT^\dag_{h, \underline{\mpi}} {\wtF}_{h+1})\\
        =&{\sum_{i=1}^{k-1}\bracket{\int_{s_{h+1}^\dag}\TT_h(s_{h+1}^\dag|s_{i,h}^\dag,a_{i,h})\int_{a_{h+1}}\bracket{\pi_{h+1}(a_{h+1}|s_{h+1}^\dag)-\underline{\pi}_{h+1}(a_{h+1}|s_{h+1}^\dag)}\widetilde{F}_{h+1}(x_{i,h}^{\underline{\mpi},\underline{\wtF}} - r_{h} | s_{h+1}^\dag, a_{h+1})}^2}\\
        \leq& {\sum_{i=1}^{k-1}\bracket{\max_{s_{h+1}^\dag}\norm{\pi_{h+1}(\cdot|s_{h+1}^\dag)-\underline{\pi}_{h+1}(\cdot|s_{h+1}^\dag)}_1}^2}\\
        \leq& K\epsilon\,.
    \end{align*}
    For the second term, we also have:
    \begin{align*}
        &\operatorname{Dist}^\LSR_{k,h,\underline{\mpi},\underline{\wtF}}(\gT^\dag_{h,\underline{\mpi}}\wtF_{h+1} || \gT^\dag_{h, \underline{\mpi}} \underline{\wtF}_{h+1}) \\
        =& {\sum_{i=1}^k \bracket{\int_{s_{h+1}^\dag}\TT_h(s_{h+1}^\dag|s_{i,h}^\dag,a_{i,h})\int_{a_{h+1}}\underline{\pi}_{h+1}(a_{h+1}|s_{h+1}^\dag)\bracket{\widetilde{F}_{h+1}-\underline{\wtF}_{h+1}}(x_{i,h}^{\underline{\mpi},\underline{\wtF}} - r_{h} | s_{h+1}^\dag, a_{h+1})}^2}\\
        \leq& {\sum_{i=1}^k\max_{a_{h+1},s_{h+1}^\dag}\norm{\underline{\wtF}_{h+1}(s_{h+1}^\dag,a_{h+1})-\wtF_{h+1}(s_{h+1}^\dag,a_{h+1})}_{\infty}^2}\\
        \leq& K\epsilon\,.
    \end{align*}
    Thus we have 
    \begin{align*}
        &\dist_{k,h,\underline{\mpi},\underline{\wtF}}^{\LSR}\bracket{\gT^\dag_{h,\mpi}\wtF_{h+1}||\widehat{F}_{k,h,\underline{\mpi},\underline{\wtF}}}\\
        \leq& \dist_{k,h,\underline{\mpi},\underline{\wtF}}^{\LSR}\bracket{\cT_{h,\underline{\mpi}}^\dag\underline{\tilde{F}}_{h+1}||\widehat{F}_{k,h,\underline{\mpi},\underline{\wtF}}}+\operatorname{Dist}^\LSR_{k,h,\underline{\mpi},\underline{\wtF}}(\gT^\dag_{h,\mpi}\wtF_{h+1} || \gT^\dag_{h, \underline{\mpi}} \underline{\wtF}_{h+1})\\
        \leq&\mathcal{O}(\gamma^{\LSR}+K\epsilon)\\
        =& \mathcal{O}(\gamma^{\LSR})\,.
    \end{align*}
    
    From the definition of the confidence set, we have for any $\tilde{F}_{h+1}\in\cF_{h+1}$ and $\mpi\in\Pi^\dag$, $\gT_{h, \mpi}^\dag \tilde{F}_{h+1} \in \wh{\cF}_{k,h,\underline{\mpi},\underline{\wtF}}^{\LSR}=\widehat{\gF}^\LSR_{k,h,{\mpi}, {\wtF}}$ where the equality is because the $\epsilon$-approximation of $\underline{\pi}$ and $\underline{\wtF}$ are themselves.
\end{proof}

\begin{proof}[Proof of Condition \ref{con:mfconcentration}]
We proof the Lemma via induction. If $F_{h+1}^\mpi\in\widehat{\cF}_{k,h+1,{\mpi},{F}}^{\LSR}$, we have $F_h^\mpi=\cT_{h,\mpi}^\dag F_{h+1}^\mpi\in\widehat{\cF}^{\LSR}_{k,h,{\mpi},{F}}$. So for all $h\in[H]$ we have $F_h^\mpi=\cT_{h,\mpi}^\dag F_{h+1}^\mpi\in\widehat{\cF}^{\LSR}_{k,h,{\mpi},{F}}$. From the definition of $\cF_{k,\pi}$, we have $\mF^\mpi\in\widehat{\cF}_{k,\mpi}$.
\end{proof}
The next lemma decomposes the supremum distance between the CDFs of the cumulative return via the bellman error, which is the distributional analogue of the performance difference lemma.
\begin{lemma}[Performance difference]\label{lem:model free lsr difference}
For any random variable $Z$ representing the estimated cumulative return, with CDF function $\bm{F}=F_1\times F_2\times \cdots F_H$, we can decompose the $\ell_\infty$ distance between the estimated return CDF $F_Z$ and the real return CDF $F_{Z^\mpi}$ for policy $\mpi\in\Pi^\dag$ by the bellman error as follows:
    \begin{align*}
        \norm{F_{Z}-F_{Z^\mpi}}_{\infty}\leq \sum_{h=1}^H \EE_{\mu^{\mpi}}\norm{F_h(\cdot|s_{h}^\dag,a_h)-\cT_{h,\mpi}^\dag F_{h+1}(s_h^\dag,a_h)}_\infty\,.
    \end{align*}
\end{lemma}

\begin{proof}
    \begin{align*}
        &\norm{F_{Z}-F_{Z^\mpi}}_{\infty}\\
        =&\sup_{x\in[0,H]} \abs{\int_{a_1}\pi_1(a_1|s_1^\dag) \bracket{F_1(x|s_1^\dag,a_1)-F_1^\mpi(x|s_1^\dag,a_1)}}\\
        \leq& \EE_{s_1^\dag,a_1\sim \mu^\mpi} \norm{F_1(s_1^\dag,a_1)-F_1^\mpi(s_1^\dag,a_1)}_{\infty}\\
        =& \EE_{s_1^\dag,a_1\sim \mu^\mpi} \norm{F_1(s_1^\dag,a_1)-\cT_{h,\mpi}^\dag F_2(s_1^\dag,a_1)}_{\infty}\\
        \leq&\EE_{\mu^\mpi}\norm{F_1(s_1^\dag,a_1)-\cT_{h,\mpi}^\dag F_2(s_1^\dag,a_1)}_{\infty}+\EE_{\mu^\mpi}\norm{\cT_{h,\mpi}^\dag F_2(s_1^\dag,a_1)-\cT_{h,\mpi}^\dag F_2^\mpi(s_1^\dag,a_1)}_\infty\\
        \leq & \EE_{\mu^\mpi}\norm{F_1(s_1^\dag,a_1)-\cT_{h,\mpi}^\dag F_2(s_1^\dag,a_1)}_{\infty}+\EE_{\mu^\mpi}\norm{F_2(s_2^\dag,a_2)-F_2^\mpi(s_2^\dag,a_2)}_\infty\,,
    \end{align*}
    where the first and second inequalities holds by triangle inequality, the third inequality is because:
    \begin{align*}
        &\EE_{\mu^\mpi}\norm{\cT_{h,\mpi}^\dag F_2(s_1^\dag,a_1)-\cT_{h,\mpi}^\dag F_2^\mpi(s_1^\dag,a_1)}_\infty\\
        =& \int_{s_1^\dag,a_1}\mu^\mpi(s_1^\dag,a_1)\sup_{x\in[0,H]}\abs{\int_{s_1^\dag,a_2}\TT(s_2^\dag|s_1^\dag,a_1)\pi(a_2|s_2^\dag) \bracket{F_2(x|s_2^\dag,a_2)-F_2^\mpi(x|s_2^\dag,a_2)}}\\
        \leq& \int_{s_1^\dag,a_1}\mu^\mpi(s_1^\dag,a_1)\int_{s_1^\dag,a_2}\TT(s_2^\dag|s_1^\dag,a_1)\pi(a_2|s_2^\dag)\norm{F_2(s_2^\dag,a_2)-F_2^\mpi(s_2^\dag,a_2)}_{\infty}\\
        =&\EE_{\mu^\mpi}\norm{F_2(s_2^\dag,a_2)-F_2^\mpi(s_2^\dag,a_2)}_\infty\,.
    \end{align*}
    Repeat this analysis for every step $h \in [H]$, we have
    \begin{align*}
        &\norm{F_{Z}-F_{Z^\mpi}}_{\infty}\\
        \leq & \EE_{\mu^\mpi}\norm{F_1(s_1^\dag,a_1)-\cT_{h,\mpi}^\dag F_2(s_1^\dag,a_1)}_{\infty}+\EE_{\mu^\mpi}\norm{F_2(s_2^\dag,a_2)-F_2^\mpi(s_2^\dag,a_2)}_\infty\\
        \leq &\cdots \\
        \leq& \sum_{h=1}^H \norm{F_h(\cdot|s_{h}^\dag,a_h)-\cT_{h,\mpi}^\dag F_{h+1}(s_h^\dag,a_h)}_\infty
    \end{align*}
\end{proof}

Equipped with the technical lemmas above, we are able to prove the Condition~\ref{con:mfelliptical} for model-free LSR estimation function.
\begin{lemma}[Condition \ref{con:mfelliptical}]
    For $0 < \delta \leq 1$, we have 
    \begin{align*}
        \sum_{k=1}^K \norm{F_{\widehat{Z}^k}-F_{Z^{\mpi^k}}}_\infty\leq \mathcal{O}\bracket{\poly(H)\sqrt{K\gamma^{\LSR}\dim_E(\gF, \sqrt{K})}}\,,
    \end{align*}
    holds with probability at least $1 -\delta$.
\end{lemma}
\begin{proof}[proof of Condition \ref{con:mfelliptical}]
 Using Hoeffding inequality in Lemma \ref{lem:model free lsr difference}, we have:
    \begin{align*}
        &\norm{F_{\widehat{Z}^k}-F_{Z^{\mpi^k}}}_\infty\\
        \leq& \sum_{h=1}^H \EE_{\mu^{\mpi^k}}\norm{\widehat{F}_h^k(\cdot|s_{h}^\dag,a_h)-\cT_{h,\mpi}^\dag \widehat{F}_{h+1}^k(s_h^\dag,a_h)}_\infty\\
        \leq& \sum_{h=1}^H \norm{\widehat{F}^k_h(s_{k,h}^\dag,a_{k,h})-\cT_{h,\mpi}^\dag \widehat{F}^k_{h+1}(s_{k,h}^\dag,a_{k,h})}_{\infty}+\mathcal{O}\bracket{\sqrt{KH\log\frac{1}{\delta}}}\,.
    \end{align*}
    Since $\cT_{h,\mpi}^\dag \widehat{F}^k_{h+1}\in\widehat{\cF}^{\LSR}_{k,h,\underline{\mpi}^k,\underline{\widehat{F}}^k}$ by Lemma $\ref{lem:model free lsr onestep concentration}$, we have:
    \begin{align*}
    &\sum_{k=1}^K\sum_{h=1}^H\norm{\widehat{F}^k_h(s_{k,h}^\dag,a_{k,h})-\cT_{h,\mpi}^\dag \widehat{F}^k_{h+1}(s_{k,h}^\dag,a_{k,h})}_{\infty}\\
    \leq &\sum_{k=1}^K\sum_{h=1}^H \sup_{x\in[0,H]}\left|\sup_{F_1 \in \widehat{\cF}^{\LSR}_{k,h,\underline{\mpi}^k,\underline{\widehat{F}}^k}} F_1(x | s_{k,h}^\dag, a_{k,h}) - \inf_{F_2 \in \widehat{\cF}^{\LSR}_{k,h,\underline{\mpi}^k,\underline{\widehat{F}}^k}} F_2(x | s_{k,h}^\dag, a_{k,h})\right| \\
    =&\sum_{k=1}^K\sum_{h=1}^H\sup_{F_1 \in \wh{\gF}_{k,h,{\mpi}^k, \wh{F}^k}^\LSR} F_1(x_{k,h}^{\underline{\mpi}^k,\underline{\widehat{F}}^k} | s_{k,h}^\dag, a_{k,h}) - \inf_{F_2 \in \wh{\gF}_{k,h,{\mpi}^k, \wh{F}^k}^\LSR} F_2(x_{k,h}^{\underline{\mpi}^k,\underline{\widehat{F}}^k} | s_{k,h}^\dag, a_{k,h})\,,
\end{align*}
    which is by the definition of $x_{k,h}^{\underline{\mpi^k},\underline{\widehat{F}}^k}$ in Equation \ref{eq:x_i,h mflsr}.
    Denote 
    \begin{align*}
        G_{k,h}=\sup_{F_1 \in \wh{\gF}_{k,h,{\mpi}^k, \wh{F}^k}^\LSR} F_1(x_{k,h}^{\underline{\mpi}^k,\underline{\widehat{F}}^k} | s_{k,h}^\dag, a_{k,h}) - \inf_{F_2 \in \wh{\gF}_{k,h,{\mpi}^k, \wh{F}^k}^\LSR} F_2(x_{k,h}^{\underline{\mpi}^k,\underline{\widehat{F}}^k} | s_{k,h}^\dag, a_{k,h})\,.
    \end{align*}
    Using similar techniques as Lemma 9 of \cite{chen2023provably}, we have:
     \begin{equation*}
        \sum_{k=1}^K \sum_{h=1}^H G_{k,h}^2 \leq H + H\dim_E(\gF, \sqrt{K}) + 4H\gamma^{\LSR}\dim_E(\gF, \sqrt{K})(\log(K) + 1)\,.
    \end{equation*}
    Thus, using Cauchy inequality, we obtain:
    \begin{align*}
        \sum_{k=1}^K\norm{F_{\widehat{Z}^k}-F_{Z^{\mpi^k}}}_\infty\leq\mathcal{O}\bracket{\poly(H)\sqrt{K\gamma^{\LSR}\dim_E(\gF, \sqrt{K})}}\,.
    \end{align*}
\end{proof}

\section{Model-Free Estimation by MLE approach}\label{app:mfmle}
\subsection{Bellman Eluder Dimension}
In this section, we define the bellman eluder dimension \cite{jin2021bellman}, which is a famous structural complexity. First, we define the $\ell_2$ norm distributional eluder dimension for a function class (Definition 7 in \cite{jin2021bellman}).
\begin{definition}[$\ell_2$ norm distributional eluder dimension]\label{ass:ell_2 dist eluder dim}
    We consider $\Phi$ be a function class on domain $\mathcal{X}$ where for $\phi\in\Phi$, $\abs{\phi(x)}\leq1$. $\cD$ is a family of distributions on $\mathcal{X}$.
    Let $L$ be the longest sequence that there exists $\epsilon'>\epsilon$ and $\mu_1\cdots \mu_L\in\cD$, for all $t\in[L]$, there exists $\phi\in\Phi$,  $\abs{\EE_{ \mu_t}[\phi(x)]}\geq\epsilon$ and $\sum_{i=1}^{t-1} \bracket{\EE_{ \mu_i}[\phi(x)]}^2\leq \epsilon^2$. We denote $L$ as the bellman eluder dimension $\operatorname{d_{DE}}(\Phi,\cD,\epsilon)$.
\end{definition}
Given the function set have low eluder dimension, we have the standard elliptical potential lemma as below:
\begin{lemma}[Lemma 17 in \cite{jin2021bellman}] \label{lem:bellman eluder dimension}
    Given a function class $\phi\in\Phi$ in domain $\mathcal{X}$ with $\abs{\phi(x)}\leq1$. Let $\cD$ be a families of distributions on $\mathcal{X}$. Suppose $\sets{\phi_k}_{[K]}\subset \Phi$ and $\sets{\mu_k}_{[K]}\subset \cD$ be two sequences. If for any $k\in [K]$, $\sum_{i=1}^k\bracket{\EE_{\mu_i}[\phi_k]}^2\leq\beta $, then for any $k\in[K]$, $\sum_{i=1}^k \abs{\EE_{\mu_i}[\phi_i]}\leq\mathcal{\Tilde{O}}\bracket{\sqrt{\operatorname{d_{DE}}(\Phi,\cD,1/K)\beta K}}$.
\end{lemma}
In this section, we define our bellman eluder dimension as the distributional eluder dimension for the specific function class below:
\begin{definition}[Bellman Eluder Dimension]\label{ass:bellman eluder dim}
    Given a policy class $\Pi^\dag$, and a PDF function class $\cF$. For $h\in[H]$, we define the function class $\Phi_h$ as $\sets{\TV\bracket{f_h||\cT_{h,\mpi}f_{h+1}}}$, and the distribution family $\cD_h$ as $\mu_h^\mpi$ with domain $\gX_h$ as $s_h^\dag,a_h$. We define the bellman eluder dimension of our problem as:
    \begin{align*}
        \operatorname{d_{BE}}=\max_{h\in [H]}\operatorname{d_{DE}}(\Phi_h,\cD_h,1/K)\,.
    \end{align*}
\end{definition}
    
\subsection{Setting}
\subsubsection{Notation}\label{sec:mfmlenotation}
In this setting, we assume the density function of $Z_h^\mpi\in\mathcal{Z}_h$ belongs to a function class $\cF_h$. We denote the density function of $Z_h^\mpi(s^\dag,a)\in\mathcal{Z}_h$ at point $z\in\RR$ as $f_h^\mpi(z|s^\dag,a)\in\cF_h$.

Consider an upper and lower $\epsilon$-bracketing of $\cF$ under $\norm{\cdot}_1$, denoted as ${\cF}^{\uparrow}$ and ${\cF}^{\downarrow}$. We denote the corresponding lower bracket of $g$ as $g^{\downarrow}$, and the upper bracket of $f$ as $f^{\uparrow}$. Since $g^{\downarrow}$ may not be a valid distribution, we denote the normalized version as: $\underline{g}$, where $\underline{g}=g^{\downarrow}/\int_z g^{\downarrow}(z)$, and $1-\epsilon\leq\int_z g^{\downarrow}(z)\leq 1$. Thus, we have: $\underline{g}\leq(1+2\epsilon)g^{\downarrow}$.

\subsubsection{Bellman Completeness}
Define the augmented bellman operator:
\begin{equation*}
    \mathcal{T}_{h,\mpi}^\dag f_{h+1}(z|s_h^\dag,a_h)=\int \TT({s_{h+1}^\dag|s_h^\dag,a_h})\pi_{h+1}(a_{h+1}|s_{h+1}^\dag) f_{h+1}(z-r_h|s_h^\dag,a_h)\,.
\end{equation*}
Here for completeness we  restate the distributional bellman completeness assumption, and give a corresponding example.
\begin{assumption}[augmented distributional bellman completeness]
 For the density function class $\bm{\cF}=\cF_1\times\cdots\cF_H$ corresponding to the class of random variables $\bm{\gZ}=\mathcal{Z}_1\times\mathcal{Z}_2\cdots\mathcal{Z}_H$, we have for any $h\in[H]$, such that for any $f_{h+1}\in\cF_{h+1}$, we have for any $\pi\in\Pi^\dag$, $\mathcal{T}_{h,\mpi}^\dag f_{h+1}\in\cF_h$.
\end{assumption}
\textbf{Instances:} for linear MDP with $\PP_h(s_{h+1}|s_h,a_h)=\innerproduct{\phi(s_h,a_h)}{\mu_h(s_{h+1})}$ and suppose the reward is dicretized into a uniform grid of $M$ points $\sets{z_i}_{i=1}^M$. Then we can write the reward distribution as
    $\RR(z_i|s_h,a_h)=\innerproduct{\mathbf{1_M}}{\theta(z_i)}$, where $\mathbf{1_M}$ is a $M$ dimensional vector with all the entries being $1$, and $\theta(z_i)$ is a $M$ dimensional vector with all $0$ but the $i$ th entry equal to $\RR(z_i|s_h,a_h)$. Then we have:
\begin{align*}
&\mathcal{T}_{h,\mpi}^\dag f_{h+1}(z|s_h^\dag,a_h)\\
=&\sum_{s_{h+1}}\PP(s_{h+1}|s_h,a_h)\sum_{z_i} \RR(z_i|s_h,a_h)\sum_{s}\pi_{h+1}\bracket{a_{h+1}|(s_{h+1},y_h+z_i)} f_{h+1}\bracket{z-z_i|(s_{h+1},y_h+z_i),a_{h+1}}\\
=& \sum_{s_{h+1}} \sum_{z_i}\phi(s_h,a_h)^\top \mu_h(s_{h+1})\theta_h(z_i)^\top \mathbf{1_M}\sum_{s}\pi_{h+1}\bracket{a_{h+1}|(s_{h+1},y_h+z_i)} f_{h+1}\bracket{z-z_i|(s_{h+1},y_h+z_i),a_{h+1}}\\
=&\phi(s_h,a_h)^\top W_h^\mpi(z,y_h)\mathbf{1_M}\,,
\end{align*}
where 
\begin{align*}
    W_h^\mpi(z,y_h)=\sum_{s_{h+1}} \sum_{z_i} \mu_h(s_{h+1})\theta_h(z_i)^\top \sum_{a}\pi_{h+1}\bracket{a_{h+1}|(s_{h+1},y_h+z_i)} f_{h+1}\bracket{z-z_i|(s_{h+1},y_h+z_i),a_{h+1}}
\end{align*}
depends only on $y_h$ and $z$.
We can also write the distribution function in a linear form, with $f_h^\mpi(z|s_h^\dag,a_h)=\Bar{\phi}(s_h,a_h)^\top w_h^\mpi(z,y_h)$, with $\Bar{\phi}(s_h,a_h)=\phi(s_h,a_h)\otimes\mathbf{1_M}\in\RR^{d\times M}$, and $w_h^\mpi(z,y_h)$ is the flattened version of $W_h^\mpi(z,y_h)$ with $w_h^\mpi(z,y_h)[d\times i+m]=W_h^\mpi(z,y_h)[i,m]$. Thus the function class has a linear structure similar to the case of risk neutral setting in linear MDPs \cite{jin2020provably}.

\subsection{Estimation of confidence set}

Here, we estimate the confidence set $\mathcal{Z}_{k,\mpi}$ via MLE using the density functions $f$.
\begin{algorithm}[htbp]
   \caption{\texttt{V-Est-MLE}$(\gH_{k}, \gZ, \mpi,\gamma^{\MLE})$}
   \label{alg:mfestmle}
\begin{algorithmic}
   \STATE {\bfseries Input:} History information $\gH_k$, density function class $\bm\gF$ of random variable class $\bm\gZ$, and policy $\mpi$.
   \STATE For all $i\in[k]$, and $(h,\bm{f},\mpi)\in[H]\times\cF\times\Pi^\dag$, sample 
   $z_{i,h+1}^{\underline{{f}},\underline{\mpi}}\sim \underline{f}_{h+1}\bracket{s_{i,h+1}^\dag,\underline{\pi}_{i,h+1}(s_{i,h+1}^\dag)}$, and let $z_{i,h}^{\underline{{f}},\underline{\mpi}}=z_{i,h+1}^{\underline{{f}},\underline{\mpi}}+r_{i,h}$. We define the version space as:
   \begin{align*}
       \widehat{\cF}_{k+1,{\mpi}}^{\MLE}=\sets{\bm{f}\in\cF:
    \sum_{i=1}^k \log f_h(z_{i,h}^{\underline{{f}},\underline{\mpi}}|s_{i,h}^\dag,a_{i,h})\geq \max_{f'\in\cF}\sum_{i=1}^k \log f'_h(z_{i,h}^{\underline{{f}},\underline{\mpi}}|s_{i,h}^\dag,a_{i,h})-\gamma^{\MLE},\forall h\in[H]}\,.
   \end{align*}
   \STATE Return $\widehat{\gZ}_{k+1,\mpi}$ as the set of random variables corresponding to $\widehat{\cF}_{k+1,\mpi}^{\MLE}$.

\end{algorithmic}
\end{algorithm}

We now describe the sampling procedure for our target function $g_{h+1}\in\cF$ and $\mpi\in\Pi$. Define $\underline{g}_{h+1}$ and $\underline{\mpi}$ as the normalized lower bracket in Section \ref{sec:mfmlenotation}.
For $1\leq i\leq k$,  
we sample $z_{i,h+1}^{\underline{{g}},\underline{\mpi}}\sim \underline{g}_{h+1}(s_{i,h+1}^\dag,\underline{\pi}_{h+1}(s_{i,h+1}^\dag))$, then we construct a one-step-back sample as $z_{i,h}^{\underline{{g}},\underline{\mpi}}=z_{i,h+1}^{\underline{{g}},\underline{\mpi}}+r_{i,h}$ where $r_{i,h}=y_{i,h+1}-y_{i,h}$. We estimate the likelihood of $f_h$ for $\mathcal{T}_{h,\mpi}^\dag g_{h+1}$ as $\log f_h(z_{i,h}^{\underline{{g}},\underline{\mpi}}|s_{i,h}^\dag,a_{i,h}) $. Then we can define the MLE confidence set as:
\begin{align*}
    \cF_{h,k,g,{\mpi}}^{\MLE}=\sets{f_h\in\cF_h:
    \sum_{i=1}^k \log f_h(z_{i,h}^{\underline{{g}},\underline{\mpi}}|s_{i,h}^\dag,a_{i,h})\geq \max_{f'\in\cF}\sum_{i=1}^k \log f'_h(z_{i,h}^{\underline{{g}},\underline{\mpi}}|s_{i,h}^\dag,a_{i,h})-\gamma^{\MLE}}\,,
\end{align*} where $\gamma^{\MLE}=\log\bracket{\cN_{[\cdot]}\bracket{\cF,\epsilon,\norm{}_{\infty}}\cN_{[\cdot]}(\Pi^\dag,\epsilon,\norm{}_{1})/\delta}+K\epsilon$.
Then, we can show that w.h.p, we have 
\begin{align*}
    \mathcal{T}_{h,\mpi}^\dag g_{h+1}\in \cF_{h,k,g,\mpi}^{\MLE}
\end{align*}
We can define our version space as:
\begin{align*}
    \widehat{\cF}^{\MLE}_{k,\mpi}=\sets{f\in\cF: f_h\in\cF_{h,k,f,\mpi}^{\MLE}, h\in[H]}\,.
\end{align*}
Thus, we can prove that $f^\mpi\in\widehat{\cF}^{\MLE}_{k,\mpi}$, and we have that $Z^\mpi\in\widehat{\gZ}_{k,\mpi}$

\subsubsection{Proof of Condition \ref{lem:model free framework concentration}}
In this section we prove the concentration Condition \ref{con:mfconcentration} (Condition \ref{lem:model free framework concentration} in the Appendix)

Following standard MLE concentration analysis, we have:

\begin{lemma}\label{lem:model free mle confidence}
    For any $f_h\in\cF_h\,,~h\in[H]$, there exists a constant $c$ such that: 
    \begin{align*}
         \sum_{i=1}^k \log f_h(z_{i,h}^{\underline{{g}},\underline{\mpi}}|s_{i,h}^\dag,a_{i,h})\leq \sum_{i=1}^k \log (\cT_{h,{\mpi}}^{\dag}g_{h+1})(z_{i,h}^{\underline{{g}},\underline{\mpi}}|s_{i,h}^\dag,a_{i,h})+c\gamma^{\MLE}\,.
    \end{align*}
\end{lemma}
\begin{proof}
Consider an upper and lower $\epsilon$-bracketing of $\cF$ under $\norm{\cdot}_1$, denoted as ${\cF}^{\uparrow}$ and ${\cF}^{\downarrow}$. We denote the corresponding lower bracket of $g_{h+1}$ as $g^{\downarrow}$, and the upper bracket of $f$ as $f^{\uparrow}$. Since $g^{\downarrow}$ may not be a valid distribution, we denote the normalized version as: $\underline{g}$, where $\underline{g}=g^{\downarrow}/\int_z g^{\downarrow}(z)$, and $1-\epsilon\leq\int_z g^{\downarrow}(z)\leq 1$

Then, we have:
\begin{align*}
    &\EE_{\mu_{h}^{\mpi^i}}\mbracket{\frac{f^{\uparrow}_h(z_{i,h}^{\underline{{g}},\underline{\mpi}}|s_{i,h}^\dag,a_{i,h})}{\cT_{h,\mpi}^{\dag}g^{\downarrow}_{h+1}(z_{i,h}^{\underline{{g}},\underline{\mpi}}|s_{i,h}^\dag,a_{i,h})}}\\
    =&\int_{s_h^\dag,a_h,s_{h+1}^\dag,z} \mu_{h}^{\mpi^i}(s_{h}^\dag,a_{h},s_{h+1}^\dag) \frac{\int_{\cA}\underline{\pi}_{h+1}(a_{h+1}|s_{h+1}^\dag) \underline{g}_{h+1}\bracket{z-r_h|s_{h+1}^\dag,a_{h+1}}f^{\uparrow}_h(z|s_h^\dag,a_h)}{\int \TT(s_{h+1}^\dag|s_h^\dag,a_h)\int_{\cA}\pi_{h+1}(a_{h+1}|s_{h+1}^\dag)g^{\downarrow}_{h+1}\bracket{z-r_h|s_{h+1}^\dag,a_{h+1}}}\\
    =& \int_{s_h^\dag,a_h}\mu_{h}^{\mpi^i}(s_h^\dag,a_h)\int_z f^{\uparrow}_h(z|s_h^\dag,a_h)\int_{s_{h+1}^\dag}\frac{\TT(s_{h+1}^\dag|s_h^\dag,a_h)\int_{\cA}\underline{\pi}_{h+1}(a_{h+1}|s_{h+1}^\dag) \underline{g}_{h+1}\bracket{z-r_h|s_{h+1}^\dag,a_{h+1}}}{\int_{s_{h+1}^\dag}\TT(s_{h+1}^\dag|s_h^\dag,a_h)\int_{\cA}\pi_{h+1}(a_{h+1}|s_{h+1}^\dag)g^{\downarrow}_{h+1}\bracket{z-r_h|s_{h+1}^\dag,a_{h+1}}}\\
    \leq& \int_{s_h^\dag,a_h}\mu_{h}^{\mpi^i}(s_h^\dag,a_h)\int_z f^{\uparrow}_h(z|s_h^\dag,a_h)\int_{s_{h+1}^\dag}\frac{\TT(s_{h+1}^\dag|s_h^\dag,a_h)\int_{\cA}(1+2\epsilon){\pi}^{\downarrow}_{h+1}(a_{h+1}|s_{h+1}^\dag) \underline{g}_{h+1}\bracket{z-r_h|s_{h+1}^\dag,a_{h+1}}}{\int_{s_{h+1}^\dag}\TT(s_{h+1}^\dag|s_h^\dag,a_h)\int_{\cA}\pi_{h+1}(a_{h+1}|s_{h+1}^\dag)g^{\downarrow}_{h+1}\bracket{z-r_h|s_{h+1}^\dag,a_{h+1}}}\\
    \leq&\int_{s_h^\dag,a_h}\mu_{h}^{\mpi^i}(s_h^\dag,a_h)\int_z f^{\uparrow}_h(z|s_h^\dag,a_h) (1+2\epsilon)^2\\
    \leq&1+\frac{6}{K}\,.
\end{align*}
Thus we obtain the result via Markov inequality:
\begin{align*}
    & \PP\bracket{\sum_{i=1}^k f^{\uparrow}_h(z_{i,h}^{\underline{{g}},\underline{\mpi}}|s_{i,h}^\dag,a_{i,h})-\log (\cT_{h,\mpi}^{\dag}g^{\downarrow}_{h+1})(z_{i,h}^{\underline{{g}},\underline{\mpi}}|s_{i,h}^\dag,a_{i,h})\geq \log\bracket{1/\delta}}\\
    \leq & \EE\mbracket{\exp\bracket{\sum_{i=1}^k \log\frac{f^{\uparrow}(z_{i,h}^{\underline{{g}},\underline{\mpi}}|s_{i,h}^\dag,a_{i,h})}{\cT_{h,\mpi}^{\dag}g^{\downarrow}(z_{i,h}^{\underline{{g}},\underline{\mpi}}|s_{i,h}^\dag,a_{i,h})}} }\exp(-\log (1/\delta))\\
    \leq& e^6\delta\,.
\end{align*}
Applying a union bound, for all $f^{\uparrow}$ and $g^{\downarrow}$, we have w.p. $1-\delta$, there exists a constant $c$,
\begin{align*}
    \sum_{i=1}^k \log f^{\uparrow}_h(z_{i,h}^{\underline{{g}},\underline{\mpi}}|s_{i,h}^\dag,a_{i,h})-\log (\cT_{h,{\mpi}}^{\dag}g^{\downarrow}_{h+1})(z_{i,h}^{\underline{{g}},\underline{\mpi}}|s_{i,h}^\dag,a_{i,h})\leq c\beta\,.
\end{align*}
We conclude our result by the definition of upper and lower brackets:
\begin{align*}
    &\sum_{i=1}^k \log f_h(z_{i,h}^{\underline{{g}},\underline{\mpi}}|s_{i,h}^\dag,a_{i,h})-\log (\cT_{h,\mpi}^{\dag}g_{h+1})(z_{i,h}^{\underline{{g}},\underline{\mpi}}|s_{i,h}^\dag,a_{i,h})\\
    \leq&  \sum_{i=1}^k \log f^{\uparrow}_h(z_{i,h}^{\underline{{g}},\underline{\mpi}}|s_{i,h}^\dag,a_{i,h})-\log (\cT_{h,\mpi}^{\dag}g^{\downarrow}_{h+1})(z_{i,h}^{\underline{{g}},\underline{\mpi}}|s_{i,h}^\dag,a_{i,h})\\
    \leq& c\beta\,.
\end{align*}
As a result, we have:
\begin{align*}
    \cT_{h,\mpi}^\dag f_{h+1}\in\cF_{h,k,f,\mpi}^{\MLE}
\end{align*}
by the definition of $\cF_{h,k,f,\mpi}^{\MLE}$
\end{proof}

We prove Condition \ref{lem:model free framework concentration} via induction. If $f_{h+1}^\mpi\in \cF_{h+1,k,f^\mpi,\mpi}^{\MLE}$, From Lemma \ref{lem:model free mle confidence}, $f_h^\mpi=\cT_{h,\mpi}^\dag f_{h+1}^\mpi\in \cF_{h,k,f^\mpi,\mpi}^{\MLE}$. Since $f_{H}^\mpi\in \cF_{H,k,f^\mpi,\mpi}^{\MLE}$ we have for every $h$, $f_h^\mpi\in\cF_{h,k,f^\mpi,\mpi} $. Thus $\bm{f}^\mpi\in\widehat{\cF}^{\MLE}_{k,\mpi}$. As a result, $Z^\mpi\in\widehat{\gZ}_{k,\mpi}$ for any $k\in[K]$ and $\mpi\in\Pi^\dag$.

\subsubsection{Proof of Condition \ref{lem:model free pigeon hole}}
In this section we prove the elliptical potential Condition \ref{con:mfelliptical} (Condition \ref{lem:model free pigeon hole} in the Appendix).

The following lemma is the standard result for MLE generalization bound.
\begin{lemma}[MLE concentration]\label{lem:model free mle concentration}
    
    We can bound the square TV distance of the bellman error for any $\bm{f}\in\cF_{k,\mpi}$
    \begin{align*}
        \sum_{i=1}^k\EE_{s_h^\dag,a_h\sim\mu^{\mpi^i}_h}\mbracket{\TV^2\bracket{f_h(s_h^\dag,a_h)||\mathcal{T}_{h,\mpi}^\dag f_{h+1}(s_h^\dag,a_h)}}\leq \mathcal{O}(\gamma^{\MLE})\,.
    \end{align*}
\end{lemma}
\begin{proof}
  Since we have for any $\bm{f}\in\cF_{k,\mpi}$
    \begin{align*}
      &  \sum_{i=1}^k \log \bracket{\cT_{h,\underline{\mpi}}^\dag \underline{f}_{h+1}}(z_{i,h+1}^{\underline{{f}},\underline{\mpi}}|s_{i,h}^\dag,a_{i,h})-\log f_h(z_{i,h+1}^{\underline{{f}},\underline{\mpi}}|s_{i,h}^\dag,a_{i,h})\\
        \leq&\sum_{i=1}^k \log \bracket{\cT_{h,{\mpi}^{\downarrow}}^\dag {f}^\downarrow_{h+1}}(z_{i,h+1}^{\underline{{f}},\underline{\mpi}}|s_{i,h}^\dag,a_{i,h})-2\log(1-\epsilon)-\log f_h(z_{i,h+1}^{\underline{{f}},\underline{\mpi}}|s_{i,h}^\dag,a_{i,h})
    \end{align*}
    which holds by the normalization constant of $\underline{f}$ and $\underline{\pi}$. By the definition of the lower bracket function $f^\downarrow_{h+1}$, we have
    \begin{align*}
        &\sum_{i=1}^k \log \bracket{\cT_{h,{\mpi}^{\downarrow}}^\dag {f}^\downarrow_{h+1}}(z_{i,h+1}^{\underline{{f}},\underline{\mpi}}|s_{i,h}^\dag,a_{i,h})-2\log(1-\epsilon)-\log f_h(z_{i,h+1}^{\underline{{f}},\underline{\mpi}}|s_{i,h}^\dag,a_{i,h}) \\
        \leq& \sum_{i=1}^k \log \bracket{\cT_{h,{\mpi}}^\dag {f}_{h+1}}(z_{i,h+1}^{\underline{{f}},\underline{\mpi}}|s_{i,h}^\dag,a_{i,h})-2\log(1-\epsilon)-\log f_h(z_{i,h+1}^{\underline{{f}},\underline{\mpi}}|s_{i,h}^\dag,a_{i,h})\\
        \leq& \max_{f'\in\cF}\sum_{i=1}^k \log f'_h(z_{i,h}^{\underline{{f}},\underline{\mpi}}|s_{i,h}^\dag,a_{i,h})-\log f_h(z_{i,h+1}^{\underline{{f}},\underline{\mpi}}|s_{i,h}^\dag,a_{i,h})+\mathcal{O}(k\epsilon)\\
        \leq& \gamma^{\MLE}+\mathcal{O}(k\epsilon)\\
        =&\mathcal{O}(\gamma^{\MLE})\,,
    \end{align*}
    where the second inequality is due to the distributional bellman completeness, which ensures that $\cT_{h,{\mpi}}^\dag {f}_{h+1}\in\cF_h$. The last two inequalities are due to the construction of the confidence set and the choice of $\epsilon=1/K$.
    
    Since the conditional distribution of $z_{i,h+1}^{\underline{{f}},\underline{\mpi}}$ given $s_{i,h}^\dag,a_{i,h}$ is the same as $\mathcal{T}_{h,\underline{\mpi}}^\dag \underline{f}_{h+1}(s_h^\dag,a_h)$, from Lemma \ref{lem:mle generalization bound}, we have:
    \begin{align*}
\sum_{i=1}^k\EE_{s_h^\dag,a_h\sim\mu^{\mpi^i}_h}\mbracket{\TV^2\bracket{f_h(s_h^\dag,a_h)||\mathcal{T}_{h,\underline{\mpi}}^\dag \underline{f}_{h+1}(s_h^\dag,a_h)}}\leq\mathcal{O}(\gamma^{\MLE})\,.
    \end{align*}
    We also have:
    \begin{align*}
        &\TV\bracket{\mathcal{T}_{h,{\mpi}}^\dag f_{h+1}(s_h^\dag,a_h)||\mathcal{T}_{h,\underline{\mpi}}^\dag \underline{f}_{h+1}(s_h^\dag,a_h)}\\
        \leq &\TV\bracket{\mathcal{T}_{h,{\mpi}}^\dag f_{h+1}(s_h^\dag,a_h)||\cT_{h,\underline{\mpi}}^\dag f_{h+1}(s_h^\dag,a_h)}+\TV\bracket{\mathcal{T}_{h,\underline{\mpi}}^\dag \underline{f}_{h+1}(s_h^\dag,a_h)||\cT_{h,\underline{\mpi}}^\dag f_{h+1}(s_h^\dag,a_h)}\,.
    \end{align*}
    We can bound the first term as:
    \begin{align*}
        &\TV\bracket{\mathcal{T}_{h,{\mpi}}^\dag f_{h+1}(s_h^\dag,a_h)||\cT_{h,\underline{\mpi}}^\dag f_{h+1}(s_h^\dag,a_h)}\\
        \leq&\int_{s_{h+1}^\dag}\TT(s_{h+1}^\dag|s_h^\dag,a_h)\int_{a_{h+1}}\abs{\underline{\pi}_{h+1}(a_{h+1}|s_{h+1}^\dag)-\pi_{h+1}(a_{h+1}|s_{h+1}^\dag)}\int_z(f_{h+1}(z-r_h|s_{h+1}^\dag,a_{h+1})\\
        \leq&\max_{s_{h+1}^\dag}\norm{\underline{\pi}_{h+1}(\cdot|s_{h+1}^\dag)-\pi_{h+1}(\cdot|s_{h+1}^\dag)}_1\\
        \leq&\epsilon\,.
    \end{align*}
     We can also bound the second term as:
    \begin{align*}
        &\TV\bracket{\mathcal{T}_{h,\underline{\mpi}}^\dag \underline{f}_{h+1}(s_h^\dag,a_h)||\cT_{h,\underline{\mpi}^\dag} f_{h+1}(s_h^\dag,a_h)}\\
\leq&\int_{s_{h+1}^\dag}\TT(s_{h+1}^\dag|s_h^\dag,a_h)\int_{a_{h+1}}\underline{\pi}_{h+1}(a_{h+1}|s_{h+1}^\dag)\int_z \abs{\bracket{f_{h+1}-\underline{f}_{h+1}}(z-r_h|s_{h+1}^\dag,a_{h+1})}\\
\leq & \max_{s_{h+1}^\dag,a_{h+1}}\norm{f_{h+1}(s_{h+1}^\dag,a_{h+1})-\underline{f}_{h+1}(s_{h+1}^\dag,a_{h+1})}_1\\
\leq& \mathcal{O}(\epsilon)\,,
    \end{align*}
since $\underline{f}\leq (1+2\epsilon)f^\downarrow\leq(1+2\epsilon)f$ and $\underline{f}\geq f^\downarrow$ point wise. Then we can have $\norm{\underline{f}-f}_1=\int_z \abs{f(z)-\underline{f}(z)}\leq \int_z \max\sets{2\epsilon f, f-f^\downarrow}\leq \int_z 2\epsilon f+\int_z(f-f^\downarrow)\leq 3\epsilon$.

    We can conclude that: $\sum_{i=1}^k\EE_{s_h^\dag,a_h\sim\mu^{\mpi^i}_h}\mbracket{\TV^2\bracket{\mathcal{T}_{h,{\mpi}}^\dag f_{h+1}(s_h^\dag,a_h)||\mathcal{T}_{h,\underline{\mpi}}^\dag \underline{f}_{h+1}(s_h^\dag,a_h)}}\leq\mathcal{O}(k\epsilon)$. Thus, 
    \begin{align*}
        \sum_{i=1}^k\EE_{s_h^\dag,a_h\sim\mu^{\mpi^i}_h}\mbracket{\TV^2\bracket{f_h(s_h^\dag,a_h)||\mathcal{T}_{h,\mpi}^\dag f_{h+1}(s_h^\dag,a_h)}}\leq \mathcal{O}(\gamma^{\MLE})\,.
    \end{align*}
\end{proof}

Next, we present the distribution difference lemma for our model free analysis.
\begin{lemma}\label{lem:model free distribution difference}
We can bound the distance between the CDFs of the return by the bellman error of each step as:
\begin{align*}
    \norm{F_{Z}-F_{Z^{\mpi}}}_{\infty}\leq \sum_{h=1}^H \EE_{(s_h^\dag,a_h)\sim\mu^{\mpi}_h}\mbracket{\norm{f_h(s_h^\dag,a_h)-\mathcal{T}_{h,\mpi}^\dag f_{h+1}(s_h^\dag,a_h)}_1}\,.
\end{align*}
    
\end{lemma}

\begin{proof}
    We begin by induction. By triangle inequality:
    \begin{align*}
        \norm{f_h(s_h^\dag,a_h)-f_h^{\mpi}(s_h^\dag,a_h)}_1\leq \norm{f_h(s_h^\dag,a_h)-\mathcal{T}_{h,\mpi}^\dag f_{h+1}(s_h^\dag,a_h)}_1+\norm{\mathcal{T}_{h,\mpi}^\dag f_{h+1}(s_h^\dag,a_h)-\mathcal{T}_{h,\mpi}^\dag f_{h+1}^{\mpi}(s_h^\dag,a_h)}_1\,.
    \end{align*}
    We derive a recursion for the second term.
    \begin{align*}
        &\norm{\mathcal{T}_{h,\mpi}^\dag f_{h+1}(s_h^\dag,a_h)-\mathcal{T}_{h,\mpi}^\dag f_{h+1}^{\mpi}(s_h^\dag,a_h)}_1\\
        =&\int_z\abs{\int_{s_{h+1}^\dag,a_h} \TT(s_{h+1}^\dag|s_h^\dag,a_h)\pi_{h+1}(a_{h+1}|s_{h+1}^\dag)\bracket{f_{h+1}(z-r_h|s_{h+1}^\dag,a_{h+1})-f_{h+1}^{\mpi}(z-r_h|s_{h+1}^\dag,a_{h+1})} }\\
        \leq& \int_{s_{h+1}^\dag,a_h} \TT(s_{h+1}^\dag|s_h^\dag,a_h)\pi_{h+1}(a_{h+1}|s_{h+1}^\dag) \int_z \abs{f_{h+1}(z-r_h|s_{h+1}^\dag,a_{h+1})-f_{h+1}^{\mpi}(z-r_h|s_{h+1}^\dag,a_{h+1})} \\
        =& \EE_{(s_{h+1}^\dag,a_{h+1})\sim(s_h^\dag,a_h)}\norm{f_{h+1}(s_{h+1}^\dag,a_{h+1})-f_{h+1}^\mpi(s_{h+1}^\dag,a_{h+1})}_1\,.
    \end{align*}
Then we have:
\begin{align*}
     & \EE_{(s_h^\dag,a_h)\sim \mu^{\mpi}_h}\mbracket{\norm{f_h(s_h^\dag,a_h)-f_h^{\mpi}(s_h^\dag,a_h)}_1}\\
    \leq& \EE_{(s_h^\dag,a_h)\sim\mu^{\mpi}_h}\mbracket{\norm{f_h(s_h^\dag,a_h)-\mathcal{T}_{h,\mpi}^\dag f_{h+1}(s_h^\dag,a_h)}_1}\\
    &+\int_{s_h^\dag,a_h} \mu^{\mpi}_h(s_h^\dag,a_h)\int_{s_{h+1}^\dag,a_{h+1}} \mu^{\mpi}(s_{h+1}^\dag,a_{h+1}|s_h^\dag,a_h)\norm{f_{h+1}(s_{h+1}^\dag,a_{h+1})-f_{h+1}^\mpi(s_{h+1}^\dag,a_{h+1})}_1\\
    \leq &\EE_{(s_h^\dag,a_h)\sim\mu^{\mpi}_h}\mbracket{\norm{f_h(s_h^\dag,a_h)-\mathcal{T}_{h,\mpi}^\dag f_{h+1}(s_h^\dag,a_h)}_1}+ \EE_{(s_{h+1}^\dag,a_{h+1})\sim\mu^{\mpi}_{h+1}}\mbracket{\norm{f_{h+1}(s_{h+1}^\dag,a_{h+1})-f_{h+1}^\mpi(s_{h+1}^\dag,a_{h+1})}_1}\,.
\end{align*}
Using the definition of $F_Z(x)=\int_{s_1^\dag,a_1}\mu_1^\mpi(s_1^\dag,a_1) \int_{z\leq x}f_1(z|s_1^\dag,a_1)$, we have:
\begin{align*}
    &\norm{F_{Z}-F_{Z^\mpi}}_\infty\\
\leq&\EE_{s_1^\dag,a_1\sim\mu^{\mpi}_1}\mbracket{\norm{f_1\bracket{s_1^\dag,a_1}-f_1^\mpi(s_1^\dag,a_1)}_1}\\
    \leq& \sum_{h=1}^H \EE_{(s_h^\dag,a_h)\sim\mu^{\mpi}_h}\mbracket{\norm{f_h(s_h^\dag,a_h)-\mathcal{T}_{h,\mpi}^\dag f_{h+1}(s_h^\dag,a_h)}_1}\,.
\end{align*}
\end{proof}

Combining the elliptical potential condition for low bellman eluder dimension (Lemma \ref{lem:bellman eluder dimension}) and the concentration result Lemma \ref{lem:model free mle concentration}, we have for all $h\in[H]$ and any $\hat{\bm{f}}^k\in\widehat{\cF}_{k,\mpi^k}^{\MLE}$:
\begin{align*}
    \sum_{k=1}^K \EE_{s_h^\dag,a_h\sim\mu^{\mpi^k}_h}\mbracket{\TV\bracket{\hat{f}^k_h(s_h^\dag,a_h)||\mathcal{T}_{h,\mpi^k}^\dag \hat{f}_{h+1}^k(s_h^\dag,a_h)}}\leq \mathcal{O}(\sqrt{\operatorname{d_{BE}}\gamma^{\MLE} K})\,.
\end{align*}
Here we invoke Lemma \ref{lem:bellman eluder dimension} by setting $\gX:\cS^\dag\times\cA$, $\Phi: \TV\bracket{f_h(s^\dag,a)||\cT_{h,\mpi}^\dag f_{h+1}(s^\dag,a)}$  for all $(f,\mpi)\in\cF\times\Pi^\dag$. $\cD:\mu^{\mpi},\,\mpi\in\Pi^\dag$ is the family of all the visitation measures defined on $(s^\dag,a)$. 

Thus, using Lemma \ref{lem:model free distribution difference}, we have:
\begin{align*}
   \sum_{k=1}^K \norm{F_{\widehat{Z}^k}-F_{Z^{\mpi^k}}}_\infty\leq \mathcal{O}(H\sqrt{\operatorname{d_{BE}}\gamma^{\MLE} K})\,.
\end{align*}

\section{Auxiliary Lemmas}

\begin{lemma}[MLE generalization bound (Theorem 21 of \cite{agarwal2020flambe})]\label{lem:mle generalization bound}
    Let $\gX$ be a feature space and $\gY$ be the output space. Given a dataset $D=\sets{(x_i,y_i)}_{i=1}^n$ which is collected from a martingale process: $x_i\sim\gD_i(x_{1:i-1},y_{1:i-1})$, and $y_i\sim p(\cdot|x_i)$. Given a function set $\gF:\gX\times\gY\rightarrow\RR$, we have the real conditional distribution $f^*(x,y)=p(y|x)\in\gF$. Then, there exists a constant $c$, for any $\delta>0$, with probability at least $1-\delta$, we have:
    \begin{align*}
        \sum_{i=1}^n \EE_{x\sim\cD_i}\mbracket{\TV\bracket{f(x,\cdot)||f^*(x,\cdot)}^2}\leq c\bracket{\sum_{i=1}^n \log\bracket{\frac{f^*(x_i,y_i)}{f(x_i,y_i)}}+\log \bracket{\cN_{[\cdot]}(\epsilon,\cF,\norm{\cdot}_1/)\delta}+n\epsilon}
    \end{align*}
\end{lemma}

\begin{lemma}[Concentration Lemma(Theorem 5 in \cite{ayoub2020model})]\label{lem:model free lsr auxillary concentration}
    let $(X_p,Y_p)_{p=1,2\cdots}$ be a set of random variables, $X_p\in\gX$ for some measurable set $\gX$ and $Y_p\in\RR$. Let $\cF$ be a set of real valued measurable function with domain $\mathcal{X}$. Let $\mathbb{F}=\bracket{\mathbb{F}_p}_{p=0,1,2\cdots}$ be a filtration such that for all $p\geq1$, we have $\bracket{X_1,Y_1,\cdots,X_{p-1},Y_{p-1},X_p}$ is $\mathbb{F}_{p-1}$ measurable, and such that there exists some function $f_*\in\cF$ such that $\EE[Y_p|\mathbb{F}_{p-1}]=f_*(X_p)$ for all $p\geq1$. Let $\hat{f_t}=\argmin_{f\in\cF}\sum_{p=1}^{t}\bracket{f(X_p)-Y_p}^2$. Let $\cN(\cF,\alpha)$ be the $\alpha$-covering number of set $\cF$ under $\norm{\cdot}_{\infty}$ metric at scale $\alpha$. Define $\dist_t(f||f_t)=\sum_{p=1}^t\bracket{f(X_p)-f_t(X_p)}^2$.

    If the functions in $\cF$ are bounded by some constant $C>0$. Assume that for each $p\geq1$, $\bracket{Y_p-f_*(X_p)}$ is conditionally $\sigma$-sub-gaussian given $\mathbb{F}_{p-1}$. Then, for any $\alpha>0$, with probability $1-\delta$ for all $t\geq1$, we have:
    \begin{align*}
        \dist\bracket{f_*||f_t}\leq 8\sigma^2\log\bracket{\cN(\cF,\alpha)/\delta}+4t\alpha\bracket{C+\sqrt{\sigma^2 \log\bracket{t(t+1)/\delta}}}
    \end{align*}
\end{lemma}

\newpage
\section{Linear {CVaR}}

Similarly to \cite{liu2023optimistic,jin2021bellman}, the general algorithms provided for general version space are information theoretic, which means that they cannot be implemented efficiently in general. This is because we consider the general risk measure LRM and the general function approximation settings. However, when specified to the CVaR risk measure under the discretized linear MDP, a distributional extension of natural linear MDP \cite{jin2020provably}, we can design and implement an efficient model-free algorithm that achieves sub-linear regret. 
\begin{definition}[Discretized Linear MDP]\label{def:discretized mdp}
    An augmented MDP $\cM^\dag$ is a discretized linear MDP with feature map $\phi:(s,a)\rightarrow \RR^d$ and an uniform grid of $M$ points $\sets{z_i}_{i=1}^M$, if for any $h\in[H]$, $(r_h,y_h)\in\sets{z_i}_{i=1}^M$, and there exists unknown measures $\mu_h:\cS\rightarrow\RR^d$ and $\theta_h:\sets{z_i}_{i=1}^M\rightarrow\RR^d$, such that:
    \begin{align*}&\PP_h(s_{h+1}|s_h,a_h)=\phi(s_h,a_h)^\top\mu_h(s_{h+1})\\&\RR_h(z_i|s_h,a_h)=\phi(s_h,a_h)^\top\theta_h(z_i)~,
    \end{align*}
    for all $s_h,a_h,s_{h+1},z_i$.
\end{definition}
This discretized linear MDP is the natural extension of the linear MDP assumption in \cite{jin2020provably}, where we consider the discretized distributional reward instead of determined reward. and we generalize the linear expected reward to its distributional counterpart. Another important ingredient in our definition is the discretized reward space, which is commonly used in practice. C51 and Rainbow \cite{bellemare2017distributional,hessel2018rainbow} both set $M=51$ and achieved empirical success in Atari games. We need the discretized reward space mainly to bound the covering number of the value distribution, similar to \cite{wang2023benefits}.

In a discretized linear MDP, we have that the distribution function have a quadratic structure: $f_h^\pi(z_i|(s_h,y_h),a_h)=\phi(s_h,a_h)W_h^\pi(z_i,y_h)\phi(s_h,a_h)$. Thus, we can use linear regression in estimating statistical functionals of $Z_h^\mpi(s_h^\dag,a_h)$. We present our regret bound as follows:
\begin{theorem}
\label{thm:linearcvar}
    If MDP $\cM^\dag$ is a discretized linear MDP satisfying Definition \ref{def:discretized mdp}, and the risk measure $\rho(Z^\mpi)=CVaR_{\tau}(Z^\pi)=\max_b\sets{b-\tau^{-1}\EE[(b-Z^\pi)^+]}$, we can bound the regret as:
    \begin{align*}
        \operatorname{Regret}(K)\leq \mathcal{\Tilde{O}}(\tau^{-1}d^3H^2\sqrt{MK})
    \end{align*}
\end{theorem}
We highlight that this is the first efficient model free algorithm for linear MDPs, and the $\sqrt{M}$ dependency is due to the covering number of the value distribution class which also appears in other model free distributional RL algorithms, such as \cite{wang2023benefits}.

Define $\operatorname{CVaR}_\tau(Z^\pi)=\argmax_{b}\sets{b-\tau^{-1}\EE[(b-Z^\pi)^{+}]}$. We can define the statistical functionals $Q$ and $V$ as: $Q^\pi_h(s_h,y_h,a_h)=\sum_{z_i}f_h^\pi(z_i|s_h,y_h,a_h)[(-z_i-y_h)^{+}]$, and $V^\pi_h(s_h,y_h)=\sum_{z_i}f_h^\pi(z_i|s_h,y_h,\pi(s_h,y_h))[(-z_i-y_h)^{+}]$. Then we can write the {CVaR} objective as $\operatorname{CVaR}_\tau(Z^\pi)=\argmax_{b\in[0,H]}\sets{b-\tau^{-1}V_1^{\pi}(s_1,-b)}$. 

Define recursively $\pi^*_h(s_h,y_h)=\argmin_a Q_h^{\pi^*}(s_h,y_h,a_h)$, then by \cite{wang2023nearminimaxoptimal} we have for any $b\in[0,H]$,
\begin{align*}
    V_1^*(s_1,-b)=V_1^{\pi^*}(s_1,-b)=\argmin_{\pi\in\Pi^\dag}V_1^\pi(s_1,-b)
\end{align*}

We denote $\operatorname{CVaR}_{\tau}^*=\max_{\pi\in\Pi^\dag}\operatorname{CVaR}_{\tau}^\pi=b^*-\tau^{-1}V_1^*(s_1,-b^*)$

\subsection{Linear Augmented MDPs}
When consider the linear function approximation (Definition~\ref{def:discretized mdp}), we can also linearize the augmented MDPs. We ahve $$\TT_h(s_{h+1},r_h|s_h,a_h)=\phi(s_h,a_h)^\top\theta(r_h)\mu(s_{h+1})^\top\phi(s_h,a_h)=\psi(s_h,a_h)\chi(r_h,s_{h+1})$$
where $\psi(s_h,a_h)$ and $\chi(r_h,s_{h+1})$ are the flattened versions of $\phi(s_h,a_h)\phi(s_h,a_h)^\top$ and $\theta(r_h)\mu\bracket{s_{h+1}}^\top$. Also, we assume the reward space is discretized into $M$ points $z_1\cdots z_M$ such that for all $h\in[H]$, $y_h\in \sets{z_i}_{i=1}^M$. We have that $-H\leq z_1\leq z_2\cdots z_M\leq H$.
We highlight that this discretization is standard in practice as in C51 \cite{bellemare2017distributional} and Rainbow \cite{hessel2018rainbow}. Here we need this assumption to bound the complexity of the function class. 

\subsection{Linear Completeness}
Because the density function of the rewards satisfy the distributional bellman equation:
\begin{align*}
    f_h^{\pi}(z_i|s_h,y_h,a_h)=\sum_{s_{h+1},r_h}\TT_h(s_{h+1},r_h|s_h,a_h)f_{h+1}^\pi(z_i-r_h|s_{h+1},y_h+r_h,\pi(s_{h+1},y_h+r_h))
\end{align*}
We have the statistical functionals satisfy the augmented bellman equation:
\begin{align*}
    Q_h^\pi(s_h,y_h,a_h)=&\sum_{z_i}\sum_{s_{h+1},r_h}\TT_h(s_{h+1},r_h|s_h,a_h)f_{h+1}^\pi(z_i-r_h|s_{h+1},y_h+r_h,\pi(s_{h+1},y_h+r_h))(-z_i-y_h)^{+}\\
    =&\sum_{s_{h+1},r_h}\TT_h(s_{h+1},r_h|s_h,a_h)\sum_{z_i}f_{h+1}^{\pi}(z_i|s_{h+1},y_h+r_h,\pi(s_{h+1},y_h+r_h))(-z_i-y_h-r_h)^{+}\\
    =&\bracket{\TT_h V_{h+1}}(s_h,y_h,a_h)
\end{align*}
Where we denote the augmented transition operating on a function $V$ as:
\begin{align*}
    \TT_hV_{h+1}(s_h,y_h,a_h)=\sum_{s_{h+1},r_h}\TT_h(s_{h+1},r_h|s_h,a_h)V_{h+1}(s_{h+1},y_h+r_h)
\end{align*}

Since $f_h^\pi(z_i|s_h,y_h,a_h)=\phi(s_h,a_h)^\top W_h^\pi(y_h,z_i)\phi(s_h,a_h)$, we have $Q_h^{\pi}(s_h,y_h,a_h)=\psi(s_h,a_h)^\top w_h^\pi(y_h)$. Where $w_h^\pi[i\cdot d+j]=\sum_{z_i}W_h^\pi(y_h,z_i)[i,j][(-z_i-y_h)^{+}]$

\subsection{Algorithm}
In this section, we present our computationally efficient algorithm \texttt{RSRL-Linear-CVaR}. Notice that we present Upper Confidence Bound Value Iteration (UCV-VI) in this algorithms instead of the general optimistic planning used in previous frameworks, due to its computation-tractable property.
\begin{algorithm}[htbp]
   \caption{\texttt{RSRL-Linear-CVaR}}
\label{alg:linear CVaR}
\begin{algorithmic}
   \STATE {\bfseries Input:} Features $\Psi:\psi(s,a)$ Bonus $b_h^k(s,a)=\beta \norm{\psi(s,a)}_{\Lambda_{k,h}^{-1}}$ with $\beta=c_\beta d^2H\sqrt{M\log(d^2UMK/\delta)}$
   \FOR{$k=1$ {\bfseries to} $K$}
   \STATE Set $V_{H+1}(s_{H+1},y)=[(-y)^+]$ for all $y\in\sets{z_i}_{i=1}^M$
        \FOR{$h=H$ {\bfseries to} $1$}
        \STATE \begin{align*}
           &\Lambda_{k,h}=\sum_{i=1}^{k-1}\psi(s_h^i,a_h^k)\psi(s_h^i,a_h^i)^\top+\lambda \mathbf{I}\\
           &w_h^k(y_h)=\Lambda_{k,h}^{-1}\sum_{i=1}^{k-1}\psi(s_h^i,a_h^i)V_{h+1}^k(s_{h+1}^i,y_h+r_h^i)\\
           &Q_h^k(s_h,y_h,a_h)=\psi(s_h,a_h)w_h^k(y_h)-b_h^k(s_h,a_h)\\
           &\pi^k(s_h,y_h)=\argmin_a Q_h^k(s_h,y_h,a_h)\\
           &V_h^k(s_h,y_h)=\max\sets{Q_h^k(s_h,y_h,\pi^k(s_h,y_h)),0}
        \end{align*}
        \ENDFOR
   \STATE \begin{equation*}
       b_k=\argmax_{b\in\sets{z_i}_{i=1}^{M}}\sets{b-\tau^{-1}V_1^k(s_1,-b)}
   \end{equation*}
   \STATE
   Start at state $s_1^\dag=(s_1,-\lambda^k)$ and execute policy $\pi^k$ in the augmented MDP $\cM^\dag$, collect information $\sets{(s_h^k,a_h^k,r_h^k)}_{h\in[H]}$
\ENDFOR
\end{algorithmic}
\end{algorithm}

\subsection{Concentration and Covering}
In this section we provide the concentration and covering arguments needed for our linear analysis.
\begin{lemma}[Concentration Inequality of Self-normalized Process \cite{jin2020provably}]\label{lem:linear selfnormalized }
    Let $\sets{x_\tau}_{\tau=1}^{\infty}$ be a stochastic process on domain $\gX$ with corresponding filtration $\sets{\cF_\tau}$. Let $\sets{\psi_\tau}$ be an $\RR^{d^2}$ valued stochastic process  stochastic process such that $\psi_\tau\in\cF_{\tau-1}$ and $\norm{\psi_\tau}\leq 1$. Let $\Lambda_k=\sum_{\tau=1}^{k-1}\psi_\tau\psi_\tau^\top$. Then, for any $\delta>0$, with probability at least $1-\delta$, for all $k>0$, and any $V\in\cV$ such that $\abs{\sup_x V(x)-\inf_x V(x)}\leq H$, we have:
    \begin{align*}
        \norm{\sum_{\tau=1}^{k-1}\psi_\tau\sets{V(x_\tau)-\EE[V(x_\tau|\cF_{\tau-1}]}}_{\Lambda_k^{-1}}^2\leq H^2 \mbracket{\frac{d}{2}\log\bracket{\frac{k+\lambda}{\lambda}}+\log\bracket{\cN_C(\epsilon,\cV,\norm{\cdot}_{\infty}}/\delta}+\frac{8k^2\epsilon^2}{\lambda}~.
    \end{align*}
\end{lemma}
\begin{lemma}[Lemma D.6 in \cite{jin2020provably}]\label{lem:linear convering number}

    Let $\gV$ denote a class of functions mapping from domain $\gX$ to $\RR$ with the following parametric form: 
    \begin{align*}
        V(\cdot)=w^\top\phi(\cdot)+\beta \norm{\phi(\cdot)}_{\Lambda^{-1}}
    \end{align*}
    where $\phi\in\RR^d$ are features on domain $\gX$.
    The parameters satisfy $\norm{w}\leq L$, $ \lambda_{\min}(\Lambda)\geq \lambda$, $\beta\in[0,B]$, and $\norm{\phi(\cdot)}\leq1$. Then, the log covering number can be bounded as:
    \begin{align*}
        \log\bracket{(\cN_C(\epsilon,\gV,\norm{\cdot}_\infty)}\leq d \log (1+4L/\epsilon)+d^2\log(1+8d^{1/2}B^2/\lambda\epsilon^2)
    \end{align*}
\end{lemma}

\begin{lemma}[Lemma B.2 in \cite{jin2020provably}]\label{lem:linear weight algo}
    For $h\in[H]$ and $k\in[K]$, if $V_{h+1}^k(s_{h+1},y_{h+1})\leq H$, then $\norm{w_h^k}_2\leq Hd\sqrt{k/\lambda}$
\end{lemma}
\begin{lemma}[Lemma B.1 in \cite{jin2020provably}]\label{lem:linear weight}
    For any $h\in[H]$, $\norm{w_h^\pi}\leq 2Hd$
\end{lemma}
\begin{lemma}\label{lem:linear martingale}
    If $V_{h+1}^k\leq H$ for any $k\in[K]$, then there exists a constant $C$, for any $y_h\in\sets{z_i}_{i=1}^M$, $\delta>0$, with probability $1-\delta$, we can bound the self normalized martingale process as:
    \begin{equation*}
        \norm{\sum_{i=1}^{k-1}\psi_{i,h}\bracket{V_{h}^k(s_{h}^i,y_{h-1}+r_{h-1}^i)-\TT_{h-1} V_{h}^k(s_{h-1}^i,y_{h-1},a_{h-1}^i)}}_{\Lambda_{k,h}^{-1}}^2\leq C d^4 H^2M{\log(c_\beta d^2 KM/\delta)}
    \end{equation*}
   
\end{lemma}

\begin{proof}
    According to Lemma \ref{lem:linear weight algo}, we have that $\norm{w_h^k}_2\leq Hd\sqrt{k/\lambda}$. According to Lemma \ref{lem:linear convering number}, we have for a fixed $z_i$ the covering number of the function class $V_h^k(s_h,z_i)$ can be bounded as (notice that $\psi\in\RR^{d^2}$):
\begin{align*}
    \log\bracket{(\cN_C(\epsilon,\gV(\cdot,z_i),\norm{\cdot}_\infty)}\leq d^2 \log (1+4Hd\lambda/k\epsilon)+d^4\log(1+8d\beta^2/\lambda\epsilon^2)
\end{align*}
Then we can bound the entire function class by regarding each $V(\cdot,z_i)$ an individual function. Thus, the total covering number can be bounded as:
\begin{align*}
    \log\bracket{(\cN_C(\epsilon,\gV,\norm{\cdot}_\infty)}\leq  Md^2 \log (1+4Hd\lambda/k\epsilon)+d^4\log(1+8d\beta^2/\lambda\epsilon^2)
\end{align*}
Thus, we can apply Lemma \ref{lem:linear selfnormalized }, and we have the result by choosing $\epsilon=1/k$, $\lambda=1$ and $\beta=c_\beta d^4HM\sqrt{\log(d^2UMK/\delta)}$.
\end{proof}

\subsection{Proof of Theorem~\ref{thm:linearcvar}}

Define the event
\ref{eq:linear optimism at h} at $h\in[H]$ and $k\in[K]$ as:
\begin{equation}\label{eq:linear optimism at h}\tag{Optimism}
   Q_{h}^{k}(s_{h},y_{h},a_{h})\leq Q_{h}^{*}(s_{h},y_{h},a_{h})
\end{equation}
Also, define 
\begin{align*}
        \psi(s,a)w_h^k(y)-Q_h^\pi(s,y,a)-\TT_h (V_{h+1}^k-V_{h+1}^\pi)(s,a,y)=\Delta_h^k(s,a) 
\end{align*}

Then we have the following concentration result:
\begin{lemma}\label{lem:linear concentration}
    If \ref{eq:linear optimism at h} holds at $h\in[H]$ and $k\in[K]$. We have that $\Delta_{h}^k(s,a)\leq\beta \norm{\psi(s,a)}_{\Lambda_{k,h}^{-1}}$.
\end{lemma}
\begin{proof}
    Since we have:
\begin{align*}
    \psi(s_h^k,a_h^k)\TT_h V_{h+1}^\pi(s_h^k,a_h^k,y_h)=& \sum_{s_{h+1},r_h}\psi(s_h^k,a_h^k)\psi(s_h^k,a_h^k)^\top\chi(s_{h+1},r_h)V_{h+1}^\pi(s_{h+1},y_h+r_h)=\psi(s_h^k,a_h^k)\psi(s_h^k,a_h^k)^\top w_h^\pi(y_h)
\end{align*}
We bound the bias as:
\begin{align*}
    w_h^k(y_h)-w_h^\pi(y_h)=&\Lambda_{k,h}^{-1}\sets{-\lambda w_h^\pi+\mbracket{\sum_{i=1}^{k-1}\psi_{i,h}\bracket{V_{h+1}^k(s_{h+1}^i,y_h+r_h^i)-\TT_h V_{h+1}^\pi(s_h^i,a_h^i,y_h)}}}\\
    =&\underbrace{-\lambda \Lambda_{k,h}^{-1}w_h^{\pi}}_{q_1}+\underbrace{\Lambda_{k,h}^{-1}\sum_{i=1}^k\psi_{i,h}\bracket{V_{h+1}^k(s_{h+1}^i,y_h+r_h^i)-\TT_h V_{h+1}^k(s_h^i,a_h^i,y_h)}}_{q_2}\\
    &+\underbrace{\Lambda_{k,h}^{-1}\sum_{i=1}^k \psi_{i,h}\TT_h(V_{h+1}^k-V_{h+1}^\pi)(s_h^i,a_h^i,y_h)}_{q_3}
\end{align*}
We have according to Lemma \ref{lem:linear weight}:
\begin{align*}
    \abs{\psi(s_h,a_h)^\top q_1}\leq \sqrt{\lambda}\norm{w_h^\pi}\norm{\psi(s_h,a_h)}_{\Lambda_{k,h}^{-1}}\leq 2\sqrt{\lambda}Hd\norm{\psi(s_h,a_h)}_{\Lambda_{k,h}^{-1}}
\end{align*}
For the second term, on event \ref{eq:linear optimism at h}, we have that $V_{h+1}^k\leq V_{h+1}^*\leq H$. Apply Lemma \ref{lem:linear martingale} we have:
\begin{align*}
    \abs{\psi(s_h,a_h)^{\top}q_2}\leq  C d^2 H\sqrt{M\log(c_\beta d^2 KM/\delta)}\norm{\psi(s_h,a_h)}_{\Lambda_{k,h}^{-1}}
\end{align*}
For the third term, we have:
\begin{align*}
    {\psi(s_h,a_h)^\top q_3}=&{\psi(s_h,a_h)\Lambda_{k,h}^{-1}\sum_{i=1}^k \psi_{i,h}\psi_{i,h}^\top \sum_{s_{h+1},r_h}\chi(s_{h+1},r_h)(V_{h+1}^k-V_{h+1}^\pi)(s_{h+1},y_h+r_h)}\\
    =&\underbrace{\psi(s_h,a_h)^\top\sum_{s_{h+1},r_h}\chi(s_{h+1},r_h)(V_{h+1}^k-V_{h+1}^\pi)(s_{h+1},y_h+r_h) }_{p_1}\\
    &-\lambda\underbrace{\psi(s_h,a_h)\sum_{s_{h+1},r_h}\chi(s_{h+1},r_h)(V_{h+1}^k-V_{h+1}^\pi)(s_{h+1},y_h+r_h)}_{p_2}
\end{align*}
\begin{align*}
    p_1= \TT_h(V_{h+1}^k-V_{h+1}^\pi)(s_h,a_h)
\end{align*}
\begin{align*}
    \abs{p_2}\leq \sqrt{\lambda}d^2U(H-h+1))\norm{\psi(s_h,a_h)}_{\Lambda_{k,h}^{-1}}
\end{align*}
Thus, by choosing the appropriate constants, we have:
\begin{align*}
    \abs{\psi(s_h,a_h)^\top w_h^k(y_h)-Q_h^\pi(s_h,y_h,a_h)-\TT_h(V_{h+1}^k-V_{h+1}^\pi)(s_h,a_h,y_h)}\leq \beta \norm{\psi(s_h,a_h)}_{\Lambda_{k,h}^{-1}}
\end{align*}
\end{proof}
\begin{lemma}[Optimism]\label{lem:optimism}
    For any $h\in[H]$ and $k\in[K]$, event \ref{eq:linear optimism at h} holds
\end{lemma}
\begin{proof}
    We prove the Lemma via induction. For any $k\in[K]$, \ref{eq:linear optimism at h} holds at step $H$. Suppose that \ref{eq:linear optimism at h} holds at step $h+1$, then we have from Lemma \ref{lem:linear concentration}, 
    \begin{align*}
        Q_h^k(s_h,y_h,a_h)=&\psi(s_h,a_h)w_h^k(y_h)-\beta\norm{\psi(s,a)}_{\Lambda_{k,h}^{-1}}\\
        =& Q_h^\pi(s_h,y_h,a_h)+\TT_h(V_{h+1}^k-V_{h+1}^\pi)(s_h,a_h,y_h)+\Delta_h^k(s_h,a_h)\\
        \leq Q_h^\pi(s_h,y_h,a_h)
    \end{align*}
    So \ref{eq:linear optimism at h} holds at step $h$.
\end{proof}

\begin{proof}[Proof of Thoerem~\ref{thm:linearcvar}]
    On event \ref{eq:linear optimism at h} for all $h\in[H]$ and $k\in[K]$, we have that:
    \begin{align*}
        b_k-\tau^{-1}V_1^k(s_1,-b_k)\geq b^*-\tau^{-1}V_1^k(s_1,-b^*)\geq b^*-\tau^{-1}V_1^*(s_1,-b^*)=\operatorname{CVaR}_\tau^*
    \end{align*}
    Also, we have that $\operatorname{CVaR}_\tau^{\pi^k}=\argmax_b\sets{b-\tau^{-1}V_1^{\pi^k}(s_1,-b)}\geq b_k-\tau^{-1}V_1^{\pi^k}(s_1,-b_k)$
    Thus, the total regret can be bounded as:
    \begin{align*}
        \sum_{k=1}^K \operatorname{CVaR}_{\tau}^*-\operatorname{CVaR}_{\tau}^{\pi^k}\leq \sum_{k=1}^K
\tau^{-1}\bracket{V_1^{\pi^k}(s_1,-b_k)-V_1^k(s_1,-b_k)}    \end{align*}

Define $\delta_h^k=V_h^{\pi^k}(s_h^k,y_h^k)-V_h^k(s_h^k,y_h^k)$, and $\zeta_{h+1}^k=\EE[\delta_{h+1}^k|s_h^k,y_h^k,a_h^k]-\delta_{h+1}^k$.
According to Lemma \ref{lem:linear concentration}, we have that:
\begin{align*}
    \delta_h^k\leq\delta_{h+1}^k+\zeta_{h+1}^k+2\beta\norm{\psi_{k,h}}_{\Lambda_{k,h}^{-1}}
\end{align*}
Then we have:
\begin{align*}
    \sum_{k=1}^K V_1^{\pi^k}(s_1,-b_k)-V_1^k(s_1,-b_k)\leq\sum_{h=1}^H \sum_{k=1}^K\zeta_h^k+2\sum_{h=1}^H\sum_{k=1}^K\beta  \norm{\psi_{k,h}}_{\Lambda_{k,h}^{-1}}
\end{align*}
Since we have event \ref{eq:linear optimism at h}, we have $\abs{\delta_h^k}\leq H$. Thus, using the Hoeffding inequality, we have:
\begin{align*}
    \sum_{h=1}^H\sum_{k=1}^K \zeta_{h}^k \leq \mathcal{O}\bracket{H\sqrt{K\log(KH/\delta)}}
\end{align*}
Using the standard linear elliptical potential lemma, we obtain:
\begin{align*}
    \sum_{h=1}^H\sum_{k=1}^K \norm{\psi_{k,h}}_{\Lambda_{k,h}^{-1}}\leq \mathcal{O}(Hd\sqrt{K\log(KH/\delta)})
\end{align*}
Then we have the regret bounded as:
\begin{align*}
    \sum_{k=1}^K \operatorname{CVaR}_{\tau}^*-\operatorname{CVaR}_{\tau}^{\pi^k}\leq{\mathcal{\Tilde{O}}}\bracket{\tau^{-1}H^2d^3\sqrt{MK}}
\end{align*}

\end{proof}
\newpage
\subsection{Numerical Experiment Results}\label{sec:experiments}
In this section we provide the details of our numerical experiments. We construct a zero-mean MDP where the expected return for all the state-action pairs are $0$, thus risk-neutral algorithms such as LSVI-UCB of \cite{jin2020provably} will learn nothing. We also compare our results with the optimistic MDP algorithm of \cite{bastani2022regret}. For simplicity we constructed a toy MDP with $S=3$, $A=2$, $d=2$, $H=6$, $M=3$. The results are in Figure \ref{fig:figures}. From the figures we can see that the risk-neutral algorithm LSVI-UCB fails to learn anything, while the model-based algorithm of \cite{bastani2022regret} converges much slower than ours. Also, with smaller $\tau$ we have larger regret, which is consistent with previous analysis.
\begin{figure}[h!]

\begin{subfigure}{0.45\textwidth}
    \includegraphics[width=\textwidth]{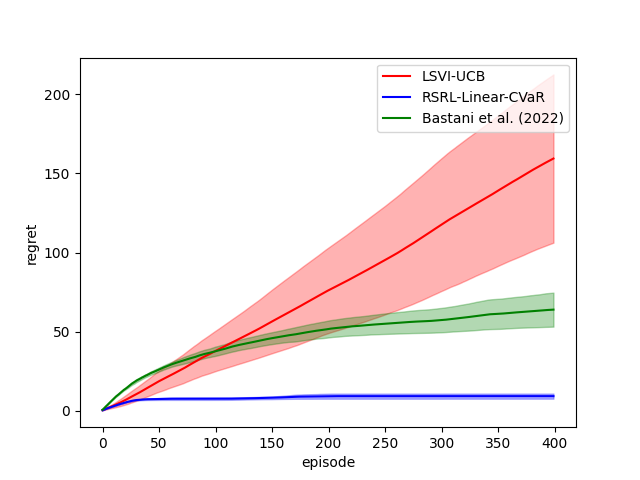}
    \caption{$\operatorname{CVaR}_{0.2}$}
    \label{fig:first}
\end{subfigure}
\hfill
\begin{subfigure}{0.45\textwidth}
    \includegraphics[width=\textwidth]{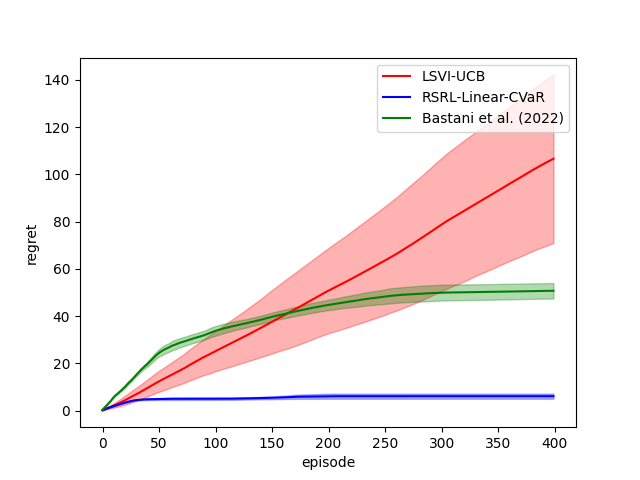}
    \caption{$\operatorname{CVaR}_{0.3}$.}
    \label{fig:second}
\end{subfigure}
\hfill
\begin{subfigure}{0.45\textwidth}
    \includegraphics[width=\textwidth]{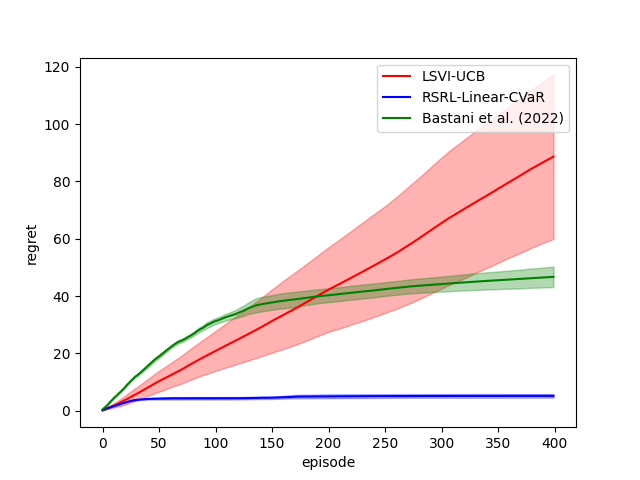}
    \caption{$\operatorname{CVaR}_{0.5}$}
    \label{fig:third}
\end{subfigure}
\hfill
\begin{subfigure}{0.45\textwidth}
    \includegraphics[width=\textwidth]{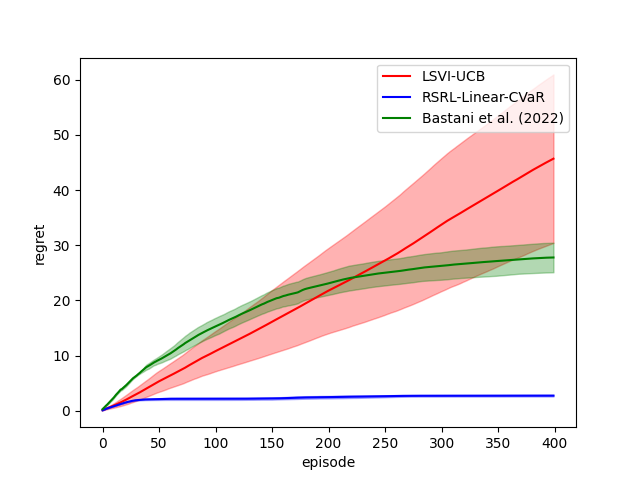}
    \caption{$\operatorname{CVaR}_{0.7}$}
    \label{fig:fourth}
\end{subfigure}
        
\caption{Comparison for different algorithms for the CVaR objective $\operatorname{CVaR}_{\tau}$ under different risk parameter $\tau$.}
\label{fig:figures}
\end{figure}

\end{document}